\documentclass{article} 
\usepackage{iclr2027_conference,times}

\usepackage{amsmath,amsfonts,bm}

\def\eqref#1{equation~\ref{#1}}

\def\1{\bm{1}}

\DeclareMathAlphabet{\mathsfit}{\encodingdefault}{\sfdefault}{m}{sl}
\SetMathAlphabet{\mathsfit}{bold}{\encodingdefault}{\sfdefault}{bx}{n}

\newcommand{\E}{\mathbb{E}}

\usepackage{hyperref}
\usepackage{url}
\usepackage{amsmath}

\usepackage{enumitem}

\usepackage{tcolorbox}
\tcbuselibrary{skins, breakable}

\usepackage[utf8]{inputenc} 
\usepackage[T1]{fontenc}    
\usepackage{hyperref}       
\usepackage{url}            
\usepackage{booktabs}       
\usepackage{amsfonts}       
\usepackage{amsthm}         
\usepackage{nicefrac}       
\usepackage{microtype}      
\usepackage{xcolor}         
\usepackage{algorithm}
\usepackage{algpseudocode}
\usepackage{multirow}
\usepackage[table]{xcolor}
\usepackage{colortbl}
\usepackage{graphicx}
\usepackage{subcaption}
\usepackage{titletoc} 
\usepackage{bbm}

\title{Self-Play Enhancement via Advantage-Weighted Refinement in Online Federated LLM Fine-Tuning}

\author{Seohyun Lee\textsuperscript{\rm 1}, 
Wenzhi Fang\textsuperscript{\rm 1}, 
Dong-Jun Han\textsuperscript{\rm 2}, \\
\textbf{Seyyedali Hosseinalipour}\textsuperscript{\rm 3}, 
\textbf{Christopher G. Brinton}\textsuperscript{\rm 1} \\
\textsuperscript{\rm 1}Purdue University, 
\textsuperscript{\rm 2}Yonsei University, 
\textsuperscript{\rm 3}University at Buffalo-SUNY\\
\texttt{lee3296@purdue.edu}, \texttt{fang375@purdue.edu}, \texttt{djh@yonsei.ac.kr}, \\
\texttt{alipour@buffalo.edu}, \texttt{cgb@purdue.edu}\\
}

\iclrfinalcopy 
\begin{document}

\maketitle

\newcommand{\tok}{\mathrm{tok}}
\newcommand{\concat}{\mathbin{\|}}

\newtheorem{theorem}{Theorem}
\newtheorem{lemma}[theorem]{Lemma}
\newtheorem{corollary}[theorem]{Corollary}
\newtheorem{assumption}{Assumption}

\begin{abstract}
Recent works have advanced feedback-based learning systems, whereby a foundation model is able to intake incoming feedback (e.g., from a user) to self-improve, creating a self-loop system of training. However, existing works are limited in needing to consider an offline setup to allow for such feedback-based methods, requiring privileged ground-truth contexts for training. Moreover, there is limited consideration of federated learning (FL), which is particularly well-suited for incorporating external feedback across large networks of end users, for example, but requires methods to be efficient for training on resource-constrained edge devices. Therefore, we introduce SPEAR (Self-Play Enhancement via Advantage-Weighted Refinement), an efficient online learning algorithm for federated LLM fine-tuning. SPEAR utilizes a \textit{feedback-guided self-play loop} to construct naturally contrastive pairs per prompt which are used to train the model with (i) standard maximum likelihood on satisfactory completions and (ii) confidence-weighted unlikelihood on tail tokens of unsatisfactory completions. SPEAR requires only an interaction-level binary satisfaction signal together with non-answer feedback, without requiring a ground-truth answer to be provided, and does not need expensive group generations, meaning training both online and in a resource-efficient manner is feasible. We validate SPEAR across various benchmark datasets, demonstrating its superior performance in comparison to relevant baselines.
\end{abstract}

\section{Introduction}\label{sec:intro}
Federated Learning (FL) \citep{konevcny2016federated, mcmahan2017communication} has emerged as a promising paradigm for training machine learning (ML) models in a distributed and privacy-preserving manner. In FL, a network of edge devices (clients) collaboratively train a shared model by performing local updates on their private data and transmitting intermediate representations (e.g., model parameters) to a central server for aggregation. This framework eliminates the need to transfer sensitive data while still enabling the model to benefit from diverse data sources. Motivated by these advantages, recent work has begun exploring how to adapt FL for fine-tuning foundation models \citep{bommasani2021opportunities}, such as large language models (LLMs) \citep{lee2025tap, ye2024openfedllm, fang2025federated}.

In parallel, centralized LLM fine-tuning paradigms have increasingly focused on feedback-driven learning, where models iteratively improve by incorporating external feedback signals \citep{zhao2026self}. In addition, reinforcement learning-based fine-tuning and preference optimization, including approaches such as Group Relative Policy Optimization (GRPO) \citep{shao2024deepseekmath}, have demonstrated strong performance gains by aligning model outputs with desired behaviors. 

\textbf{Challenges:} Despite these advances, there remains limited understanding of how feedback-driven training can be effectively integrated into FL settings. Specifically, existing approaches to feedback-based fine-tuning of LLMs are typically designed for centralized, offline pipelines that rely on privileged information (e.g., ground-truth answers or curated preference datasets) \citep{zhao2026self, song2026expanding}. In contrast, real-world deployments often require online learning from noisy and imperfect feedback sources, such as end users. Furthermore, while methods like GRPO provide an effective framework for preference-based optimization, they introduce significant computational overhead, making them challenging to deploy on resource-constrained edge devices in FL environments. These limitations highlight the need for efficient and online feedback-aware LLM fine-tuning strategies tailored to FL.

\subsection{Contributions}
Driven by these challenges, we propose a \textit{computationally efficient, online, and novel} learning algorithm suited for federated LLM fine-tuning. Our self-play enhancement via advantage-weighted refinement (SPEAR) algorithm uses a \textit{feedback-guided self-play} loop to construct natural contrasting pairs: satisfactory completions (wins) and unsatisfactory completions (losses). SPEAR executes LLM fine-tuning via (i) a standard maximum-likelihood estimation on the win traces and (ii) a \textit{confidence-gated unlikelihood} to suppress the high confidence tail tokens for unsatisfactory outputs. Overall, our key contributions are outlined as follows:

\begin{itemize}[leftmargin=*, itemsep=2pt, topsep=2pt, parsep=0pt]
    \item \textit{Online, Real-time Feedback-based FL.} We formulate a feedback-guided environment for fine-tuning of LLMs in an FL setting. Without requiring complete ground-truth answers in the feedback, we enable an online learning protocol for LLM fine-tuning. 
    \item \textit{Self-Play Enhancement Refinement Algorithm.} We propose a novel feedback-guided self-play enhanced methodology for LLM fine-tuning, which seeks to separate samples into win and lose traces to create contrasting samples for training. Via hyperparameters that control the confidence and number of tail tokens targeted for unsatisfactory lose traces, SPEAR's unlikelihood auxiliary loss discourages high-confidence tokens associated with unsatisfactory outputs. We show theoretically that a model whose SPEAR loss is bounded by $\epsilon$ satisfies a guaranteed per-token separation between win and active lose completions. 
    \item \textit{Experimental Validation.} We conduct extensive experiments on a diverse set of benchmark datasets, comparing SPEAR with existing relevant baselines. Via this, we show the superiority of SPEAR in not only accuracy, but also computational efficiency, meaning that even in an online setting with no privileged context, SPEAR still maintains superior performance.  
\end{itemize}

\section{Related Work}\label{sec:related_work}
\textbf{Parameter Efficient Fine-tuning (PEFT):} For fine-tuning existing foundation models \citep{bommasani2021opportunities}, a common approach is to adopt PEFT techniques that do not require complete re-training of the model, resulting in computational efficiency and naturally suited for FL scenarios \citep{wang2019adaptive, lee2025tap, kuo2024federated}. Common PEFT methods include prefix-tuning \citep{li2021prefix} and LoRA \citep{hu2022lora} (which adds and updates two low-rank matrices instead of the original weights), with the latter being the most popular option and the approach we adopt.

\textbf{Federated LLM Fine-Tuning:} A crucial component of FL \citep{konevcny2016federated, mcmahan2017communication} is consideration for resource-constrained edge devices and enhancing computational efficiency \citep{sattler2019robust, shahid2021communication, imteaj2021survey}. As a consequence, considerable attention has been given to fine-tuning pre-trained foundation models, particularly LLMs, in a FL setup with PEFT methods \citep{lee2025tap}. For example, \citet{fang2025federated} consider how to train LLMs when their transmitted update sizes differ (dependent on each client's available resources) via a sketching mechanism while \citep{yuan2025local} seek a method to work with LLMs in a fast and efficient manner in a cloud environment.

\textbf{Feedback-based Learning:} Feedback-guided fine-tuning involves the fine-tuning of pre-trained foundation models via an external feedback source, which is utilized to self-improve the model. In this realm, self-distillation \citep{zhao2026self, shenfeld2026self} has emerged as an option, whereby the model aims to minimize the KL-divergence \citep{kullback1951information, hinton2015distilling} between the feedback-augmented output and original model output via knowledge distillation (KD) \citep{zhao2026self}. Other works have considered utilizing reinforcement learning (RL) \citep{kaelbling1996reinforcement} in conjunction with a distillation process \citep{song2026expanding}. However, prior works such as \citet{zhao2026self} rely on privileged access to ground-truth contexts for feedback. In addition, RL-based approaches often require explicitly defined reward functions and can be computationally inefficient \citep{song2026expanding}, particularly when built upon algorithms such as GRPO due to multiple generations required for group calculation \citep{shao2024deepseekmath}. Compared to prior work, SPEAR is crafted to operate in \textit{online, imperfect feedback settings} and is capable of \textit{preventing the reinforcement of undesired feedback-based outputs}, which can lead to model collapse \citep{shumailov2024ai}.

\textbf{Unlikelihood Training:} Unlikelihood training \citep{welleck2019neural}, in contrast with standard supervised fine-tuning (SFT), seeks to minimize the probability that certain tokens are generated by taking the complement of the maximization probability. Unlikelihood training has been shown to be effective in discouraging certain outputs over various use cases, such as overuse of frequent words and excessive repetition \citep{li2020don, lagutin2021implicit}. This naturally makes it suitable for self-play feedback based systems, where outputs can be susceptible to undesirable completions and therefore model collapse \citep{shumailov2024ai}. SPEAR, through its confidence-gated, tail-targeted unlikelihood loss, is the first work to consider this dimension.  

\section{Proposed Methodology}

\begin{figure*}[t]
    \centering
    \includegraphics[width=0.85\linewidth]{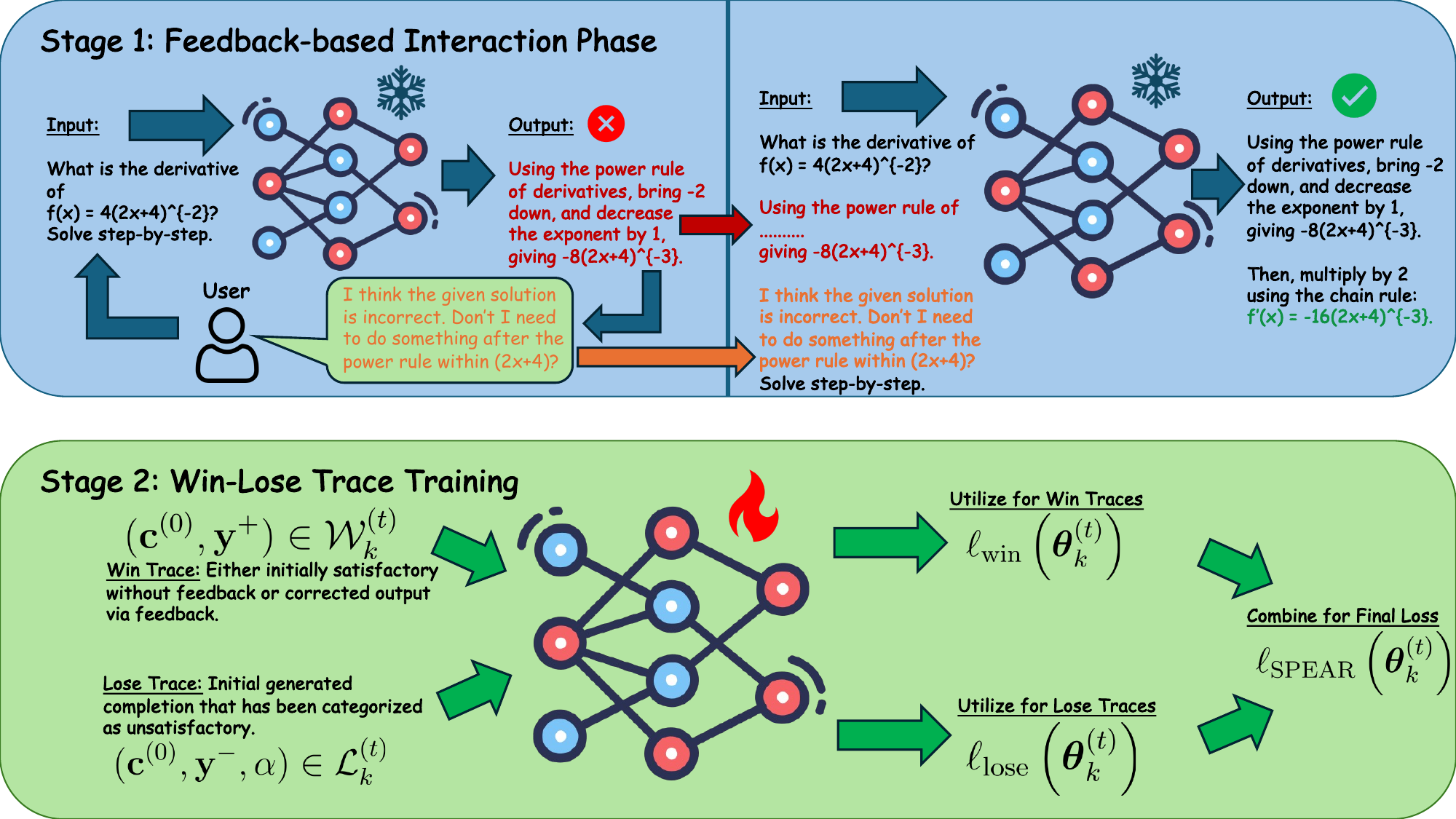}
    \caption{The two phases of the SPEAR algorithm. Firstly, the model interacts with an incoming feedback source (e.g., a user) to correct unsatisfactory generations. After the interaction phase, it categorizes the samples into wins and losses, which are then used to train a standard MLE and unlikelihood objective. This process repeats at each federated round $t$ for each client selected for aggregation.}
    \label{fig:SPEAR_method}

\vspace{-0.17in}
\end{figure*}

\subsection{Preliminaries and Problem Outline}

We consider an \textit{online, feedback-driven FL} setting, where a model is deployed across $K$ distributed clients indexed by $k \in \{1,2,\ldots,K\} = \mathcal{K}$ and interacts with an incoming feedback source (e.g., a human user). 
Unlike standard supervised or offline preference-based learning, SPEAR does not require ground-truth answers or curated preferred responses. Instead, it uses a binary satisfaction signal and partial, non-answer feedback, allowing multiple distinct responses to be treated as satisfactory.

At each client $k$, a private, but evolving incoming data stream exists, where this incoming data is considered \textit{off-policy} (i.e., not generated by the current model) and follows a distribution $\mathbf{x}_k \sim \mathcal{D}_k$.
For a given prompt $\mathbf{x}_k \sim \mathcal{D}_k$ formed from the incoming data stream, the model first produces an initial completion. Upon observing this completion, an existing (e.g., end user) feedback source provides feedback $\mathbf{f}$, which may contain partial hints or intermediate reasoning, but does not reveal the full ground-truth answer (Stage 1 of Fig. \ref{fig:SPEAR_method}). The feedback can be appended to the end of the initial context to form a revised input, enabling iterative refinement. 

This process induces a \textit{trajectory} per prompt consisting of an initial generation and subsequent feedback-conditioned revisions. These trajectories form the fundamental training signal in our setting, rather than static labeled pairs. In particular, each trajectory can be categorized into successful (i.e., initially satisfactory or corrected) and unsuccessful generations, which are later used to construct contrasting training objectives. These generations can be formed from the following process: Given a context $\mathbf{c}$ (a token sequence), the model defines
\[
p_{\boldsymbol{\theta}} (v \mid \mathbf{c}) = \text{softmax}(z_{\boldsymbol{\theta}} (\mathbf{c}))_v, \quad v \in \mathcal{V},
\]
where $z_{\boldsymbol{\theta}}(\mathbf{c}) \in \mathbb{R}^{|\mathcal{V}|}$ are the logits, $\mathcal{V}$ is the vocabulary, and $\boldsymbol{\theta}$ captures the model parameters. Then, for a token sequence $\mathbf{y} = (y_0, \ldots, y_{T-1})$, the likelihood factorizes as:
\begin{equation}\label{eq:tokenseq_prob}
    p_{\boldsymbol{\theta}} (\mathbf{y} \mid \mathbf{c}) = \prod_{i=0}^{T-1} p_{\boldsymbol{\theta}} (y_i \mid \mathbf{c}, \mathbf{y}_{<i}),
\end{equation}
where $\mathbf{y}_{<i} = (y_0, \ldots, y_{i-1})$.

These trajectory-inducing interactions occur locally on each client within a federated system coordinated by a central server over communication rounds $t \in \{0,1,\ldots,S\}$. At round $t$, a subset of clients $\mathcal{K}_t \subseteq \mathcal{K}$ is selected to perform local updates based on their collected interaction trajectories. Each client maintains local model parameters $\boldsymbol{\theta}_k^{(t)}$, initialized from the global model $\boldsymbol{\theta}^{(t)}$. After local training, updated parameters are transmitted to the server and aggregated via Federated Averaging (FedAvg) \citep{konevcny2016federated} based on the number of initially satisfactory or corrected generations:
\[
\boldsymbol{\theta}^{(t+1)} \leftarrow \sum_{k \in \mathcal{K}_t} \frac{|\mathcal{W}_k^{(t)}| \cdot \boldsymbol{\theta}_k^{(t)}}{\sum_{k' \in \mathcal{K}_t} |\mathcal{W}_{k'}^{(t)}|},
\]
where $\mathcal{W}_k^{(t)}$ is called the win set for client $k$ at round $t$ (specifics outlined in Sec. \ref{sec:interaction-phase}). Overall, the goal is to learn model parameters $\boldsymbol{\theta}$ that effectively leverage both successful and unsuccessful interaction trajectories under partial feedback, while remaining computationally efficient for deployment in FL environments. More detailed specifics are outlined in the subsequent sections (Sec. \ref{subsec:SPEAR}).

\subsection{SPEAR: Self-Play Enhancement via Advantage-Weighted Refinement}\label{subsec:SPEAR}
As outlined in Sec. \ref{sec:intro}, since the goal is to fine-tune the model in an online manner efficiently while interacting with an entity such as a user, it is important that feedback $\mathbf{f}$ does not require a complete ground-truth answer during training. Therefore, SPEAR consists of two phases for each client in each round: (i) an \textit{interaction phase} that generates and categorizes successful completions and (ii) a \textit{training phase} that optimizes for both satisfactory and unsatisfactory traces produced from the model.

\subsubsection{Phase 1: Interaction Phase}\label{sec:interaction-phase}
Firstly, during the interaction phase (Stage 1 of Fig. \ref{fig:SPEAR_method}) at the start of each FL round, for each prompt $\mathbf{x}_{k, j} \in \mathcal{B}_k$, where $\mathcal{B}_k$ denotes a minibatch on client $k$ and $j$ indexes a sample within a batch, client $k$ will encode $\mathbf{x}_{k, j}$ to form context $\mathbf{c}^{(0)}_{k, j}$ and generate $\mathbf{y}_{k,j}^{(0)} \sim p_{\boldsymbol{\theta}_k}(\cdot | \mathbf{c}^{(0)}_{k, j})$. Then, the user interacting with the model can indicate whether the generated answer is satisfactory, assigned to binary function $\mathsf{IsSatisfied}(\mathbf{y}_{k, j}^{(0)}, \mathbf{x}_{k, j})$. If the user deems the answer as satisfactory, tuple $(\mathbf{c}^{(0)}_{k, j}, \mathbf{y}_{k, j}^{(0)})$ is added to \textit{win set} $\mathcal{W}_k^{(t)}$. If deemed unsatisfactory, revision is required for those traces. In cases where $\mathbf{y}_{k, j}^{(0)}$ is unsatisfactory, it is assumed that the user may still possess partial knowledge relevant to the task. This can be expressed through informative feedback (e.g., intermediate reasoning steps or the identification of potential errors) even if the user is unable to provide the final step. Therefore, the model receives feedback $\mathbf{f}_{k, j} \leftarrow \mathsf{FB}(\mathbf{x}_{k,j}, \mathbf{y}_{k, j}^{(0)})$, and constructs revision context 
\begin{equation}\label{eq:initial-gen-feedback}
    \mathbf{c}^{(1)}_{k, j} = \mathbf{c}^{(0)}_{k, j} \concat \mathbf{y}_{k,j}^{(0)} \concat \mathbf{f}_{k,j},
\end{equation}
which is utilized to generate a revised answer $\mathbf{y}_{k,j}^{(1)} \sim p_{\boldsymbol{\theta}_k}(\cdot | \mathbf{c}^{(1)}_{k, j})$. 

After revision is conducted $N$ times (defined by each client/user), the revision traces are re-classified. If $\mathsf{IsSatisfied}(\mathbf{y}_{k, j}^{(n)}, \mathbf{x}_{k, j})$ is $1$ (where $n$,  $1\leq n\leq N$, is the number of revisions attempted for a sample so far), then the model has successfully self-corrected given the feedback, and $(\mathbf{c}^{(0)}_{k, j}, \mathbf{y}_{k, j}^{(n)})$ is added to $\mathcal{W}_k^{(t)}$, as we would like the model to learn to produce $\mathbf{y}_{k, j}^{(n)}$ from \textit{original prompt} $\mathbf{c}^{(0)}_{k, j}$. For traces that failed to produce a satisfactory answer, $(\mathbf{c}^{(0)}_{k, j}, \mathbf{y}_{k,j}^{(0)})$ is added to \textit{lose set} $\mathcal{L}_k^{(t)}$. We further categorize two types of lose traces: (i) ones where the initial generated answer is unsatisfactory but feedback generation is successful and (ii) those where both attempts (initial and revision) fail. In the first case, we are more confident in the lack of satisfaction for the trace, as we have received indication definitively via the corrected trace. In the second case, while $\mathbf{y}_{k,j}^{(0)}$ is indicated to be unsatisfactory, we have no concrete alternative. Therefore, we define the final construction of lose set $\mathcal{L}_k^{(t)}$ as 

\begin{equation}
\mathcal{L}_k^{(t)} = \left\{ \big((\mathbf{c}^{(0)}_{k, j}, \mathbf{y}_{k,j}^{(0)}),\, \alpha_{k,j}\big) \;\middle|\;\;
\alpha_{k,j} =
\begin{cases}
1 & \text{if $\mathsf{IsSatisfied}(\mathbf{y}_{k, j}^{(n)}, \mathbf{x}_{k, j})$} \\
0.5 & \text{otherwise}
\end{cases}
\right\},
\end{equation}
where $\alpha_{k,j}$ is the confidence weight assigned to each trace in the lose set. In this way, traces whose lack of satisfaction we are more confident about are assigned a higher weight.

\subsubsection{Phase 2: Training Contrasting Objectives}\label{subsec:final-training-phase}
After the interaction phase, training on the collected traces begins (Stage 2 of Fig. \ref{fig:SPEAR_method}). For a client $k$ at round $t$, SPEAR first seeks to engage in maximum likelihood estimation (MLE) via the standard negative log-likelihood (NLL) \citep{shen2024rethinking} on satisfactory completions from win set $\mathcal{W}_k^{(t)}$, i.e.,
\begin{equation}\label{eq:sft}
    \ell_{\text{win}}\left(\boldsymbol{\theta}_k^{(t)}\right) = - \E_{(\mathbf{c}^{(0)}_{k, j}, \mathbf{y}_{k,j}) \sim \mathcal{W}_k^{(t)}} \left[ \sum_{i=0}^{|\mathbf{y}_{k,j}|-1} \log p_{\boldsymbol{\theta}} (y_{k,j,i} | \mathbf{c}^{(0)}_{k, j}, \mathbf{y}_{k,j,< i})  \right].
\end{equation}
While standard supervised fine-tuning (SFT) on only the win traces is important to ensuring that we do not encourage the model to learn to produce unsatisfactory outputs (and thereby degrade performance), it leaves an entire set of samples generated from lose set $\mathcal{L}_k^{(t)}$ that are not utilized. Therefore, to fully exploit the information available to each client, we perform a contrastive separation between satisfactory and unsatisfactory traces, pushing away tokens associated with the latter, similar to metric learning \citep{chen2020simple, lee2025cooperative}. To accomplish this, for a lose trace $(\mathbf{c}^{(0)}_{k, j}, \mathbf{y}_{k,j}^{-}, \alpha_{k,j}) \in \mathcal{L}_k^{(t)}$, SPEAR seeks to reduce the probability of generating an unsatisfactory completion via an \textit{unlikelihood} objective \citep{welleck2019neural, xhonneux2024efficient}. To this end, instead of directly applying the unlikelihood objective to the lose traces, we selectively target tokens using two hyperparameters: a confidence-weighted unlikelihood margin, $\mu \in (0,1)$, and a tail-token selection parameter, $\tau$. The intuition is twofold. First, issues in unsatisfactory completions tend to concentrate toward the end of the sequence, motivating us to focus on the final $\tau$ tokens (we validate this in Appendix \ref{appendix:near-end-support}). Second, we penalize only high-probability tokens associated with unsatisfactory outputs. The confidence-weighted margin, $\mu$, operationalizes this by imposing a threshold on token probabilities, such that only high-confidence issues are discouraged, while low-confidence predictions are preserved to avoid over-penalization. With this, we define active token index set on a completion $\mathbf{y}_{k,j}^-$ as
\begin{equation}\label{eq:tail-set}
    \mathcal{T}_\tau(\mathbf{y}_{k,j}^-) \;=\;
\begin{cases}
\bigl\{\,|\mathbf{y}_{k,j}^-| - \min(\tau,|\mathbf{y}_{k,j}^-|),\;\ldots,\;|\mathbf{y}_{k,j}^-|-1\,\bigr\} & \tau > 0 \\[4pt]
\bigl\{0,\;\ldots,\;|\mathbf{y}_{k,j}^-|-1\bigr\} & \tau = 0,
\end{cases}
\end{equation}
with the full completion targeted when $\tau = 0$. $|\mathbf{y}_{k,j}^-|$ denotes the cardinality of the sequence. Then, with $\mathcal{T}_\tau(\mathbf{y}_{k,j}^-)$, the confidence-gated, tail-targeting unlikelihood objective can be described as
\begin{equation}
\label{eq:unlikelihood-loss}
\begin{aligned}
\ell_{\mathrm{lose}}\left(\boldsymbol{\theta}_k^{(t)}\right)
&= -\E_{(\mathbf{c}^{(0)}_{k, j},\, \mathbf{y}_{k,j}^-,\, \alpha_{k,j}) \sim \mathcal{L}_k^{(t)}} \Biggl[
\alpha_{k,j} \sum_{i \in \mathcal{T}_\tau(\mathbf{y}_{k,j}^-)}
\mathbbm{1}\!\left[
p_{\boldsymbol{\theta}}\!\left(y^-_{k,j,i}\mid \mathbf{c}^{(0)}_{k, j}, \mathbf{y}^-_{k,j,<i}\right) > \mu
\right] \\
&\qquad\qquad\qquad\qquad \cdot
\log\!\left(
1 - p_{\boldsymbol{\theta}}\!\left(y^-_{k,j,i}\mid \mathbf{c}^{(0)}_{k, j}, \mathbf{y}^-_{k,j,<i}\right)
\right)
\Biggr].
\end{aligned}
\end{equation}
Therefore, the loss function of SPEAR finalizes to
\begin{equation}\label{eq:spear-loss}
    \ell_{\text{SPEAR}}\left(\boldsymbol{\theta}_k^{(t)}\right) = \lambda_{w} \ell_{\text{win}}\left(\boldsymbol{\theta}_k^{(t)}\right) + \lambda_l \ell_{\text{lose}}\left(\boldsymbol{\theta}_k^{(t)}\right),
\end{equation}
where $\lambda_{w}$ and $\lambda_{l}$ are weighting factors for the SFT and unlikelihood objectives respectively. After training for $E$ local iterations, the parameters are transmitted to the server for aggregation. When aggregating at the server, the aggregation weight given to each client is relative to the number of \textit{win traces} on client $k$ for that iteration $t$, as we would like to prioritize clients that give more satisfactory answers. In an online setting, the number of samples continuously evolves, making it inappropriate to consider the number of samples for weighting to be fixed, as is common with traditional FedAvg.
 
\section{Analysis}\label{sec:analysis}
We now provide a theoretical characterization of the SPEAR objective, showing that it implicitly enforces a \emph{log-probability margin} between win and lose completions. This property is specific to the joint structure of SPEAR's NLL and confidence-gated unlikelihood objectives, and yields justification for its hyperparameters. We demonstrate this behavior by analyzing the objective on a single paired sample, which suffices to capture the underlying mechanism. Concretely, we consider a win trace $(\mathbf{c}^{(0)}, \mathbf{y}^+) \in \mathcal{W}$ and a lose trace $(\mathbf{c}^{(0)}, \mathbf{y}^-, \alpha) \in \mathcal{L}$ sharing a common initial context $\mathbf{c}^{(0)}$.

\begin{assumption}[Active Token Confidence]\label{assump:confidence}
All tokens in the active tail set exceed the confidence threshold, i.e., 
$
p_{\boldsymbol{\theta}}(y^-_i \mid \mathbf{c}^{(0)}, \mathbf{y}^-_{<i}) > \mu, \;\;\;\; \forall i \in \mathcal{T}_\tau(\mathbf{y}^-).
$
\end{assumption}

\begin{assumption}[Single Sample]\label{assump:single_sample}
The sets $\mathcal{W}$ and $\mathcal{L}$ each contain exactly one trace, namely $(\mathbf{c}^{(0)}, \mathbf{y}^+)$ and $(\mathbf{c}^{(0)}, \mathbf{y}^-, \alpha)$, respectively.
\end{assumption}

\paragraph{Margin definition.} We define the \emph{SPEAR log-probability margin} for a paired win-lose sample as:
\begin{equation}\label{eq:margin}
M_\tau(\boldsymbol{\theta};\, \mathbf{y}^+, \mathbf{y}^-) \;:=\; \underbrace{\frac{1}{|\mathbf{y}^+|}\log p_{\boldsymbol{\theta}}(\mathbf{y}^+ \mid \mathbf{c}^{(0)})}_{\text{win per-token log-prob}} - \underbrace{\frac{1}{|\mathcal{T}_\tau(\mathbf{y}^-)|}\sum_{i \in \mathcal{T}_\tau(\mathbf{y}^-)} \log p_{\boldsymbol{\theta}}\!\left(y^-_i \mid \mathbf{c}^{(0)}, \mathbf{y}^-_{<i}\right)}_{\text{active-lose per-token log-prob}},
\end{equation}
where, since $p_{\boldsymbol{\theta}} \in (0,1]$, both log-probability terms lie in $(-\infty, 0]$. A larger $M_\tau$ indicates the model assigns higher per-token probability to $\mathbf{y}^+$ and lower probability to the active tail tokens of $\mathbf{y}^-$, i.e., a stronger preference for the win over the lose completion.

\begin{theorem}[SPEAR Log-Probability Margin]\label{thm:margin}
Under Assumptions~\ref{assump:confidence} and \ref{assump:single_sample}, suppose $\ell_{\mathrm{SPEAR}}(\boldsymbol{\theta}) \leq \epsilon$ for the paired win-lose sample above. Then
\begin{equation}\label{eq:margin-bound}
M_\tau(\boldsymbol{\theta};\, \mathbf{y}^+, \mathbf{y}^-) \;\geq\; h(\mu) \;-\; \epsilon\!\left(\frac{1}{\lambda_w\, |\mathbf{y}^+|} \;+\; \frac{1}{\lambda_l\, \alpha\, |\mathcal{T}_\tau(\mathbf{y}^-)|}\right),
\end{equation}
where the \textit{confidence-amplified margin constant} is
\begin{equation}\label{eq:h-mu}
h(\mu) \;=\; \log\!\frac{1}{\max\!\left(\tfrac{1}{4},\; \mu(1-\mu)\right)} \;=\;
\begin{cases}
\log 4 & \mu \leq \tfrac{1}{2}, \\[4pt]
\displaystyle\log\frac{1}{\mu(1-\mu)} & \mu > \tfrac{1}{2}.
\end{cases}
\end{equation}
\end{theorem}

The derivation is provided in full in Appendix~\ref{app:proofs}. While the theorem utilizes a single win-lose pair to make the
mechanism transparent, the objective-level result extends directly to arbitrary finite minibatches (Appendix \ref{app:batch_extension}). For the derived result, we highlight two implications.

\paragraph{Universal minimum margin.}
For all $\mu \in (0,1)$, $h(\mu) \geq \log 4$, with equality at $\mu = 1/2$. Consequently, as $\epsilon \to 0$, any minimizer of $\ell_\mathrm{SPEAR}$ achieves at least $\log 4$ separation up to the error term between win and active-lose distributions. Crucially, $\log 4$ is the \emph{tight} constant arising from the algebraic coupling between $\ell_\mathrm{win}$ and $\ell_\mathrm{lose}$: the inequality $4p(1-p) \leq 1$ holds for all $p \in [0,1]$ with equality at $p=\tfrac{1}{2}$, reflecting the maximum entropy configuration of a binary token decision.

\paragraph{Coverage-strength tradeoff.}
For $\mu > \tfrac{1}{2}$, $h(\mu) > \log 4$ and grows monotonically to $+\infty$ as $\mu \to 1$, so the guaranteed per-token margin strictly strengthens with $\mu$. However, Theorem~\ref{thm:margin} characterizes only tokens that \emph{exceed} the threshold; tokens below $\mu$ are excluded from penalization entirely. This creates a coverage-strength tradeoff that motivates $\mu$ as a hyperparameter to tune: a larger $\mu$ yields a stronger per-token margin guarantee but a smaller active set $\mathcal{T}_\tau(\mathbf{y}^-)$, reducing the gradient signal available during training. Empirically, values of $\mu$ well below $0.5$ can outperform larger thresholds. The confidence weight $\alpha$ plays an analogous role: a larger $\alpha$ (assigned to lose traces where self-correction succeeded, indicating higher certainty of dissatisfaction) causes the margin in \eqref{eq:margin-bound} to increase, further justifying SPEAR's distinction between the two lose-trace types. We also note that for a larger tail token selection parameter $\tau$, the RHS of \eqref{eq:margin-bound} increases, meaning a larger margin. While this would seem to indicate targeting the entire unsatisfactory trace would be beneficial, as outlined in Sec. \ref{subsec:final-training-phase}, unsatisfactory reasoning will concentrate towards the end of the completion, meaning targeting the entire completion will result in tokens we want close to be pushed away, resulting in suboptimal performance. We verify this empirically in Appendix \ref{appendix:near-end-support}.

\section{Experimental Results}\label{sec:experiments}
\paragraph{Experimental Setup.} We apply SPEAR to two widely used open-source LLM families: Qwen2.5 (1.5B) \citep{bai2023qwen, qwen2, qwen2.5} and Llama3.2 (3B) \citep{grattafiori2024llama, touvron2023llama} (Additional experiments on Qwen2.5-7B and Qwen3.5-2B can be found in Appendix \ref{appendix:7b_model} and \ref{appendix:qwen3p5_2b_section}). These models were selected to capture performance across varying model scales, while remaining representative of the smaller footprints typical in client-side FL settings. For the FL network, 50 total clients are considered, with each client attempting to correct itself with feedback a maximum of two times ($N=2$) for each sample (ablation against $N=1$ found in Appendix \ref{appendix:n=1_revision}). We consider the usage of LoRA \citep{hu2022lora} as our PEFT method, targeting the attention and projection modules at each layer. For all experiments, the AdamW optimizer is utilized \citep{loshchilov2017decoupled}, with linear warmup and cosine decay on the learning rate \citep{alimisis2025we}. All experiments are conducted on a server with an NVIDIA A100-40GB GPU, utilizing the HuggingFace \citep{jain2022hugging} and PyTorch \citep{paszke2019pytorch} libraries for implementation. More detailed specifications can be found in Appendix \ref{appendix:hyperparams}. 

\paragraph{Datasets.} We consider four benchmark datasets encompassing a diverse range of domains: ARC-Challenge \citep{allenai:arc} for science-based question answering, HellaSwag \citep{zellers2019hellaswag} for common-sense reasoning sentence completion, MathMCQA for competition-level mathematics \citep{math_mcqa_2025}, and StrategyQA \citep{geva2021did} for multi-hop reasoning (an open-ended setting on MBPP can be found in Appendix \ref{appendix:mbpp_section}). We measure the accuracy of the final produced answer after reasoning from the testing dataset to evaluate the performance of SPEAR and the baselines. In terms of the incomplete but informative feedback given to the model, for HellaSwag and ARC-Challenge, we include the first 2-3 words of the correct completion. MathMCQA's feedback includes the first three lines or 300 characters of the step-by-step mathematical reasoning field (whichever is shorter). StrategyQA's feedback encompasses half of the ``facts" column for each sample. Importantly, SPEAR does not require the ground-truth answer to be exposed to the model. For benchmark reproducibility, we instantiate the binary satisfaction signal using task correctness as a proxy for user satisfaction. The reference answer itself is never provided as a training target. Additional results demonstrating SPEAR's performance on comparatively weaker feedback can be found in Appendix \ref{app:weaker_feedback}. Specific text templates utilized can be found in Appendix \ref{appendix:prompt-format}.

\paragraph{Baselines.}
We evaluate SPEAR against four recent relevant approaches: GRPO \citep{shao2024deepseekmath}, OPSD \citep{zhao2026self}, RLTF-SD \citep{song2026expanding}, and a feedback-based supervised fine-tuning variant without unlikelihood training (Feedback SFT). The RLTF-SD implementation uses a batch rollout strategy (collecting all rollouts upfront before multiple gradient updates) for time-efficiency. These baselines were selected to span the dominant paradigms for improving reasoning in foundation models: RL with explicit reward design (GRPO), self-distillation via structured feedback (OPSD), a mixture of both RL and self-distillation (RLTF-SD), and feedback-driven supervised fine-tuning without unlikelihood loss considerations (Feedback SFT). Importantly, GRPO, RLTF-SD, and OPSD do not operate in the setting we consider: GRPO and RLTF-SD rely on explicitly specified reward functions, while OPSD assumes access to privileged or high-quality ground-truth feedback during distillation. In contrast, SPEAR targets \textit{online learning from imperfect feedback}, where neither explicit reward functions nor ground-truth answers are provided. To the best of our knowledge, prior work has not studied this combination of constraints, and therefore there are no existing baselines for direct comparison. As a result, we compare against the closest established methods (additional results comparing against a DPO-based baseline can be found in Appendix \ref{appendix:dpo_section} and \ref{appendix:qwen3p5_2b_section}).

\paragraph{Additional Experimental Results.} Additional experimental results not included in the main text (e.g., additional ablation studies, additional clients, etc.) can be found in Appendix \ref{app:exps}.

\subsection{Results}

\begin{table*}[t]
\centering
\footnotesize
\setlength{\tabcolsep}{4pt} 
\renewcommand{\arraystretch}{1.15}
\caption{Performance of SPEAR vs. baselines on Llama3.2-3B and Qwen2.5-1.5B across 3 random seeds, showing that SPEAR outperforms the baselines across multiple models and datasets.}
\label{tab:main_results}
\begin{tabular}{llccccc}
\toprule
\textbf{Model} & \textbf{Algorithm} & \textbf{ARC-Challenge} & \textbf{HellaSwag} & \textbf{MathMCQA} & \textbf{StrategyQA} & \textbf{Avg.} \\
\midrule

\multirow{4}{*}{\textbf{Llama3.2-3B}}
& Feedback SFT & 67.95$\pm$1.45 & 84.60$\pm$0.98 & 48.60$\pm$9.89 & 49.49$\pm$1.76 & 62.66 \\
& GRPO         & 65.70$\pm$3.48 & 75.99$\pm$1.53 & 40.20$\pm$14.89 & 49.34$\pm$1.51 & 57.81 \\
& OPSD       & 24.69$\pm$1.55 & 25.03$\pm$0.64 & 0.00$\pm$0.00 & 0.00$\pm$0.00 & 12.43 \\
& RLTF-SD  & 63.28$\pm$3.58 & 60.05$\pm$12.30 & 44.70$\pm$7.92 & 51.58$\pm$0.08 & 54.90 \\
\rowcolor{blue!10}
& \textbf{SPEAR (Ours)} & 69.17$\pm$0.85 & 86.72$\pm$0.67 & 56.70$\pm$0.78 & 65.84$\pm$0.78 & \textbf{69.61} \\
\midrule

\multirow{4}{*}{\textbf{Qwen2.5-1.5B}}
& Feedback SFT & 73.89$\pm$3.38 & 76.14$\pm$3.61 & 54.93$\pm$2.04 & 50.51$\pm$1.76 & 63.87 \\
& GRPO       & 76.30$\pm$0.17 & 79.00$\pm$1.65 & 50.87$\pm$3.86 & 51.38$\pm$6.44 & 64.39 \\
& OPSD       & 25.03$\pm$1.23 & 24.92$\pm$0.69 & 0.00$\pm$0.00 & 0.00$\pm$0.00 & 12.49 \\
& RLTF-SD  & 75.03$\pm$0.40 & 60.65$\pm$20.78 & 44.70$\pm$7.41 & 54.29$\pm$10.08 & 58.67 \\
\rowcolor{blue!10}
& \textbf{SPEAR (Ours)} & 76.93$\pm$1.41 & 80.65$\pm$0.79 & 59.07$\pm$1.16 & 66.08$\pm$3.20 & \textbf{70.68} \\
\bottomrule
\end{tabular}
\end{table*}

\begin{figure}[t]
\centering

\begin{subfigure}[htbp]{0.245\textwidth}
    \centering
    \includegraphics[width=\linewidth]{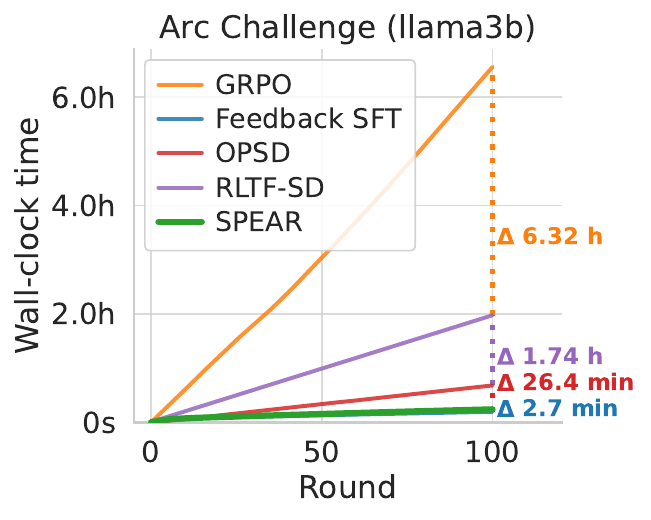}
    \caption{ARC-Challenge}
\end{subfigure}
\begin{subfigure}[htbp]{0.245\textwidth}
    \centering
    \includegraphics[width=\linewidth]{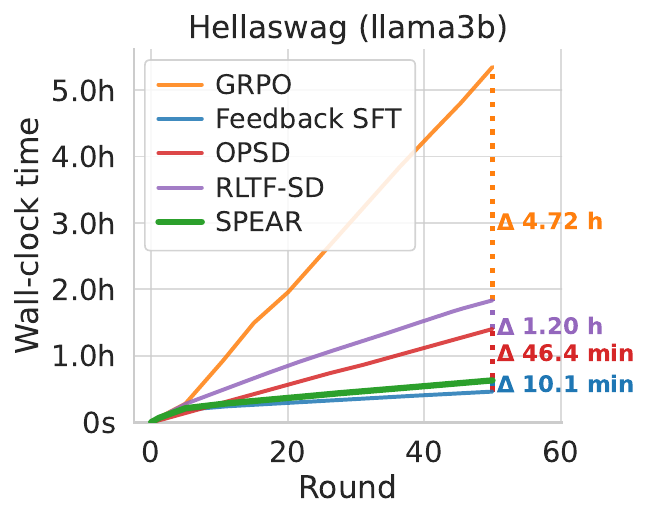}
    \caption{HellaSwag}
\end{subfigure}
\begin{subfigure}[htbp]{0.245\textwidth}
    \centering
    \includegraphics[width=\linewidth]{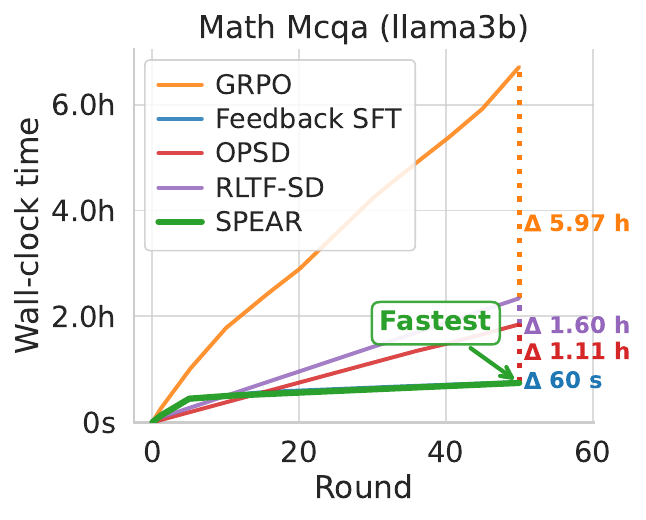}
    \caption{MathMCQA}
\end{subfigure}
\begin{subfigure}[htbp]{0.245\textwidth}
    \centering
    \includegraphics[width=\linewidth]{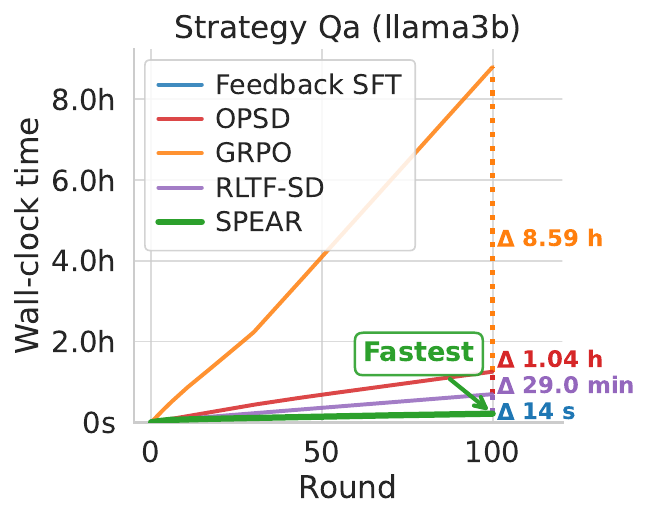}
    \caption{StrategyQA}
\end{subfigure}

\vspace{0.75em}

{\small \textbf{(a) Llama3.2-3B}}

\vspace{0.75em}

\begin{subfigure}[htbp]{0.245\textwidth}
    \centering
    \includegraphics[width=\linewidth]{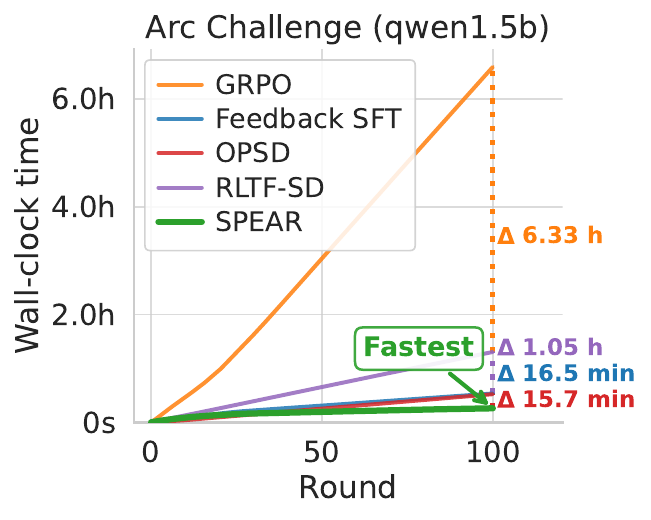}
    \caption{ARC-Challenge}
\end{subfigure}
\begin{subfigure}[htbp]{0.245\textwidth}
    \centering
    \includegraphics[width=\linewidth]{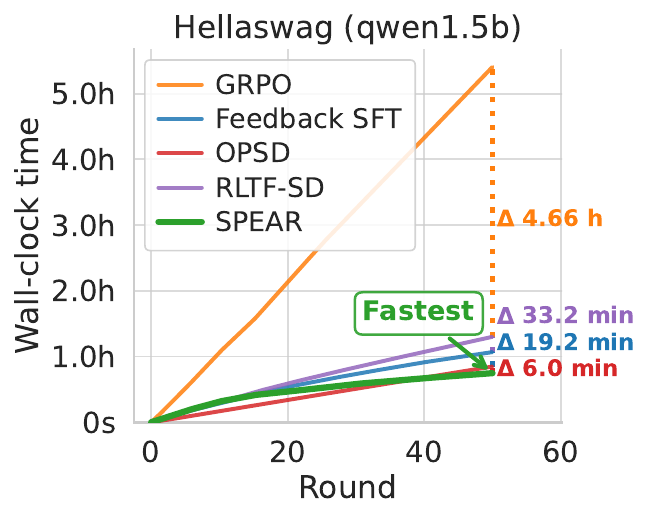}
    \caption{HellaSwag}
\end{subfigure}
\begin{subfigure}[htbp]{0.245\textwidth}
    \centering
    \includegraphics[width=\linewidth]{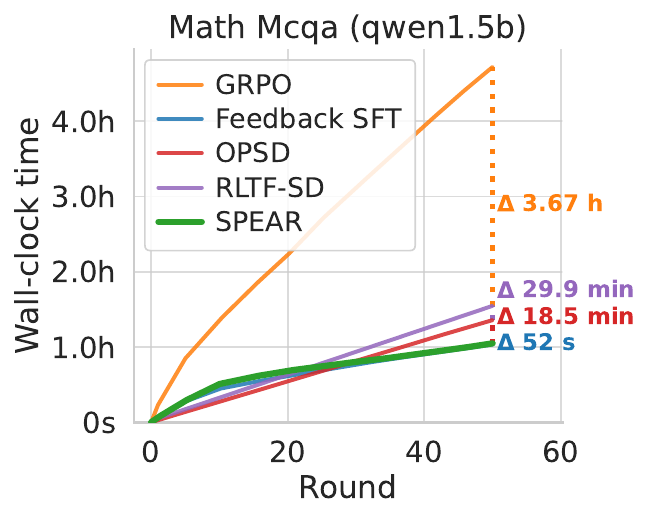}
    \caption{MathMCQA}
\end{subfigure}
\begin{subfigure}[htbp]{0.245\textwidth}
    \centering
    \includegraphics[width=\linewidth]{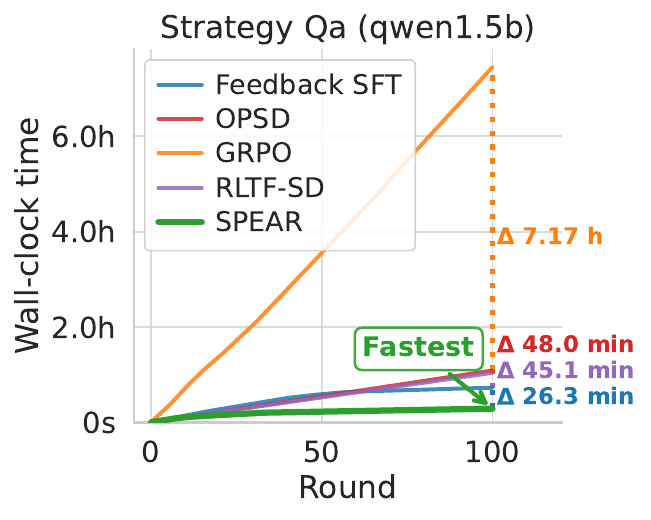}
    \caption{StrategyQA}
\end{subfigure}

\vspace{0.5em}

{\small \textbf{(b) Qwen2.5-1.5B}}

\vspace{0.5em}

\caption{Cumulative training time of each algorithm on (a) Llama-3B and (b) Qwen-1.5B models. $\Delta$ delineates the final wall-clock time difference between the corresponding algorithm and SPEAR. When SPEAR achieves the fastest final performance, it's marked with a ``Fastest" box.}
\label{fig:training_curves}

\end{figure}

\paragraph{Performance in Comparison to Baselines.} Considering the results in Table \ref{tab:main_results}, we note the superiority of SPEAR in comparison to all baselines across all evaluated datasets. In particular, we note the consistent high performance of SPEAR across all scenarios, while some baselines struggle (e.g., GRPO and RLTF-SD achieve only $40\%$ and $45\%$ on Llama for MathMCQA). We can also see engaging in pure feedback SFT on only the satisfactory traces (Feedback SFT) results in performance that lags behind SPEAR, giving credence to the use of our unlikelihood-based loss. Another important observation is the collapse of self-distillation based algorithm OPSD on all datasets in comparison to other baselines. We attribute this phenomenon to the reliance of self-distillation techniques on high-quality feedback-generated outputs. Especially in the case of smaller models typically used in edge settings, the initial feedback-generated responses may be suboptimal and/or outputting an invalid format, leading the KL objective to guide the model toward an undesired distribution. We especially believe this to be the case due to the fact that the performance of OPSD on MathMCQA and StrategyQA achieves specifically $0.0$, which is below random chance, meaning that OPSD is learning to match the KL divergence, but is aligning it with suboptimal completions generated from the feedback due to the lack of checking of the feedback produced answer. Lastly, we also observe that the margin of improvement with SPEAR in comparison with the baselines is greater on tasks that require longer chain-of-thought (CoT) reasoning, as observed with MathMCQA and StrategyQA. 

\paragraph{Wall-clock Time Comparison.} In Fig. \ref{fig:training_curves}, we compare the wall-clock training time of SPEAR and the baselines over the course of federated training. For GRPO and RLTF-SD, due to memory constraints, the number of gradient accumulation steps is increased ($\times 2$(Llama)/$\times 4$(Qwen) for ARC and HellaSwag, $\times 2$ for MathMCQA, and $\times 4$ for StrategyQA) and group size is kept low (specifics found in Appendix \ref{appendix:hyperparams}). We can see across both Llama3.2-3B and Qwen2.5-1.5B, SPEAR remains the fastest algorithm across a majority of cases, with the difference between Feedback SFT and SPEAR being small when the former is quicker than SPEAR. We find that SPEAR outperforms Feedback SFT in most scenarios, despite the added cost of computing the confidence-gated tailed unlikelihood loss. This advantage arises because SPEAR achieves higher accuracy, reducing the number of self-play feedback loops required and thereby improving overall time efficiency. This is also the reason why SPEAR's cumulative time is not always linear; with fewer corrections needed as the model improves, the total number of expensive generations decreases. GRPO and RLTF-SD suffer from substantial memory overhead due to their group-based computation, which necessitates additional gradient accumulation steps and smaller group sizes, which is particularly problematic in resource-constrained federated settings. Combined with extra generation steps, this leads to considerably slower training than SPEAR. While RLTF-SD does not suffer as much as GRPO due to early termination for already satisfactory completions, its reliance on GRPO-based group generation still adds significant overhead in comparison to SPEAR.

\begin{table*}[htbp]
\centering
\small
\setlength{\tabcolsep}{7pt}
\renewcommand{\arraystretch}{1.15}
\caption{Ablation study on $\mu$ for SPEAR across 3 random seeds for Llama3.2-3B and Qwen2.5-1.5B. SPEAR's accuracy remains stable to differing $\mu$ hyperparameter values, indicating robustness.}
\label{tab:ablation_mu}
\begin{tabular}{llcccc}
\toprule
\textbf{Model} & \textbf{$\mu$} & \textbf{ARC-Challenge} & \textbf{HellaSwag} & \textbf{MathMCQA} & \textbf{StrategyQA} \\
\midrule

\multirow{4}{*}{\textbf{Llama3.2-3B}}
& 0.1 & 67.92$\pm$1.79 & \textbf{87.09$\pm$}0.85 & \textbf{56.27$\pm$}2.36 & 64.68$\pm$3.61 \\
& 0.5 & \textbf{69.25$\pm$}1.25 & 86.10$\pm$0.58 & 55.70$\pm$2.56 & 67.20$\pm$0.69 \\
& 0.7 & \textbf{69.25$\pm$}1.37 & 87.00$\pm$0.28 & 54.60$\pm$1.30 & \textbf{68.51$\pm$}2.59 \\
& 0.9 & 69.00$\pm$1.00 & 85.95$\pm$0.93 & 52.50$\pm$1.93 & 66.23$\pm$2.70 \\
\midrule

\multirow{4}{*}{\textbf{Qwen2.5-1.5B}}
& 0.1 & \textbf{77.25$\pm$}1.14 & \textbf{79.69$\pm$}3.34 & \textbf{58.80$\pm$}1.21 & \textbf{66.86$\pm$}1.03 \\
& 0.5 & 63.23$\pm$21.81 & 78.81$\pm$2.26 & 56.77$\pm$1.53 & 66.62$\pm$1.25 \\
& 0.7 & 77.08$\pm$0.49 & 78.27$\pm$1.80 & 57.80$\pm$0.72 & 64.05$\pm$3.28 \\
& 0.9 & 75.09$\pm$0.98 & 79.41$\pm$1.53 & 57.27$\pm$1.34 & 66.62$\pm$1.69 \\
\bottomrule
\end{tabular}
\end{table*}

\paragraph{Ablation over $\mu$.} We study how the $\mu$ hyperparameter influences the training process of SPEAR. Based on the results in Table \ref{tab:ablation_mu}, we can note that in general, setting $\mu$ to a lower threshold more greatly benefits the training. This could be due to the fact that with a lower threshold, more tokens are targeted in the unlikelihood loss, leading to a greater gradient signal (as outlined in Sec. \ref{sec:analysis}). However, we can see that depending on the setting, the optimal setting differs (e.g., $0.5$/$0.7$ works best for ARC on Llama or $0.7$ for StrategyQA on Llama). This indicates that $\mu$ is a hyperparameter to be tuned rather than merely setting to $> 0.5$ or the lowest value possible. An important point, however, \textit{is that in general, the performance of SPEAR still remains high across differing settings} (with the exception of $0.5$ for ARC on Qwen2.5-1.5B), indicating that the methodology is \textit{not overly sensitive to hyperparameter tuning with respect to $\mu$.} For example, the difference between $\mu=0.1$ and $\mu=0.9$ on StrategyQA for Qwen2.5-1.5B is less than $0.5\%$. 

\section{Conclusion}\label{sec:conclusion}
In this paper, we introduced SPEAR, a novel \textit{online, self-play federated fine-tuning} algorithm that works with \textit{imperfect, but informative feedback sources}. In SPEAR, by deliberately separating generated completions into win and lose traces, we allow the model to maximize information available to it and discourage undesirable token generation through a confidence-gated, tail-targeted unlikelihood objective. Through extensive experiments against relevant baselines across a diverse set of datasets, we demonstrated the effectiveness and superiority of SPEAR in both training efficiency and performance. Despite the benefits demonstrated by SPEAR, a potential limitation and future avenue of our work is the need for the feedback to be related to the task: for example, if the user gives feedback that is contrary or irrelevant to the task at hand, it can lead to a scenario where the majority of completions are unsatisfactory, leading to no win traces to sample from. 


\subsection*{AI use statement}
In this work, we used generative AI tools to improve the clarity of the manuscript (e.g., wording, paraphrasing, and readability) and for double-checking the manuscript for inconsistencies (e.g., typos, notation inconsistencies). Moreover, it was used as an aid in the creation of tables throughout the manuscript and code implementation. Research ideation of the proposed methodology was derived from the authors, and the rest of the required disclosure tasks (e.g., generation of new synthetic datasets) are not applicable to this work. We have reviewed all AI-assisted work. As some examples, we have verified the accuracy of the hyperparameter settings in the tables and ensured the accuracy of the generated text. We take responsibility for the final content of this work, including text, claims or artifacts produced with the aid of generative AI.




\subsection*{Reproducibility statement}
For reproducibility, we provide the implementation code of the SPEAR method \textcolor{blue}{\href{https://github.com/lee3296/SPEAR}{here}}. Moreover, the hyperparameter settings and prompting templates used across the experiments can be found in Appendix \ref{appendix:hyperparams} and \ref{appendix:prompt-format} respectively. The pseudocode of SPEAR can additionally be found in Appendix \ref{appendix:pseudocode}. Lastly, the derivation of the Theorem in Sec. \ref{sec:analysis} and its associated batch-level extension can be found in Appendix \ref{app:proofs}.

\bibliography{iclr2027_conference}

@article{konevcny2016federated,
  title={Federated learning: Strategies for improving communication efficiency},
  author={Kone{\v{c}}n{\`y}, Jakub and McMahan, H Brendan and Yu, Felix X and Richt{\'a}rik, Peter and Suresh, Ananda Theertha and Bacon, Dave},
  journal={arXiv preprint arXiv:1610.05492},
  year={2016}
}

@article{shen2024rethinking,
  title={Rethinking data selection for supervised fine-tuning},
  author={Shen, Ming},
  journal={arXiv preprint arXiv:2402.06094},
  year={2024}
}

@inproceedings{chen2020simple,
  title={A simple framework for contrastive learning of visual representations},
  author={Chen, Ting and Kornblith, Simon and Norouzi, Mohammad and Hinton, Geoffrey},
  booktitle={International conference on machine learning},
  pages={1597--1607},
  year={2020},
  organization={PmLR}
}

@article{lee2025cooperative,
  title={Cooperative decentralized backdoor attacks on vertical federated learning},
  author={Lee, Seohyun and Fang, Wenzhi and Das, Anindya Bijoy and Hosseinalipour, Seyyedali and Love, David J and Brinton, Christopher G},
  journal={IEEE Transactions on Networking},
  volume={34},
  pages={2004--2019},
  year={2025},
  publisher={IEEE}
}

@article{welleck2019neural,
  title={Neural text generation with unlikelihood training},
  author={Welleck, Sean and Kulikov, Ilia and Roller, Stephen and Dinan, Emily and Cho, Kyunghyun and Weston, Jason},
  journal={arXiv preprint arXiv:1908.04319},
  year={2019}
}

@article{wang2019adaptive,
  title={Adaptive federated learning in resource constrained edge computing systems},
  author={Wang, Shiqiang and Tuor, Tiffany and Salonidis, Theodoros and Leung, Kin K and Makaya, Christian and He, Ting and Chan, Kevin},
  journal={IEEE journal on selected areas in communications},
  volume={37},
  number={6},
  pages={1205--1221},
  year={2019},
  publisher={IEEE}
}

@article{lee2025tap,
  title={TAP: Two-Stage Adaptive Personalization of Multi-task and Multi-Modal Foundation Models in Federated Learning},
  author={Lee, Seohyun and Fang, Wenzhi and Han, Dong-Jun and Hosseinalipour, Seyyedali and Brinton, Christopher G},
  journal={arXiv preprint arXiv:2509.26524},
  year={2025}
}

@inproceedings{li2021prefix,
  title={Prefix-tuning: Optimizing continuous prompts for generation},
  author={Li, Xiang Lisa and Liang, Percy},
  booktitle={Proceedings of the 59th Annual Meeting of the Association for Computational Linguistics and the 11th International Joint Conference on Natural Language Processing (Volume 1: Long Papers)},
  pages={4582--4597},
  year={2021}
}

@article{hu2022lora,
  title={Lora: Low-rank adaptation of large language models.},
  author={Hu, Edward J and Shen, Yelong and Wallis, Phillip and Allen-Zhu, Zeyuan and Li, Yuanzhi and Wang, Shean and Wang, Liang and Chen, Weizhu and others},
  journal={Iclr},
  volume={1},
  number={2},
  pages={3},
  year={2022}
}

@article{bommasani2021opportunities,
  title={On the opportunities and risks of foundation models},
  author={Bommasani, Rishi and Hudson, Drew A and Adeli, Ehsan and Altman, Russ and Arora, Simran and von Arx, Sydney and Bernstein, Michael S and Bohg, Jeannette and Bosselut, Antoine and Brunskill, Emma and others},
  journal={arXiv preprint arXiv:2108.07258},
  year={2021}
}

@inproceedings{ye2024openfedllm,
  title={Openfedllm: Training large language models on decentralized private data via federated learning},
  author={Ye, Rui and Wang, Wenhao and Chai, Jingyi and Li, Dihan and Li, Zexi and Xu, Yinda and Du, Yaxin and Wang, Yanfeng and Chen, Siheng},
  booktitle={Proceedings of the 30th ACM SIGKDD conference on knowledge discovery and data mining},
  pages={6137--6147},
  year={2024}
}

@article{fang2025federated,
  title={Federated sketching lora: On-device collaborative fine-tuning of large language models},
  author={Fang, Wenzhi and Han, Dong-Jun and Yuan, Liangqi and Hosseinalipour, Seyyedali and Brinton, Christopher G},
  journal={arXiv preprint arXiv:2501.19389},
  year={2025}
}

@article{zhao2026self,
  title={Self-Distilled Reasoner: On-Policy Self-Distillation for Large Language Models},
  author={Zhao, Siyan and Xie, Zhihui and Liu, Mengchen and Huang, Jing and Pang, Guan and Chen, Feiyu and Grover, Aditya},
  journal={arXiv preprint arXiv:2601.18734},
  year={2026}
}

@article{shao2024deepseekmath,
  title={Deepseekmath: Pushing the limits of mathematical reasoning in open language models},
  author={Shao, Zhihong and Wang, Peiyi and Zhu, Qihao and Xu, Runxin and Song, Junxiao and Bi, Xiao and Zhang, Haowei and Zhang, Mingchuan and Li, YK and Wu, Yang and others},
  journal={arXiv preprint arXiv:2402.03300},
  year={2024}
}

@article{song2026expanding,
  title={Expanding the Capabilities of Reinforcement Learning via Text Feedback},
  author={Song, Yuda and Chen, Lili and Tajwar, Fahim and Munos, Remi and Pathak, Deepak and Bagnell, J Andrew and Singh, Aarti and Zanette, Andrea},
  journal={arXiv preprint arXiv:2602.02482},
  year={2026}
}

@article{bai2023qwen,
  title={Qwen technical report},
  author={Bai, Jinze and Bai, Shuai and Chu, Yunfei and Cui, Zeyu and Dang, Kai and Deng, Xiaodong and Fan, Yang and Ge, Wenbin and Han, Yu and Huang, Fei and others},
  journal={arXiv preprint arXiv:2309.16609},
  year={2023}
}

@article{grattafiori2024llama,
  title={The llama 3 herd of models},
  author={Grattafiori, Aaron and Dubey, Abhimanyu and Jauhri, Abhinav and Pandey, Abhinav and Kadian, Abhishek and Al-Dahle, Ahmad and Letman, Aiesha and Mathur, Akhil and Schelten, Alan and Vaughan, Alex and others},
  journal={arXiv preprint arXiv:2407.21783},
  year={2024}
}

@article{loshchilov2017decoupled,
  title={Decoupled weight decay regularization},
  author={Loshchilov, Ilya and Hutter, Frank},
  journal={arXiv preprint arXiv:1711.05101},
  year={2017}
}

@article{allenai:arc,
      author    = {Peter Clark  and Isaac Cowhey and Oren Etzioni and Tushar Khot and
                    Ashish Sabharwal and Carissa Schoenick and Oyvind Tafjord},
      title     = {Think you have Solved Question Answering? Try ARC, the AI2 Reasoning Challenge},
      journal   = {arXiv:1803.05457v1},
      year      = {2018},
}

@inproceedings{zellers2019hellaswag,
  title={Hellaswag: Can a machine really finish your sentence?},
  author={Zellers, Rowan and Holtzman, Ari and Bisk, Yonatan and Farhadi, Ali and Choi, Yejin},
  booktitle={Proceedings of the 57th annual meeting of the association for computational linguistics},
  pages={4791--4800},
  year={2019}
}

@dataset{math_mcqa_2025,
  title={MATH-MCQA: A Multiple Choice Adaptation of the MATH Dataset},
  author={Biderman, Stella},
  year={2025},
  publisher={Hugging Face},
  url={https://huggingface.co/datasets/stellaathena/math_mcqa}
}

@article{geva2021did,
  title={Did aristotle use a laptop? a question answering benchmark with implicit reasoning strategies},
  author={Geva, Mor and Khashabi, Daniel and Segal, Elad and Khot, Tushar and Roth, Dan and Berant, Jonathan},
  journal={Transactions of the Association for Computational Linguistics},
  volume={9},
  pages={346--361},
  year={2021},
  publisher={MIT Press One Rogers Street, Cambridge, MA 02142-1209, USA journals-info~…}
}

@article{sattler2019robust,
  title={Robust and communication-efficient federated learning from non-iid data},
  author={Sattler, Felix and Wiedemann, Simon and M{\"u}ller, Klaus-Robert and Samek, Wojciech},
  journal={IEEE transactions on neural networks and learning systems},
  volume={31},
  number={9},
  pages={3400--3413},
  year={2019},
  publisher={IEEE}
}

@article{shahid2021communication,
  title={Communication efficiency in federated learning: Achievements and challenges},
  author={Shahid, Osama and Pouriyeh, Seyedamin and Parizi, Reza M and Sheng, Quan Z and Srivastava, Gautam and Zhao, Liang},
  journal={arXiv preprint arXiv:2107.10996},
  year={2021}
}

@inproceedings{mcmahan2017communication,
  title={Communication-efficient learning of deep networks from decentralized data},
  author={McMahan, Brendan and Moore, Eider and Ramage, Daniel and Hampson, Seth and y Arcas, Blaise Aguera},
  booktitle={Artificial intelligence and statistics},
  pages={1273--1282},
  year={2017},
  organization={Pmlr}
}

@inproceedings{yuan2025local,
  title={Local-cloud inference offloading for LLMs in multi-modal, multi-task, multi-dialogue settings},
  author={Yuan, Liangqi and Han, Dong-Jun and Wang, Shiqiang and Brinton, Christopher},
  booktitle={Proceedings of the Twenty-sixth International Symposium on Theory, Algorithmic Foundations, and Protocol Design for Mobile Networks and Mobile Computing},
  pages={201--210},
  year={2025}
}

@article{kuo2024federated,
  title={Federated lora with sparse communication},
  author={Kuo, Kevin and Raje, Arian and Rajesh, Kousik and Smith, Virginia},
  journal={arXiv preprint arXiv:2406.05233},
  year={2024}
}

@article{kullback1951information,
  title={On information and sufficiency},
  author={Kullback, Solomon and Leibler, Richard A},
  journal={The annals of mathematical statistics},
  volume={22},
  number={1},
  pages={79--86},
  year={1951},
  publisher={JSTOR}
}

@article{hinton2015distilling,
  title={Distilling the knowledge in a neural network},
  author={Hinton, Geoffrey and Vinyals, Oriol and Dean, Jeff},
  journal={arXiv preprint arXiv:1503.02531},
  year={2015}
}

@article{shenfeld2026self,
  title={Self-Distillation Enables Continual Learning},
  author={Shenfeld, Idan and Damani, Mehul and H{\"u}botter, Jonas and Agrawal, Pulkit},
  journal={arXiv preprint arXiv:2601.19897},
  year={2026}
}

@article{kaelbling1996reinforcement,
  title={Reinforcement learning: A survey},
  author={Kaelbling, Leslie Pack and Littman, Michael L and Moore, Andrew W},
  journal={Journal of artificial intelligence research},
  volume={4},
  pages={237--285},
  year={1996}
}

@article{shumailov2024ai,
  title={AI models collapse when trained on recursively generated data},
  author={Shumailov, Ilia and Shumaylov, Zakhar and Zhao, Yiren and Papernot, Nicolas and Anderson, Ross and Gal, Yarin},
  journal={Nature},
  volume={631},
  number={8022},
  pages={755--759},
  year={2024},
  publisher={Nature Publishing Group UK London}
}

@inproceedings{li2020don,
  title={Don’t say that! making inconsistent dialogue unlikely with unlikelihood training},
  author={Li, Margaret and Roller, Stephen and Kulikov, Ilia and Welleck, Sean and Boureau, Y-Lan and Cho, Kyunghyun and Weston, Jason},
  booktitle={Proceedings of the 58th Annual Meeting of the Association for Computational Linguistics},
  pages={4715--4728},
  year={2020}
}

@inproceedings{lagutin2021implicit,
  title={Implicit unlikelihood training: Improving neural text generation with reinforcement learning},
  author={Lagutin, Evgeny and Gavrilov, Daniil and Kalaidin, Pavel},
  booktitle={Proceedings of the 16th Conference of the European Chapter of the Association for Computational Linguistics: Main Volume},
  pages={1432--1441},
  year={2021}
}

@article{touvron2023llama,
  title={Llama: Open and efficient foundation language models},
  author={Touvron, Hugo and Lavril, Thibaut and Izacard, Gautier and Martinet, Xavier and Lachaux, Marie-Anne and Lacroix, Timoth{\'e}e and Rozi{\`e}re, Baptiste and Goyal, Naman and Hambro, Eric and Azhar, Faisal and others},
  journal={arXiv preprint arXiv:2302.13971},
  year={2023}
}

@article{alimisis2025we,
  title={Why do we need warm-up? A theoretical perspective},
  author={Alimisis, Foivos and Islamov, Rustem and Lucchi, Aurelien},
  journal={arXiv preprint arXiv:2510.03164},
  year={2025}
}

@article{green2001modelling,
  title={Modelling heterogeneity with and without the Dirichlet process},
  author={Green, Peter J and Richardson, Sylvia},
  journal={Scandinavian journal of statistics},
  volume={28},
  number={2},
  pages={355--375},
  year={2001},
  publisher={Wiley Online Library}
}

@inproceedings{reguieg2023comparative,
  title={A comparative evaluation of fedavg and per-fedavg algorithms for dirichlet distributed heterogeneous data},
  author={Reguieg, Hamza and El Hanjri, Mohammed and El Kamili, Mohamed and Kobbane, Abdellatif},
  booktitle={2023 10th International Conference on Wireless Networks and Mobile Communications (WINCOM)},
  pages={1--6},
  year={2023},
  organization={IEEE}
}

@article{imteaj2021survey,
  title={A survey on federated learning for resource-constrained IoT devices},
  author={Imteaj, Ahmed and Thakker, Urmish and Wang, Shiqiang and Li, Jian and Amini, M Hadi},
  journal={IEEE Internet of Things Journal},
  volume={9},
  number={1},
  pages={1--24},
  year={2021},
  publisher={IEEE}
}

@incollection{jain2022hugging,
  title={Hugging face},
  author={Jain, Shashank Mohan},
  booktitle={Introduction to transformers for NLP: With the hugging face library and models to solve problems},
  pages={51--67},
  year={2022},
  publisher={Springer}
}

@article{paszke2019pytorch,
  title={Pytorch: An imperative style, high-performance deep learning library},
  author={Paszke, Adam and Gross, Sam and Massa, Francisco and Lerer, Adam and Bradbury, James and Chanan, Gregory and Killeen, Trevor and Lin, Zeming and Gimelshein, Natalia and Antiga, Luca and others},
  journal={Advances in neural information processing systems},
  volume={32},
  year={2019}
}

@article{xhonneux2024efficient,
  title={Efficient adversarial training in llms with continuous attacks},
  author={Xhonneux, Sophie and Sordoni, Alessandro and G{\"u}nnemann, Stephan and Gidel, Gauthier and Schwinn, Leo},
  journal={Advances in Neural Information Processing Systems},
  volume={37},
  pages={1502--1530},
  year={2024}
}

@article{reddi2020adaptive,
  title={Adaptive federated optimization},
  author={Reddi, Sashank and Charles, Zachary and Zaheer, Manzil and Garrett, Zachary and Rush, Keith and Kone{\v{c}}n{\`y}, Jakub and Kumar, Sanjiv and McMahan, H Brendan},
  journal={arXiv preprint arXiv:2003.00295},
  year={2020}
}

@article{li2020federated,
  title={Federated optimization in heterogeneous networks},
  author={Li, Tian and Sahu, Anit Kumar and Zaheer, Manzil and Sanjabi, Maziar and Talwalkar, Ameet and Smith, Virginia},
  journal={Proceedings of Machine learning and systems},
  volume={2},
  pages={429--450},
  year={2020}
}

@article{baumgart2024not,
  title={Not all federated learning algorithms are created equal: A performance evaluation study},
  author={Baumgart, Gustav A and Shin, Jaemin and Payani, Ali and Lee, Myungjin and Kompella, Ramana Rao},
  journal={arXiv preprint arXiv:2403.17287},
  year={2024}
}

@misc{qwen2.5,
    title = {Qwen2.5: A Party of Foundation Models},
    url = {https://qwenlm.github.io/blog/qwen2.5/},
    author = {{Qwen Team}},
    month = {September},
    year = {2024}
}

@article{qwen2,
      title={Qwen2 Technical Report}, 
      author={An Yang and Baosong Yang and Binyuan Hui and Bo Zheng and Bowen Yu and Chang Zhou and Chengpeng Li and Chengyuan Li and Dayiheng Liu and Fei Huang and Guanting Dong and Haoran Wei and Huan Lin and Jialong Tang and Jialin Wang and Jian Yang and Jianhong Tu and Jianwei Zhang and Jianxin Ma and Jin Xu and Jingren Zhou and Jinze Bai and Jinzheng He and Junyang Lin and Kai Dang and Keming Lu and Keqin Chen and Kexin Yang and Mei Li and Mingfeng Xue and Na Ni and Pei Zhang and Peng Wang and Ru Peng and Rui Men and Ruize Gao and Runji Lin and Shijie Wang and Shuai Bai and Sinan Tan and Tianhang Zhu and Tianhao Li and Tianyu Liu and Wenbin Ge and Xiaodong Deng and Xiaohuan Zhou and Xingzhang Ren and Xinyu Zhang and Xipin Wei and Xuancheng Ren and Yang Fan and Yang Yao and Yichang Zhang and Yu Wan and Yunfei Chu and Yuqiong Liu and Zeyu Cui and Zhenru Zhang and Zhihao Fan},
      journal={arXiv preprint arXiv:2407.10671},
      year={2024}
}

@article{rafailov2023direct,
  title={Direct preference optimization: Your language model is secretly a reward model},
  author={Rafailov, Rafael and Sharma, Archit and Mitchell, Eric and Manning, Christopher D and Ermon, Stefano and Finn, Chelsea},
  journal={Advances in neural information processing systems},
  volume={36},
  pages={53728--53741},
  year={2023}
}

@misc{qwen3.5,
    title  = {{Qwen3.5}: Towards Native Multimodal Agents},
    author = {{Qwen Team}},
    month  = {February},
    year   = {2026},
    url    = {https://qwen.ai/blog?id=qwen3.5}
}

@article{austin2021program,
  title={Program Synthesis with Large Language Models},
  author={Austin, Jacob and Odena, Augustus and Nye, Maxwell and Bosma, Maarten and Michalewski, Henryk and Dohan, David and Jiang, Ellen and Cai, Carrie and Terry, Michael and Le, Quoc and others},
  journal={arXiv preprint arXiv:2108.07732},
  year={2021}
}
\bibliographystyle{iclr2027_conference}

\startcontents[sections]

\onecolumn

\newpage
\appendix
\begin{center}
    {\bf\Large Appendix}
\end{center}

\startcontents[sections]
\printcontents[sections]{l}{1}{\setcounter{tocdepth}{3}}

\newpage
\appendix

\section{Pseudocode of SPEAR}\label{appendix:pseudocode}
In this section, we present detailed pseudocode of the SPEAR algorithm. Specifically, Algorithm \ref{algo:SPEAR-client} outlines the client-side training process of SPEAR for a client $k$, with Algorithm \ref{algo:SPEAR-server} outlining the federated aggregation process at the server with win-trace based FedAvg. 

\subsection{Client-Side SPEAR}
\begin{algorithm}[H]
\caption{SPEAR on client $k$ at federated round $t$}\label{algo:SPEAR-client}
\begin{algorithmic}[1]
\Require Global weights $\boldsymbol{\theta}^{(t)}$, feedback generator $\mathsf{FB}(\cdot)$,
local steps $E$, loss weights $\lambda_{w}, \lambda_{l}$, margin $\mu$, tail tokens $\tau$,
max revision attempts $N$
\State Initialize $\boldsymbol{\theta}_k^{(t)} \leftarrow \boldsymbol{\theta}^{(t)}$
\State $\mathcal{W}_k^{(t)} \leftarrow \emptyset$, $\mathcal{L}_k^{(t)} \leftarrow \emptyset$
\Statex \hrulefill\ \textsc{Phase 1: Interaction} \hrulefill
\For{each $\mathbf{x}_{k,j} \sim \mathcal{D}_k$ sampled for client $k$}
    \State $\mathbf{c}^{(0)}_{k,j} \leftarrow \mathsf{encode}(\mathbf{x}_{k,j})$
    \State Generate $\mathbf{y}^{(0)}_{k,j} \sim p_{\boldsymbol{\theta}_k}(\cdot \mid \mathbf{c}^{(0)}_{k,j})$
    \If{$\mathsf{IsSatisfied}(\mathbf{y}^{(0)}_{k,j},\, \mathbf{x}_{k,j})$}
        \State $\mathcal{W}_k^{(t)} \leftarrow \mathcal{W}_k^{(t)} \cup \{(\mathbf{c}^{(0)}_{k,j},\; \mathbf{y}^{(0)}_{k,j})\}$ \Comment{Win: satisfactory on first try}
    \Else
        \State $\mathsf{fixed} \leftarrow \mathsf{false}$;\quad $\mathbf{y}^{\mathrm{prev}} \leftarrow \mathbf{y}^{(0)}_{k,j}$
        \For{$n = 1$ \textbf{to} $N$}
            \State $\mathbf{f}^{(n)}_{k,j} \leftarrow \mathsf{FB}(\mathbf{x}_{k,j},\, \mathbf{y}^{\mathrm{prev}})$ \Comment{Get feedback on previous attempt}
            \State $\mathbf{c}^{(n)}_{k,j} \leftarrow \mathbf{c}^{(0)}_{k,j} \concat \mathbf{y}^{\mathrm{prev}} \concat \mathbf{f}^{(n)}_{k,j}$
            \State Generate $\mathbf{y}^{(n)}_{k,j} \sim p_{\boldsymbol{\theta}_k}(\cdot \mid \mathbf{c}^{(n)}_{k,j})$
            \If{$\mathsf{IsSatisfied}(\mathbf{y}^{(n)}_{k,j},\, \mathbf{x}_{k,j})$}
                \State $\mathcal{W}_k^{(t)} \leftarrow \mathcal{W}_k^{(t)} \cup \{(\mathbf{c}^{(0)}_{k,j},\; \mathbf{y}^{(n)}_{k,j})\}$ \Comment{Win: self-corrected at attempt $n$}
                \State $\mathcal{L}_k^{(t)} \leftarrow \mathcal{L}_k^{(t)} \cup \{(\mathbf{c}^{(0)}_{k,j},\; \mathbf{y}^{(0)}_{k,j},\; \alpha_{k,j}\!=\!1.0)\}$ \Comment{Lose: original was unsatisfactory}
                \State $\mathsf{fixed} \leftarrow \mathsf{true}$;\quad \textbf{break}
            \EndIf
            \State $\mathbf{y}^{\mathrm{prev}} \leftarrow \mathbf{y}^{(n)}_{k,j}$ \Comment{Feed latest attempt into next revision}
        \EndFor
        \If{\textbf{not} $\mathsf{fixed}$}
            \State $\mathcal{L}_k^{(t)} \leftarrow \mathcal{L}_k^{(t)} \cup \{(\mathbf{c}^{(0)}_{k,j},\; \mathbf{y}^{(0)}_{k,j},\; \alpha_{k,j}\!=\!0.5)\}$ \Comment{Lose: all $N$ attempts failed}
        \EndIf
    \EndIf
\EndFor
\Statex \hrulefill\ \textsc{Phase 2: Training} \hrulefill
\For{$e = 1$ to $E$}
    \State Sample batches from $\mathcal{W}_k^{(t)}$ and $\mathcal{L}_k^{(t)}$
    \State Compute $\ell_{\mathrm{win}}$: negative log-likelihood on win batch
    \State Compute $\ell_{\mathrm{lose}}$: confidence-weighted unlikelihood on lose batch, restricted to $\mathcal{T}_\tau(\mathbf{y}^-)$
    \State $\boldsymbol{\theta}_k^{(t)} \leftarrow \boldsymbol{\theta}_k^{(t)} - \eta\, \nabla\!\left(\lambda_{w}\ell_{\mathrm{win}} + \lambda_{l}\ell_{\mathrm{lose}}\right)$
\EndFor
\State \Return $\boldsymbol{\theta}_k^{(t)}$
\end{algorithmic}
\end{algorithm}

\subsection{Server-side Aggregation}
\begin{algorithm}[H]
\caption{Federated training with SPEAR}\label{algo:SPEAR-server}
\begin{algorithmic}[1]
\Require Initial global weights $\boldsymbol{\theta}^{(0)}$
\For{$t = 0,1,2,\dots$}
    \State Select client subset $\mathcal{K}_t$
    \State Broadcast $\boldsymbol{\theta}^{(t)}$ to clients in $\mathcal{K}_t$
    \State Each client $k \in \mathcal{K}_t$ returns $\boldsymbol{\theta}_k^{(t)}$ after running SPEAR
    \State Calculate number of win traces per client: $w_k^{(t)} = |\mathcal{W}_k^{(t)}|$
    \State Aggregate (FedAvg): $\boldsymbol{\theta}^{(t+1)} \leftarrow \sum_{k \in \mathcal{K}_t} \frac{w_k^{(t)}\, \boldsymbol{\theta}_k^{(t)}}{\sum_{k' \in \mathcal{K}_t} w_{k'}^{(t)}}$
\EndFor
\end{algorithmic}
\end{algorithm}

\section{Additional Experimental Results}\label{app:exps}
In this section, we consider experimental results in addition to those presented in Sec. \ref{sec:experiments}. Unless stated otherwise, the same settings considered in Sec. \ref{sec:experiments} are applied. 

\subsection{Analysis on Assumption of Issues Near Completion End}\label{appendix:near-end-support}

\subsubsection{Ablation over $\tau$}
\begin{table*}[htbp]
\centering
\setlength{\tabcolsep}{7pt}
\renewcommand{\arraystretch}{1.15}
\caption{Ablation study on the hyperparameter $\tau$ for SPEAR on MathMCQA and StrategyQA.}
\label{tab:ablation_tau}
\begin{tabular}{llcc}
\toprule
\textbf{Model} & \boldmath$\tau$ & \textbf{MathMCQA} & \textbf{StrategyQA} \\
\midrule

\multirow{3}{*}{\textbf{Llama3.2-3B}}
& 8   & \textbf{55.89$\pm$}0.65 & 66.47$\pm$2.93 \\
& 64  & 55.13$\pm$0.77 & 66.42$\pm$1.57 \\
& 128 & 54.43$\pm$1.08 & \textbf{67.39$\pm$}2.38 \\
\midrule

\multirow{3}{*}{\textbf{Qwen2.5-1.5B}}
& 8   & \textbf{57.27$\pm$}0.41 & \textbf{64.39$\pm$}2.48 \\
& 64  & 54.61$\pm$4.17 & 64.00$\pm$1.49 \\
& 128 & 56.85$\pm$1.18 & 62.20$\pm$4.63 \\
\bottomrule
\end{tabular}
\end{table*}

We consider if the assumption that unsatisfactory tokens concentrate near the end of the completion is valid by conducting an ablation study on $\tau$. We consider only MathMCQA and StrategyQA as HellaSwag and ARC-Challenge do not require the lengthy chain-of-thought (CoT) reasoning chains that MathMCQA and StrategyQA do. From the results in Table \ref{tab:ablation_tau}, we can see that the motivation of targeting the tail of the sequence holds. With the exception of StrategyQA on Llama3.2-3B, we see that targeting $8$ tokens leads to better performance in comparison to $64$ and $128$ tokens. In terms of why StrategyQA with Llama3.2-3B performs the best, it could be that the model is producing extremely lengthy outputs, thereby requiring more tokens to be targeted.  

\subsubsection{Initial Error Localization}

\begin{table}[htbp]
\centering
\setlength{\tabcolsep}{7pt}
\renewcommand{\arraystretch}{1.15}
\caption{Location of the first error across incorrect responses on StrategyQA and MathMCQA.}
\label{tab:first_error_location}
\begin{tabular}{lcc}
\toprule
\textbf{Error Quartile Location} & \textbf{StrategyQA} & \textbf{MathMCQA} \\
\midrule

$0$--$25\%$   & 23.6\% & 19.6\% \\
$25$--$50\%$  & 19.9\% & 23.1\% \\
$50$--$75\%$  & 26.1\% & 24.1\% \\
$75$--$100\%$ & \textbf{30.4\%} & \textbf{33.2\%} \\

\bottomrule
\end{tabular}
\end{table}

Outside of the ablation study over $\tau$, we additionally consider where unsatisfactory steps initially appear within lose trace completions. For this, we take Qwen2.5-7B and instruct an LLM Judge (Qwen2.5-14B-Instruct) to find the first occurrence of an error within 200 incorrect answers across MathMCQA and StrategyQA. Based on the results from Table \ref{tab:first_error_location}, we can note that the majority of issues initially occur in the bottom half of the completion, with the error percentage most prevalent in the bottom $25\%$. This indicates a consistent tail skew, further validating our choice to target the last $\tau$ tokens.

\subsubsection{Qualitative Example on MathMCQA}

Lastly, we consider a qualitative example of an incorrect output produced from Qwen2.5-7B on MathMCQA. In the example, we see for the initial model response that the error occurs near the end, where it treats the two isolated values as the entire solution set, ignoring the $+2\pi t$ periodic family entirely. Moreover, we note that the model has corrected not only its answer, but also the reasoning steps required to get to the correct solution, indicating that the model did not merely flip the label and get lucky in its selection.

\begin{tcolorbox}[
    colback=gray!4,
    colframe=gray!55,
    title=\textbf{Problem},
    fonttitle=\bfseries
]

Solve
\[
\cos(4x)=-\frac{1}{2},
\]
given that all solutions can be expressed as
\[
x=\frac{(kn\pm1)\pi}{6},
\qquad n\in\mathbb{Z},
\]
and determine the positive value of $k$.

\medskip
\noindent
\textbf{Answer choices:}
A. $2$, B. $3$, C. $6$, D. $4$

\medskip
\noindent
\textbf{Reference answer: B} ($k=3$)

\end{tcolorbox}

\medskip

\begin{tcolorbox}[
    colback=red!5,
    colframe=red!45,
    title=\textbf{Initial model response --- Incorrect},
    fonttitle=\bfseries
]

We know that
\[
\cos(4x)=-\frac{1}{2}
\]
when
\[
4x=\frac{2\pi}{3}
\qquad \text{or} \qquad
4x=\frac{4\pi}{3}.
\]

Solving for $x$, we obtain
\[
x=\frac{\pi}{6}
\qquad \text{or} \qquad
x=\frac{\pi}{3}.
\]

The solutions have a common factor of $\pi/6$, so they can be written
in the form
\[
\frac{(kn\pm1)\pi}{6}.
\]

Comparing with the required form gives $k=2$.

\textbf{Answer: A}

\end{tcolorbox}

\medskip

\noindent
\textbf{Error in the initial reasoning.}
The initial response correctly finds the two principal solutions, but
it incorrectly treats them as the complete solution set. In particular,
it omits the $2\pi t$ periodicity of cosine.

\medskip

\begin{tcolorbox}[
    colback=green!5,
    colframe=green!45!black,
    title=\textbf{Model response after feedback --- Corrected},
    fonttitle=\bfseries
]

If
\[
\cos(4x)=-\frac{1}{2},
\]
then, for some integer $t$,
\[
4x
=
\frac{2\pi}{3}+2\pi t
=
\frac{2(3t+1)\pi}{3},
\]
or
\[
4x
=
\frac{4\pi}{3}+2\pi t
=
\frac{2(3t+2)\pi}{3}.
\]

Dividing by $4$ gives
\[
x=\frac{(3t+1)\pi}{6}
\qquad \text{or} \qquad
x=\frac{(3t+2)\pi}{6}.
\]

Thus, $k=3$.

\textbf{Answer: B}

\end{tcolorbox}

\medskip

\noindent
\textbf{Why the corrected reasoning is valid.}
The revised response explicitly restores the missing periodic terms and
propagates them through the derivation. The second solution family can
be reindexed as
\[
3t+2 = 3(t+1)-1.
\]
Therefore, the two families can be written together as
\[
x=\frac{(3n\pm1)\pi}{6},
\qquad n\in\mathbb{Z},
\]
which gives $k=3$.

\subsection{High Number of Clients with High Heterogeneity}

\begin{table*}[htbp]
\centering
\small
\setlength{\tabcolsep}{4pt} 
\renewcommand{\arraystretch}{1.15}
\caption{Performance Comparison of SPEAR vs. baselines across benchmark datasets on Llama3.2-3B and Qwen2.5-1.5B with 100 total clients and heterogeneity of \texttt{Dirichlet(0.3)}.}
\label{tab:100clients_results}
\begin{tabular}{llccccc}
\toprule
\textbf{Model} & \textbf{Algorithm} & \textbf{ARC-Challenge} & \textbf{HellaSwag} & \textbf{MathMCQA} & \textbf{StrategyQA} & \textbf{Avg.} \\
\midrule

\multirow{4}{*}{\textbf{Llama3.2-3B}}
& Feedback SFT & 65.67$\pm$0.98 & 83.07$\pm$1.63 & 41.58$\pm$0.36 & 51.53$\pm$0.01 & 60.46 \\
& GRPO        & 66.68$\pm$0.64 & 66.84$\pm$2.03 & 37.01$\pm$8.73 & 49.34$\pm$0.87 & 54.97 \\
& OPSD        & 25.10$\pm$0.63 & 24.57$\pm$1.88 & 0.00$\pm$0.00 & 0.00$\pm$0.00 & 12.42 \\
& RLTF-SD  & 61.86$\pm$2.56 & 42.19$\pm$16.46 & 40.29$\pm$0.81 & 25.76$\pm$13.26 & 42.53 \\
\rowcolor{blue!10}
& \textbf{SPEAR (Ours)} & 68.86$\pm$0.03 & 85.45$\pm$1.20 & 48.87$\pm$1.81 & 54.01$\pm$1.74 & \textbf{64.30} \\
\midrule

\multirow{4}{*}{\textbf{Qwen2.5-1.5B}}
& Feedback SFT & 73.76$\pm$1.15 & 76.15$\pm$1.14 & 44.90$\pm$0.16 & 48.47$\pm$0.12 & 60.82 \\
& GRPO       & 76.45$\pm$0.60 & 78.32$\pm$1.06 & 48.95$\pm$1.93 & 50.22$\pm$1.75 & 63.48 \\
& OPSD      & 24.75$\pm$0.27 & 26.49$\pm$0.04 & 0.00$\pm$0.00 & 0.00$\pm$0.00 & 12.81 \\
& RLTF-SD  & 66.34$\pm$8.75 & 71.41$\pm$7.88 & 46.59$\pm$5.89 & 50.00$\pm$1.53 & 58.59 \\
\rowcolor{blue!10}
& \textbf{SPEAR (Ours)} & 70.77$\pm$3.53 & 79.34$\pm$1.79 & 47.24$\pm$5.86 & 62.53$\pm$2.84 & \textbf{64.97} \\
\bottomrule
\end{tabular}
\end{table*}

We consider how SPEAR performs in comparison to the evaluated baselines with two additional dimensions common in FL settings: a large number of clients within the network, and high data heterogeneity between the clients. To model this, we consider a setting with 100 clients and data splitting across clients via a Dirichlet parameter of $0.3$ \citep{green2001modelling, reguieg2023comparative}. For this experiment, 3 clients are aggregated every FL round.

From the results in Table \ref{tab:100clients_results}, we note the performance advantages exhibited by SPEAR in Table \ref{tab:main_results} persist even when considering the additional dimensions of higher client cardinality and high heterogeneity in data between local clients. We can also see this consistency in performance persists across models, with it being the best performing method across both Llama3.2-3B and Qwen2.5-1.5B. SPEAR's consistency is maintained on all datasets, but the same can't be said for some baselines, such as Feedback SFT, RLTF-SD, or GRPO (e.g., 25.76 on StrategyQA for RLTF-SD). Overall, this indicates that SPEAR is robust to changes in the number of clients and data heterogeneity considerations in comparison to the baselines.

\subsection{Fine-tuning on Only Last $H$ Layers}

\begin{table*}[htbp]
\centering
\setlength{\tabcolsep}{4pt} 
\renewcommand{\arraystretch}{1.15}
\caption{Performance Comparison of SPEAR vs. baselines across benchmark datasets on Qwen2.5-1.5B, fine-tuned on the last $H=12$ layers.}
\label{tab:H-layers_results}
\begin{tabular}{lccccc}
\toprule
\textbf{Algorithm} & \textbf{ARC-Challenge} & \textbf{HellaSwag} & \textbf{MathMCQA} & \textbf{StrategyQA} & \textbf{Avg.} \\
\midrule

Feedback SFT ($H=12$) & 63.82$\pm$3.11 & 71.40$\pm$1.34 & 47.20$\pm$2.64 & 50.00$\pm$1.53 & 58.11 \\
GRPO ($H=12$) & 71.50$\pm$0.09 & 70.81$\pm$1.20 & 41.65$\pm$2.73 & 47.74$\pm$0.73 & 57.93 \\
OPSD ($H=12$) & 24.32$\pm$0.01 & 25.41$\pm$0.33 & 0.00$\pm$0.00 & 0.00$\pm$0.00 & 12.43 \\
RLTF-SD ($H = 12$) & 71.20$\pm$0.04 & 69.62$\pm$2.88 & 38.81$\pm$3.70 & 48.33$\pm$0.15 & 56.99 \\
\rowcolor{blue!10}
\textbf{SPEAR ($H = 12$)} & 69.75$\pm$0.38 & 75.73$\pm$1.05 & 46.08$\pm$1.22 & 57.06$\pm$7.57 & \textbf{62.16} \\

\bottomrule
\end{tabular}
\end{table*}

We now consider a scenario more constrained than those presented in Sec. \ref{sec:experiments}, where only the last $H$ layers of the model are fine-tuned with LoRA. We consider $H=12$ on Qwen2.5-1.5B. The results are presented in Table \ref{tab:H-layers_results}.

Firstly, we can note that SPEAR remains the highest-performing method, with the highest average score exhibited in comparison to all baselines. We note that unlike the baselines, SPEAR is always consistently well performing across all datasets, in comparison to GRPO (e.g., only $41.65\%$ on MathMCQA) or Feedback SFT (e.g., only $63.82\%$ on ARC-Challenge). Overall, this means even in a scenario where we only fine-tune on a limited number of layers, SPEAR is still the superior algorithm in comparison to the baselines. 

\subsection{Attempting Revision Only Once}\label{appendix:n=1_revision}

\begin{table}[htbp]
\centering
\small
\setlength{\tabcolsep}{4pt} 
\renewcommand{\arraystretch}{1.15}
\caption{Performance Comparison of SPEAR vs. baselines with $N=1$ revision attempt.}
\label{appendix:Nequals1_results}
\begin{tabular}{llccccc}
\toprule
\textbf{Model} & \textbf{Algorithm} & \textbf{ARC-Challenge} & \textbf{HellaSwag} & \textbf{MathMCQA} & \textbf{StrategyQA} & \textbf{Avg.} \\
\midrule

\multirow{4}{*}{\textbf{Llama3.2-3B}}
& Feedback SFT & 67.95$\pm$1.45 & 84.60$\pm$0.98 & 48.60$\pm$9.89 & 49.49$\pm$1.76 & 62.66 \\
& GRPO        & 65.70$\pm$3.48 & 75.99$\pm$1.53 & 40.20$\pm$14.89 & 49.34$\pm$1.51 & 57.81 \\
& OPSD     & 24.69$\pm$1.55 & 25.03$\pm$0.64 & 0.00$\pm$0.00 & 0.00$\pm$0.00 & 12.43 \\
& RLTF-SD & 63.28$\pm$3.58 & 60.05$\pm$12.30 & 44.70$\pm$7.92 & 51.58$\pm$0.08 & 54.90 \\
\rowcolor{blue!10}
& \textbf{SPEAR ($N = 1$)} & 69.91$\pm$0.20 & 86.86$\pm$0.18 & 53.33$\pm$1.87 & 66.28$\pm$1.28 & \textbf{69.10} \\
\midrule

\multirow{4}{*}{\textbf{Qwen2.5-1.5B}}
& Feedback SFT & 73.89$\pm$3.38 & 76.14$\pm$3.61 & 54.93$\pm$2.04 & 50.51$\pm$1.76 & 63.87 \\
& GRPO      & 76.30$\pm$0.17 & 79.00$\pm$1.65 & 50.87$\pm$3.86 & 51.38$\pm$6.44 & 64.39 \\
& OPSD      & 25.03$\pm$1.23 & 24.92$\pm$0.69 & 0.00$\pm$0.00 & 0.00$\pm$0.00 & 12.49 \\
& RLTF-SD & 75.03$\pm$0.40 & 60.65$\pm$20.78 & 44.70$\pm$7.41 & 54.29$\pm$10.08 & 58.67 \\
\rowcolor{blue!10}
& \textbf{SPEAR ($N = 1$)} & 76.68$\pm$0.56 & 82.44$\pm$0.75 & 56.51$\pm$1.57 & 56.72$\pm$6.44 & \textbf{68.09} \\
\bottomrule
\end{tabular}
\end{table}

\begin{figure}[htbp]
\centering

\begin{subfigure}[htbp]{0.245\textwidth}
    \centering
    \includegraphics[width=\linewidth]{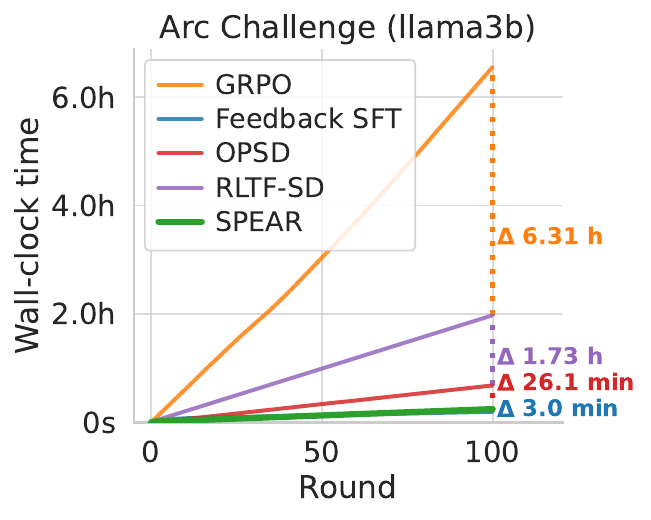}
    \caption{ARC-Challenge}
\end{subfigure}
\begin{subfigure}[htbp]{0.245\textwidth}
    \centering
    \includegraphics[width=\linewidth]{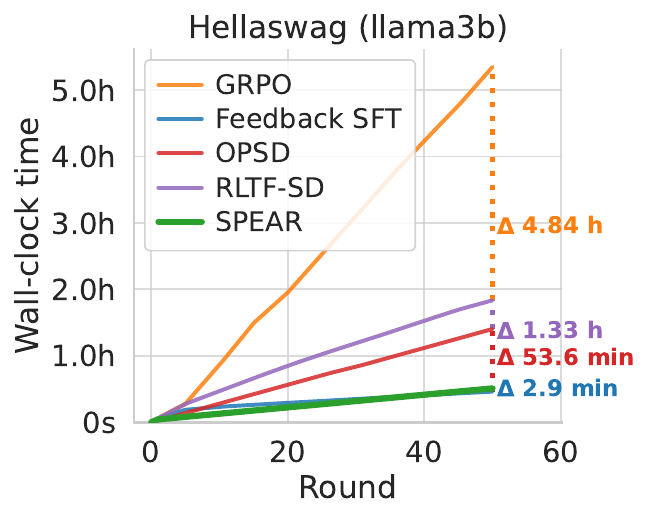}
    \caption{HellaSwag}
\end{subfigure}
\begin{subfigure}[htbp]{0.245\textwidth}
    \centering
    \includegraphics[width=\linewidth]{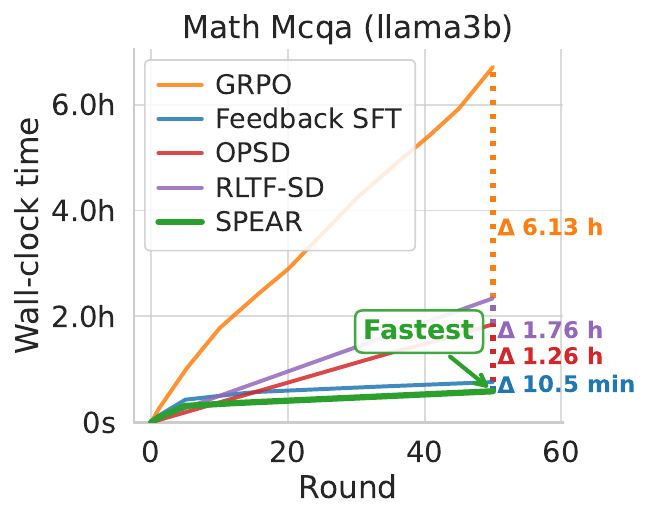}
    \caption{MathMCQA}
\end{subfigure}
\begin{subfigure}[htbp]{0.245\textwidth}
    \centering
    \includegraphics[width=\linewidth]{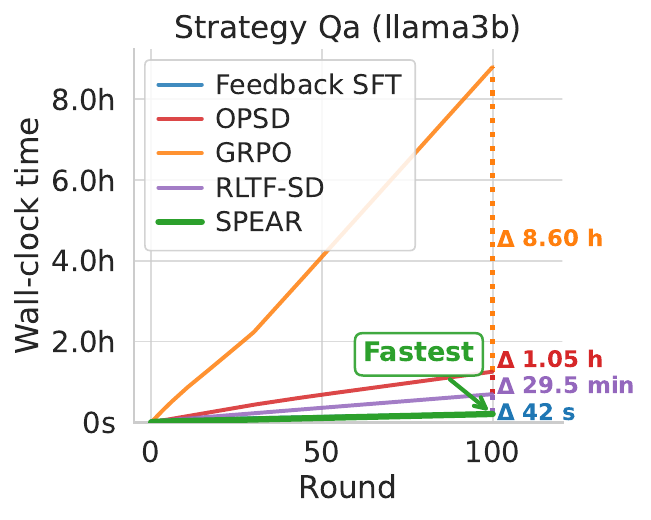}
    \caption{StrategyQA}
\end{subfigure}

\vspace{1.0em}

{\small \textbf{(a) Llama3.2-3B}}

\vspace{1.5em}

\begin{subfigure}[htbp]{0.245\textwidth}
    \centering
    \includegraphics[width=\linewidth]{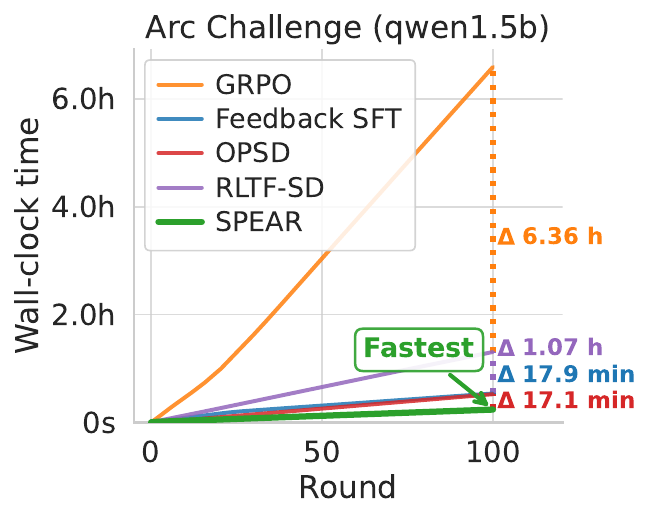}
    \caption{ARC-Challenge}
\end{subfigure}
\begin{subfigure}[htbp]{0.245\textwidth}
    \centering
    \includegraphics[width=\linewidth]{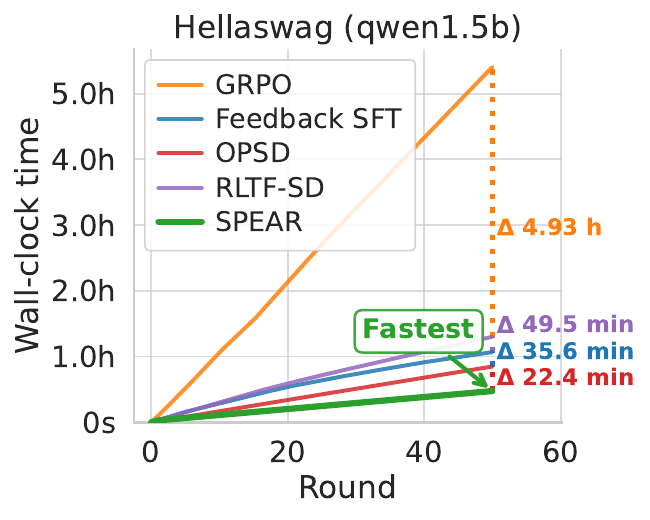}
    \caption{HellaSwag}
\end{subfigure}
\begin{subfigure}[htbp]{0.245\textwidth}
    \centering
    \includegraphics[width=\linewidth]{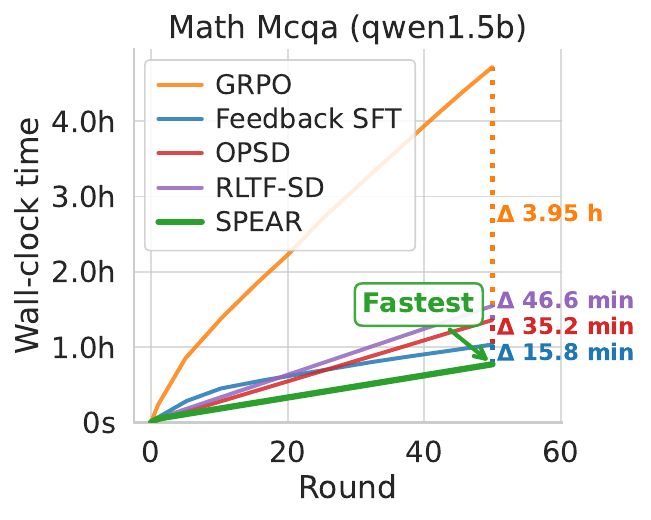}
    \caption{MathMCQA}
\end{subfigure}
\begin{subfigure}[htbp]{0.245\textwidth}
    \centering
    \includegraphics[width=\linewidth]{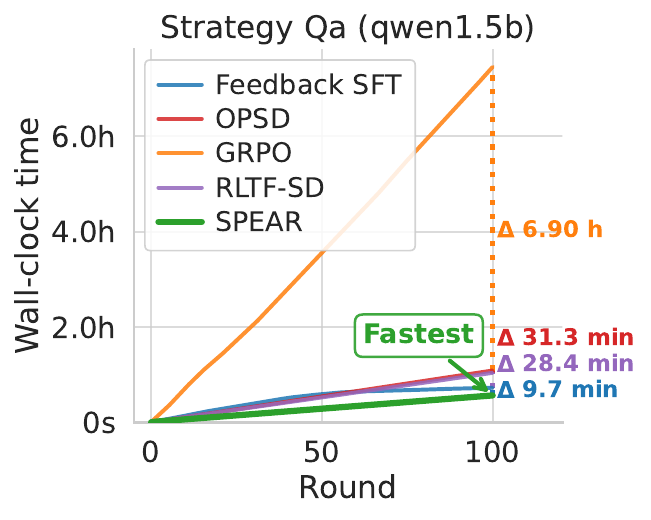}
    \caption{StrategyQA}
\end{subfigure}

\vspace{1.0em}

{\small \textbf{(b) Qwen2.5-1.5B}}


\caption{Cumulative training time of each algorithm on (a) Llama3.2-3B and (b) Qwen2.5-1.5B models for $N=1$ maximum revision on SPEAR.}
\label{appendix:n=1_time_curves}

\vspace{-3mm}
\end{figure}

In Sec. \ref{sec:experiments}, since we consider $N=2$ attempts at revision, we also seek to see if $N=1$ still maintains superior performance in comparison to the baselines. We consider both the performance of SPEAR with $N=1$ and the run time in comparison to the baselines. Accuracy results are presented in Table \ref{appendix:Nequals1_results}.

Firstly, from Table \ref{appendix:Nequals1_results}, we note that even with $N=1$, the performance of SPEAR remains superior to the evaluated baselines on all settings. Likewise, the performance across all datasets and models remains high, meaning that SPEAR is robust to settings where the number of revision attempts is set low. 

We also note the training time of doing SPEAR for $N=1$ revisions in comparison to the baselines in Fig. \ref{appendix:n=1_time_curves}. We can see similarly to Fig. \ref{fig:training_curves} in Sec. \ref{sec:experiments}, SPEAR is in a majority of scenarios the fastest algorithm. When it isn't (ARC-Challenge and HellaSwag on Llama3.2-3B, with Feedback SFT being the best), it is very close to the fastest, with a difference of only a couple of minutes. Moreover, as seen in Table \ref{appendix:Nequals1_results}, the slightly slower training time is made up for with performance gains in accuracy. When further comparing with Fig. \ref{fig:training_curves} in Sec. \ref{sec:experiments}, we can see the margin of cumulative time difference between the slower baselines increases, which intuitively makes sense as when the number of generations is lower, less overall tokens are generated. Overall, this means that in settings where training must be conducted in a quick manner, SPEAR is able to accommodate only one attempt at revision without significant change to the overall performance after multiple FL rounds.

\subsection{Additional Comparisons of $N=1$ vs. $N=2$}

\begin{figure}[htbp]
\centering

\begin{subfigure}[htbp]{0.245\textwidth}
    \centering
    \includegraphics[width=\linewidth]{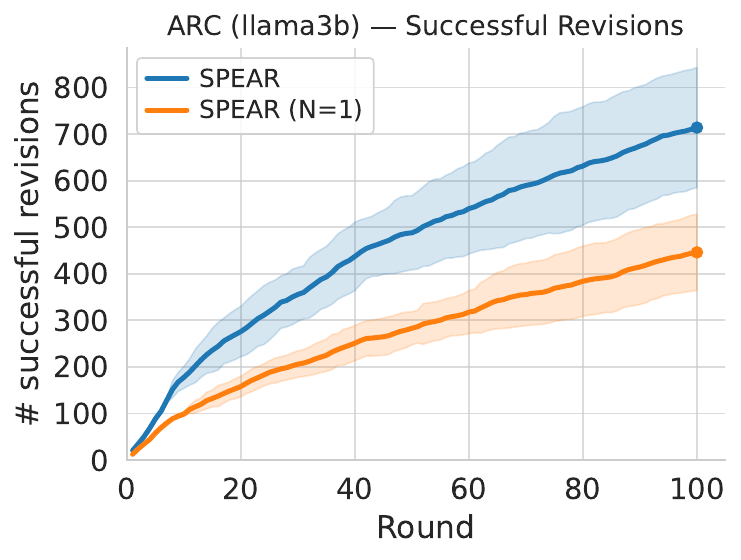}
    \caption{ARC-Challenge}
\end{subfigure}
\begin{subfigure}[htbp]{0.245\textwidth}
    \centering
    \includegraphics[width=\linewidth]{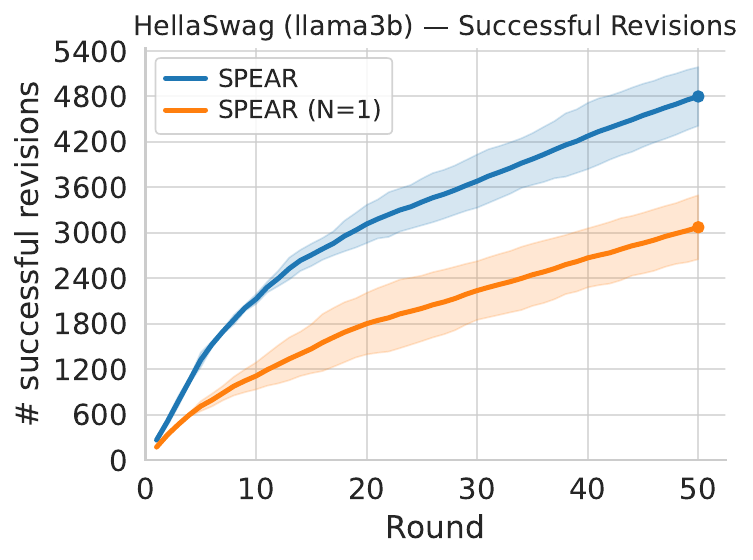}
    \caption{HellaSwag}
\end{subfigure}
\begin{subfigure}[htbp]{0.245\textwidth}
    \centering
    \includegraphics[width=\linewidth]{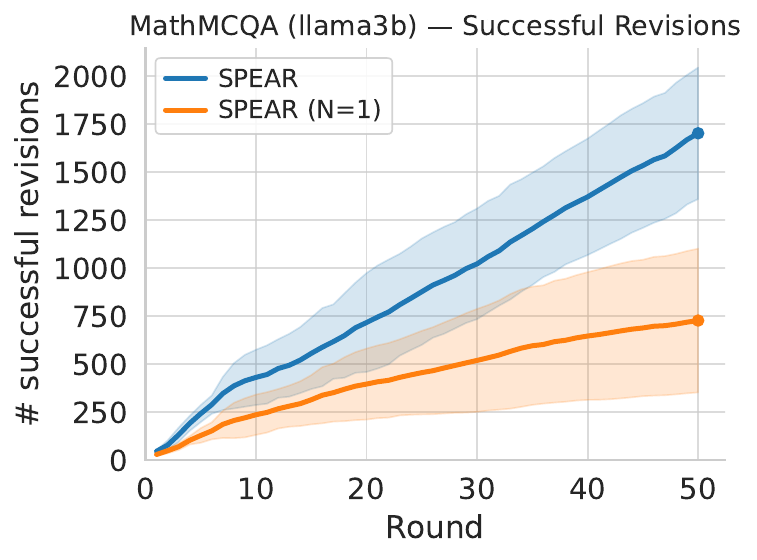}
    \caption{MathMCQA}
\end{subfigure}
\begin{subfigure}[htbp]{0.245\textwidth}
    \centering
    \includegraphics[width=\linewidth]{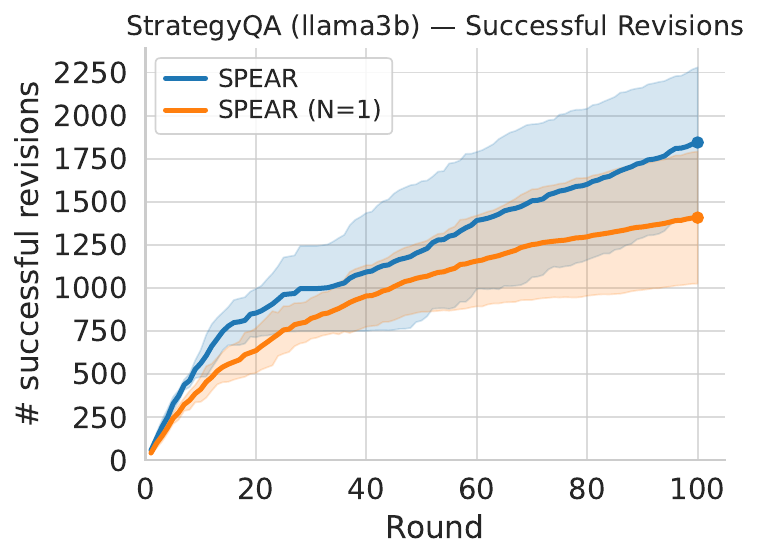}
    \caption{StrategyQA}
\end{subfigure}

\vspace{1.0em}

{\small \textbf{(a) Llama3.2-3B}}

\vspace{1.5em}

\begin{subfigure}[htbp]{0.245\textwidth}
    \centering
    \includegraphics[width=\linewidth]{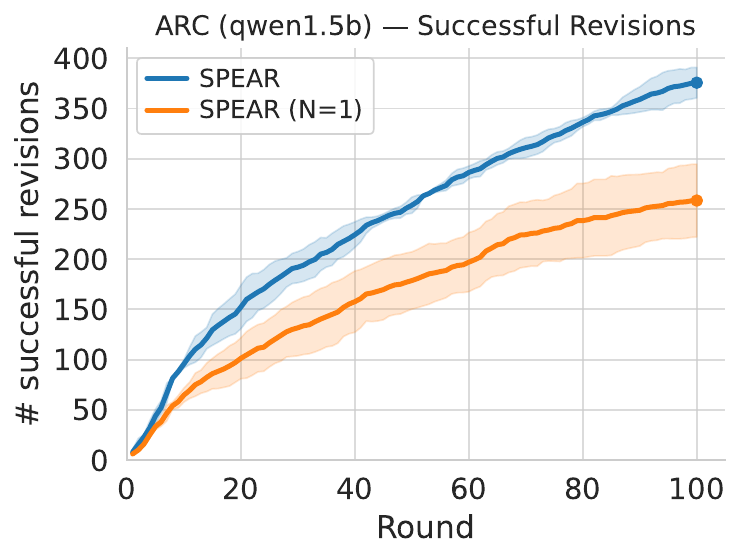}
    \caption{ARC-Challenge}
\end{subfigure}
\begin{subfigure}[htbp]{0.245\textwidth}
    \centering
    \includegraphics[width=\linewidth]{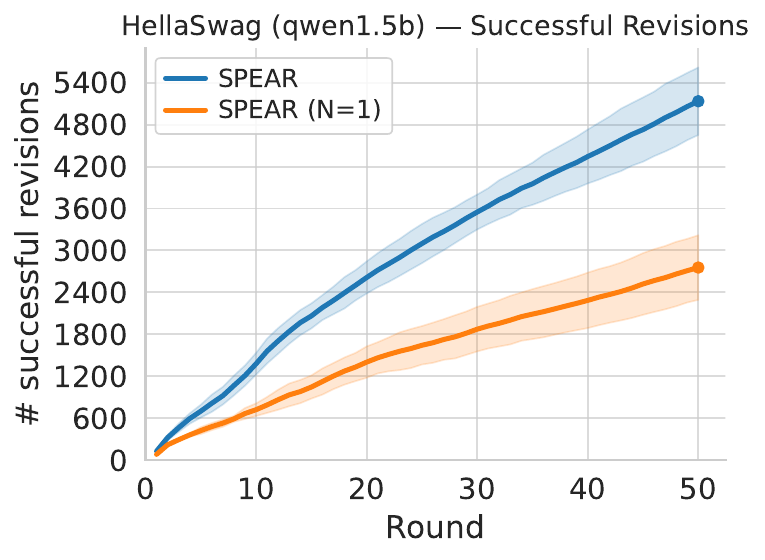}
    \caption{HellaSwag}
\end{subfigure}
\begin{subfigure}[htbp]{0.245\textwidth}
    \centering
    \includegraphics[width=\linewidth]{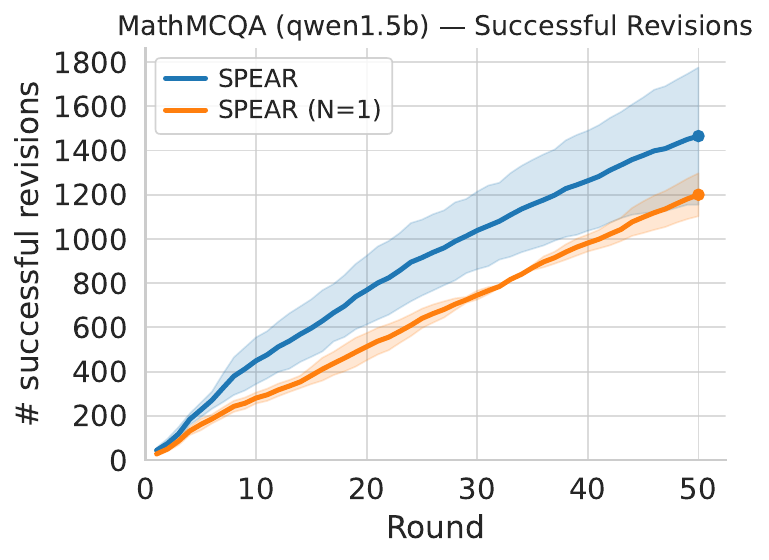}
    \caption{MathMCQA}
\end{subfigure}
\begin{subfigure}[htbp]{0.245\textwidth}
    \centering
    \includegraphics[width=\linewidth]{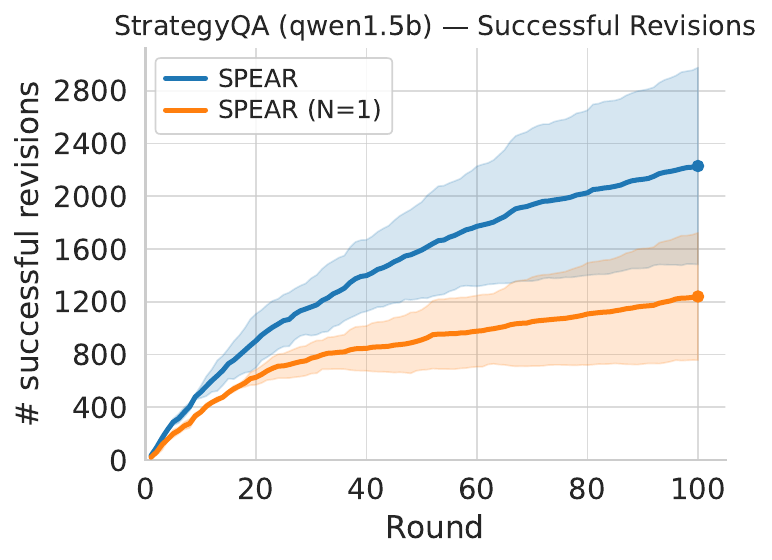}
    \caption{StrategyQA}
\end{subfigure}

\vspace{1.0em}

{\small \textbf{(b) Qwen2.5-1.5B}}

\vspace{0.5em}

\caption{Cumulative successful revisions of SPEAR with $N=1$ and $N=2$ (default) on (a) Llama3.2-3B and (b) Qwen2.5-1.5B models.}
\label{appendix:cumulative_successful_revisions}

\end{figure}

\begin{figure}[htbp]
\centering

\begin{subfigure}[htbp]{0.245\textwidth}
    \centering
    \includegraphics[width=\linewidth]{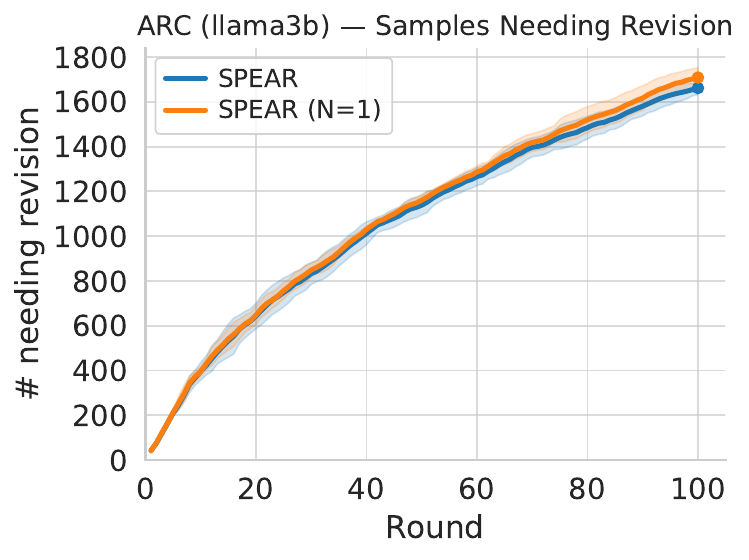}
    \caption{ARC-Challenge}
\end{subfigure}
\begin{subfigure}[htbp]{0.245\textwidth}
    \centering
    \includegraphics[width=\linewidth]{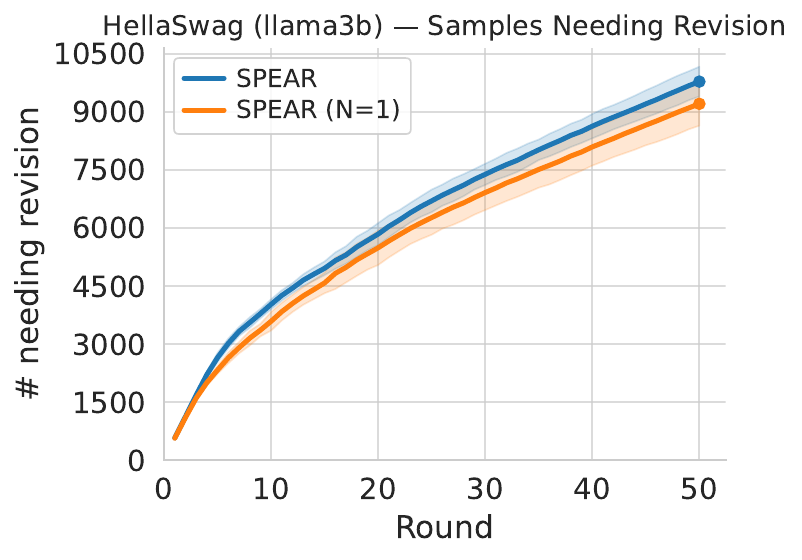}
    \caption{HellaSwag}
\end{subfigure}
\begin{subfigure}[htbp]{0.245\textwidth}
    \centering
    \includegraphics[width=\linewidth]{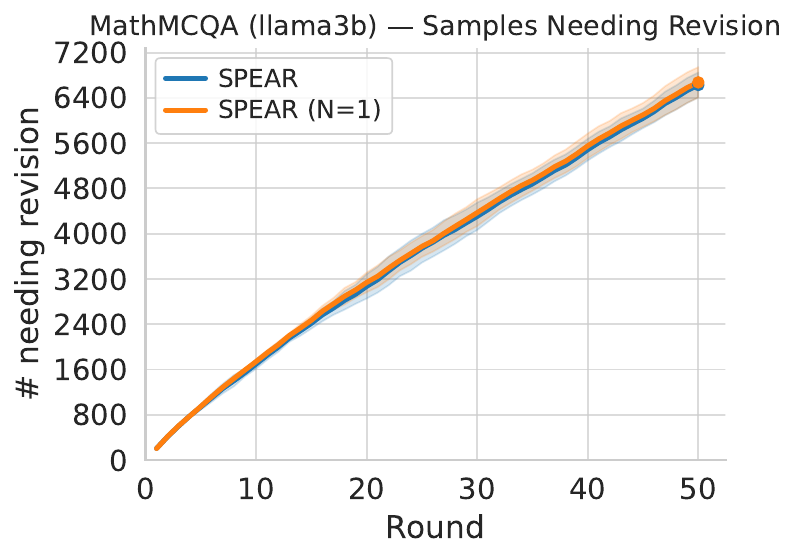}
    \caption{MathMCQA}
\end{subfigure}
\begin{subfigure}[htbp]{0.245\textwidth}
    \centering
    \includegraphics[width=\linewidth]{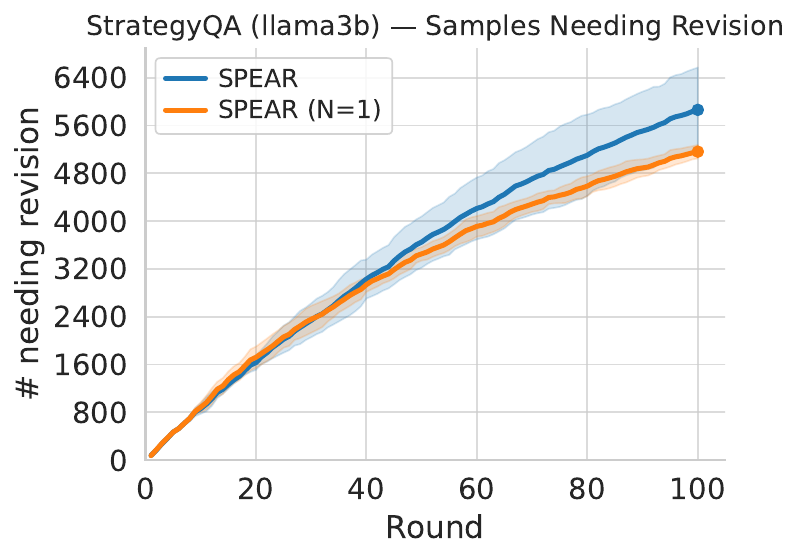}
    \caption{StrategyQA}
\end{subfigure}

\vspace{1.0em}

{\small \textbf{(a) Llama3.2-3B}}

\vspace{1.5em}

\begin{subfigure}[htbp]{0.245\textwidth}
    \centering
    \includegraphics[width=\linewidth]{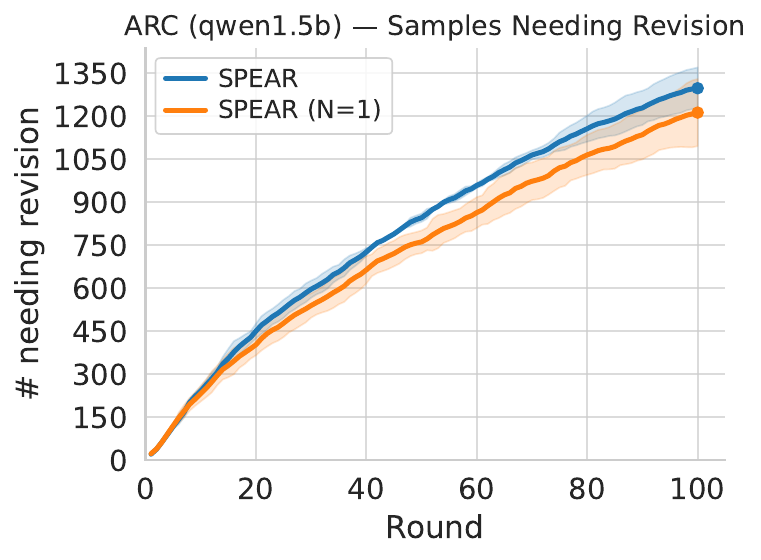}
    \caption{ARC-Challenge}
\end{subfigure}
\begin{subfigure}[htbp]{0.245\textwidth}
    \centering
    \includegraphics[width=\linewidth]{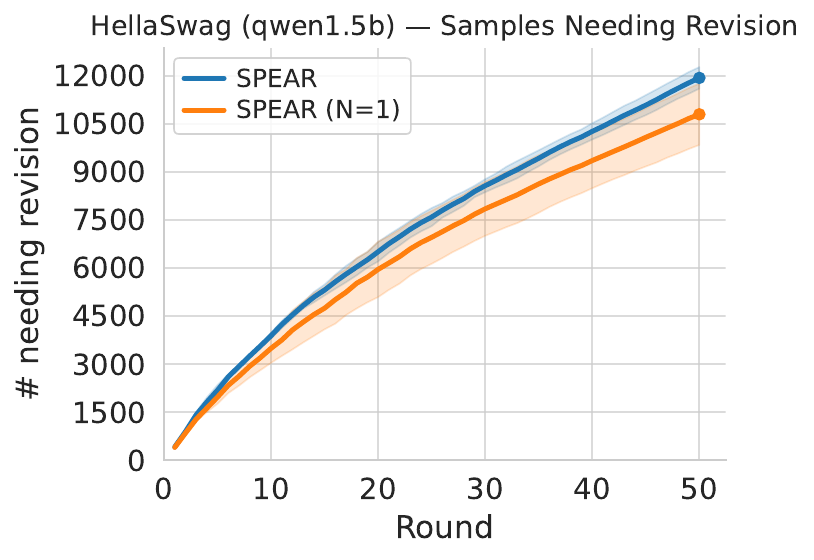}
    \caption{HellaSwag}
\end{subfigure}
\begin{subfigure}[htbp]{0.245\textwidth}
    \centering
    \includegraphics[width=\linewidth]{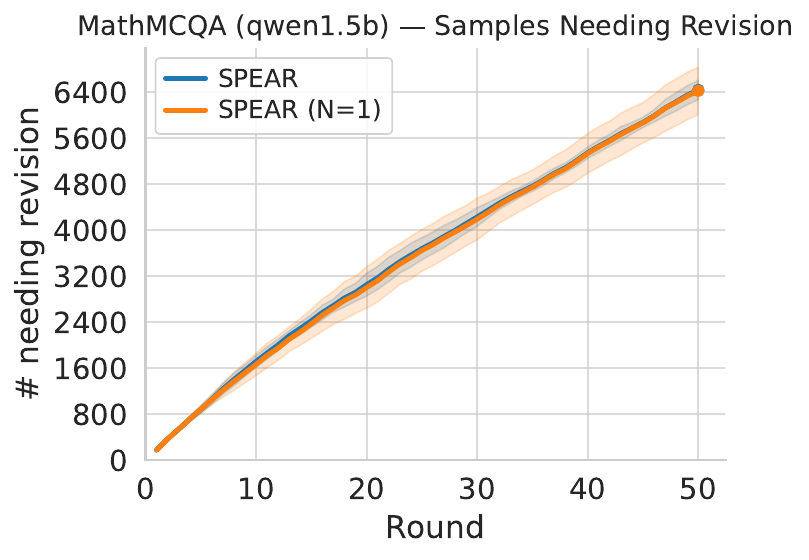}
    \caption{MathMCQA}
\end{subfigure}
\begin{subfigure}[htbp]{0.245\textwidth}
    \centering
    \includegraphics[width=\linewidth]{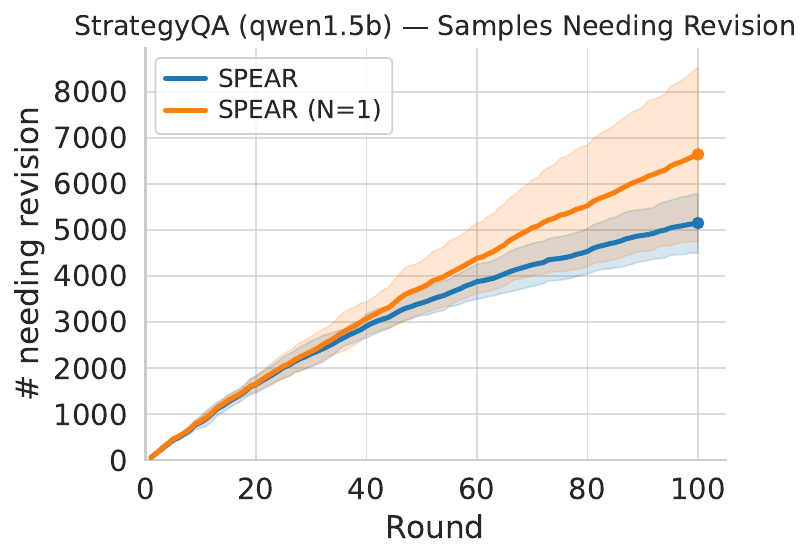}
    \caption{StrategyQA}
\end{subfigure}

\vspace{1.0em}

{\small \textbf{(b) Qwen2.5-1.5B}}

\vspace{0.5em}

\caption{Cumulative number of samples needing revision of SPEAR with $N=1$ and $N=2$ (default) on (a) Llama3.2-3B and (b) Qwen2.5-1.5B models.}
\label{appendix:cumulative_samples_needing_revision}

\end{figure}

\begin{figure}[htbp]
\centering

\begin{subfigure}[htbp]{0.245\textwidth}
    \centering
    \includegraphics[width=\linewidth]{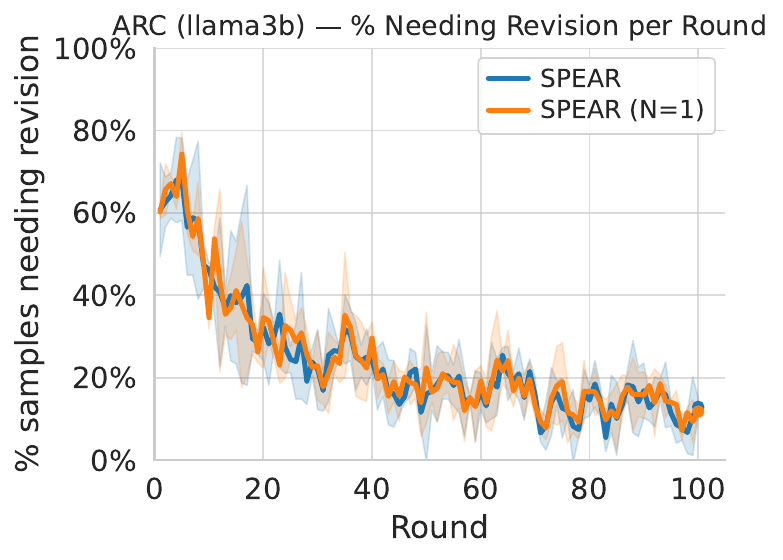}
    \caption{ARC-Challenge}
\end{subfigure}
\begin{subfigure}[htbp]{0.245\textwidth}
    \centering
    \includegraphics[width=\linewidth]{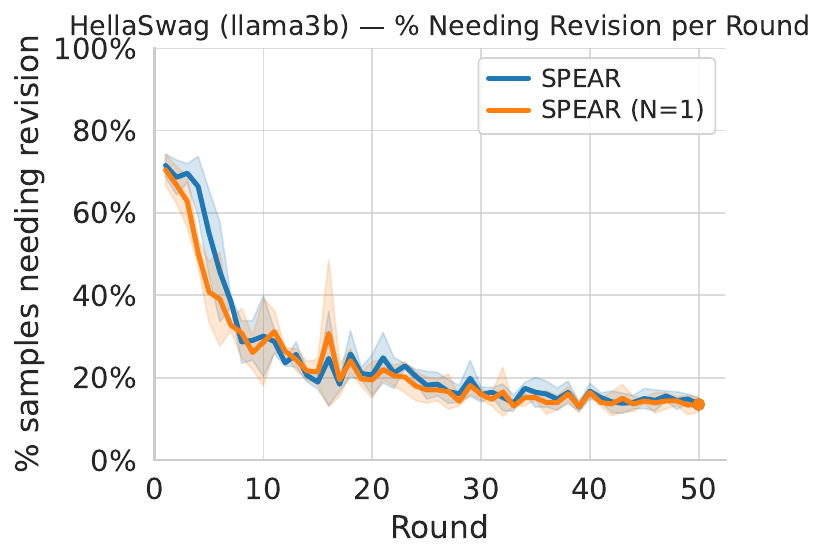}
    \caption{HellaSwag}
\end{subfigure}
\begin{subfigure}[htbp]{0.245\textwidth}
    \centering
    \includegraphics[width=\linewidth]{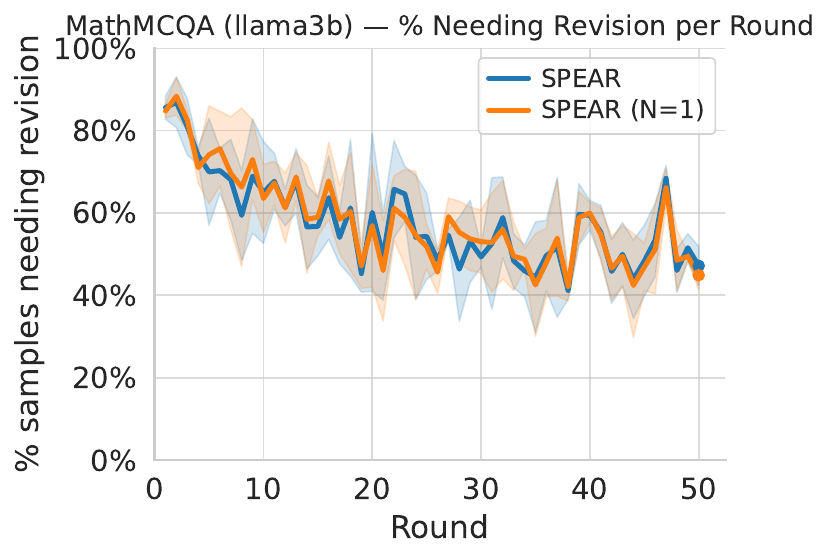}
    \caption{MathMCQA}
\end{subfigure}
\begin{subfigure}[htbp]{0.245\textwidth}
    \centering
    \includegraphics[width=\linewidth]{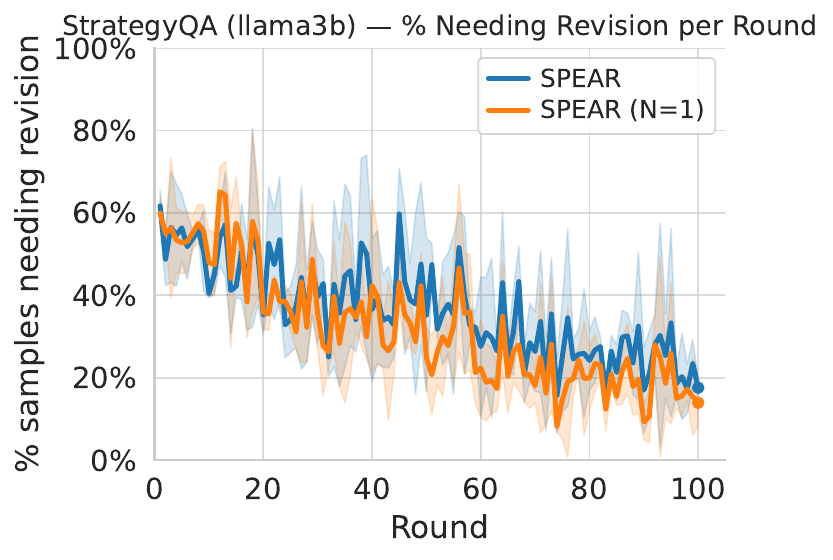}
    \caption{StrategyQA}
\end{subfigure}

\vspace{1.0em}

{\small \textbf{(a) Llama3.2-3B}}

\vspace{1.5em}

\begin{subfigure}[htbp]{0.245\textwidth}
    \centering
    \includegraphics[width=\linewidth]{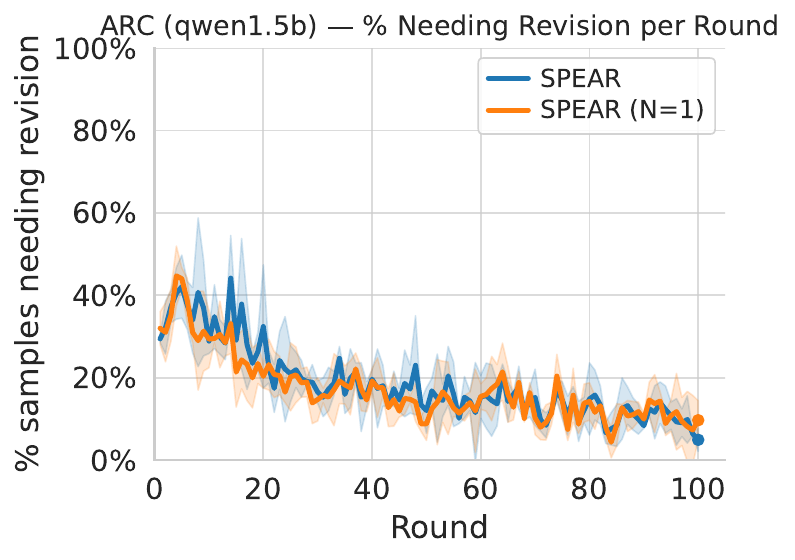}
    \caption{ARC-Challenge}
\end{subfigure}
\begin{subfigure}[htbp]{0.245\textwidth}
    \centering
    \includegraphics[width=\linewidth]{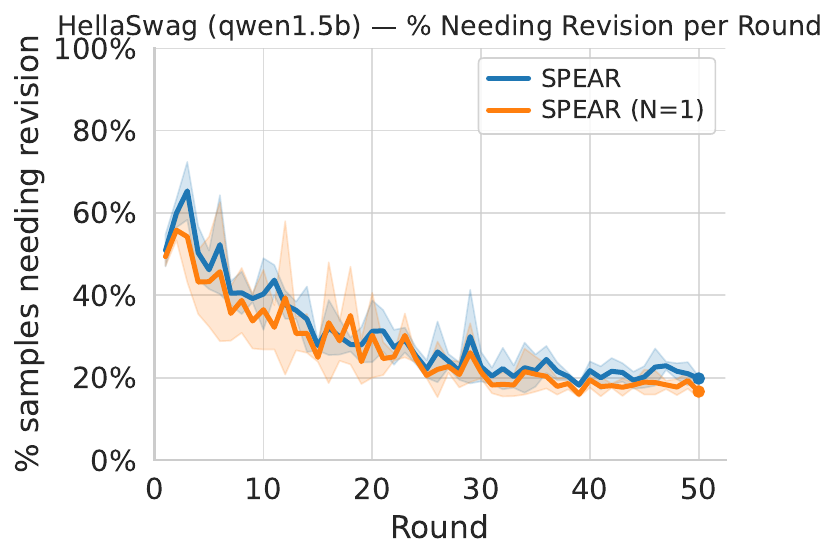}
    \caption{HellaSwag}
\end{subfigure}
\begin{subfigure}[htbp]{0.245\textwidth}
    \centering
    \includegraphics[width=\linewidth]{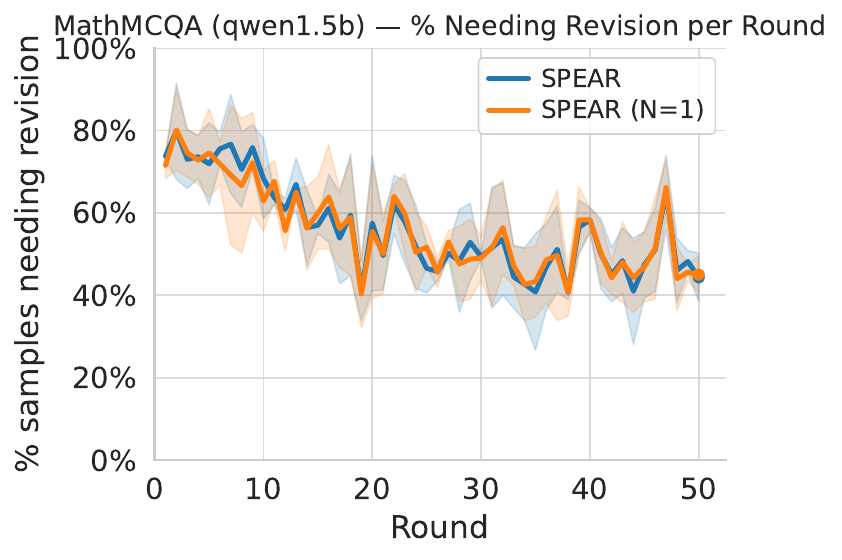}
    \caption{MathMCQA}
\end{subfigure}
\begin{subfigure}[htbp]{0.245\textwidth}
    \centering
    \includegraphics[width=\linewidth]{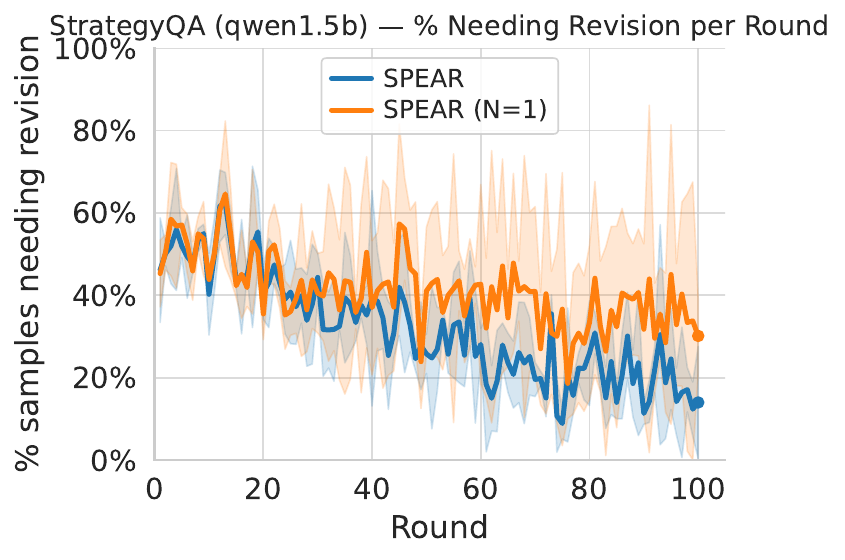}
    \caption{StrategyQA}
\end{subfigure}

\vspace{1.0em}

{\small \textbf{(b) Qwen2.5-1.5B}}

\vspace{0.5em}

\caption{Percentage of samples per-round needing revision of SPEAR with $N=1$ and $N=2$ (default) on (a) Llama3.2-3B and (b) Qwen2.5-1.5B models.}
\label{appendix:pct_needed_per_round}

\end{figure}

Based on the results in Appendix \ref{appendix:n=1_revision}, we additionally seek to observe the following characteristics of $N$: the cumulative number of successful revisions (Fig. \ref{appendix:cumulative_successful_revisions}), the number of cumulative samples needing revision (Fig. \ref{appendix:cumulative_samples_needing_revision}), and the percentage of samples needing revision each round over the course of FL training (Fig. \ref{appendix:pct_needed_per_round}).  

From Fig. \ref{appendix:cumulative_successful_revisions}, we can first observe that SPEAR with $N=2$ achieves a higher level of total success in terms of number of samples it is able to correct in comparison to only attempting revision one time. This is to be expected as allowing the model to make multiple attempts at generation allows it to explore differing solutions to come to a satisfactory answer. However, the total number of samples overall needing revision over the course of training does not significantly change, as indicated in Fig. \ref{appendix:cumulative_samples_needing_revision}. This observation further matches with Appendix \ref{appendix:n=1_revision}, where the performance of SPEAR $(N=1)$ remains higher than the baselines and is comparable to the results of SPEAR in Table \ref{tab:main_results}. This means that while attempting more revisions does increase the number of satisfactory traces, it does not necessarily mean the total number of samples that will need to be revised changes significantly in the context of the entire data. 

Secondly, we can note that in Fig. \ref{appendix:pct_needed_per_round} that over time, the percentage of samples per-round needing revision decreases, indicating that the SPEAR algorithm is effective at learning from the win-lose trace dual loss function. Additionally, in most cases, the difference between $N=1$ and $N=2$ revision attempts does not significantly differ from one another. Overall, this further reinforces the results from Appendix \ref{appendix:n=1_revision}, where conducting revision only once with SPEAR is still an effective means of fine-tuning in an online manner with incomplete feedback.

\subsection{Ablation on $\lambda_{l}$}

\begin{table*}[htbp]
\centering
\setlength{\tabcolsep}{7pt}
\renewcommand{\arraystretch}{1.15}
\caption{Ablation study on $\lambda_{l}$ for SPEAR across benchmark datasets on Llama3.2-3B and Qwen2.5-1.5B.}
\label{appendix:tab_ablation_lose}
\begin{tabular}{llcccc}
\toprule
\textbf{Model} & \textbf{$\lambda_{l}$} & \textbf{ARC-Challenge} & \textbf{HellaSwag} & \textbf{MathMCQA} & \textbf{StrategyQA} \\
\midrule

\multirow{4}{*}{\textbf{Llama3.2-3B}}
& 0.1 & 69.17$\pm$0.85 & \textbf{86.72$\pm$}0.67 & \textbf{56.70$\pm$}0.78 & 65.84$\pm$0.78 \\
& 0.3 & 68.83$\pm$0.92 & 86.46$\pm$1.10 & 53.83$\pm$0.83 & \textbf{67.69}$\pm$1.66 \\
& 0.5 & 69.48$\pm$1.29 & 86.81$\pm$0.09 & 52.93$\pm$0.21 & 66.08$\pm$2.61 \\
& 0.7 & \textbf{69.57}$\pm$1.11 & 86.17$\pm$0.51 & 52.35$\pm$1.99 & 66.72$\pm$1.14 \\
\midrule

\multirow{4}{*}{\textbf{Qwen2.5-1.5B}}
& 0.1 & 76.93$\pm$1.41 & 80.65$\pm$0.79 & \textbf{59.07$\pm$}1.16 & \textbf{66.08$\pm$}3.20 \\
& 0.3 & 75.57$\pm$3.42 & \textbf{82.03}$\pm$1.22 & 56.31$\pm$0.54 & 62.93$\pm$1.55 \\
& 0.5 & \textbf{78.01}$\pm$0.46 & 79.11$\pm$2.50 & 55.87$\pm$0.84 & 59.58$\pm$2.13 \\
& 0.7 & 76.51$\pm$1.01 & 82.02$\pm$0.42 & 51.10$\pm$4.04 & 64.29$\pm$2.08 \\
\bottomrule
\end{tabular}
\end{table*}

As outlined in Sec. \ref{subsec:final-training-phase}, the confidence-gated and tail-targeted unlikelihood loss plays an important role in the overall formulation of SPEAR. For this reason, we present an ablation study on how the weighting given to this loss ($\lambda_{l}$) affects the overall performance of the proposed algorithm.

From the results presented in Table \ref{appendix:tab_ablation_lose}, we first note that while there is no set value that is best across all datasets, setting a lower value for $\lambda_{l}$ often results in better performance. For example, with MathMCQA, the difference between $0.1$ and $0.5$ weighting is around $3$-$4\%$, with around a $6\%$ difference from $0.1$ to $0.5$ on StrategyQA for Qwen2.5-1.5B. Overall, across all datasets, the optimal value seems to range from $0.1$ to $0.3$ in weight. We believe this to occur due to the fact that weighting the lose traces higher leads to less emphasis to be put on the win traces. Because accuracy is primarily driven by MLE on the win traces, the SFT loss should remain the dominant training signal. Our auxiliary confidence-gated, tail-targeted unlikelihood loss improves results by steering the model away from some completions, but if weighted too heavily, it can pull optimization away from the main objective and hurt accuracy. However, we can see that for ARC-Challenge, the converse is true--higher $\lambda_{l}$ results in better performance. For this factor, we believe this is because the outputs we prompt ARC-Challenge to produce are inherently shorter than datasets such as MathMCQA or StrategyQA. Overall, while the results indicate that setting $\lambda_l$ lower than $\lambda_w$ is more appropriate, it is also important to consider the length of the generated outputs.

\subsection{SPEAR on Differing FL Methods}

\begin{table*}[htbp]
\centering
\setlength{\tabcolsep}{5pt}
\renewcommand{\arraystretch}{1.15}
\caption{Usage of differing common FL algorithms on SPEAR across benchmark datasets on Llama3.2-3B and Qwen2.5-1.5B.}
\label{appendix:tab_fed_algos}
\begin{tabular}{llcccc}
\toprule
\textbf{Model} & \textbf{FL Method} & \textbf{ARC-Challenge} & \textbf{HellaSwag} & \textbf{MathMCQA} & \textbf{StrategyQA} \\
\midrule

\multirow{4}{*}{\textbf{Llama3.2-3B}}
& FedAvg & 69.58$\pm$0.47 & 86.99$\pm$0.48 & 56.25$\pm$0.05 & 65.36$\pm$0.58 \\
& FedProx & 69.58$\pm$0.47 & 87.13$\pm$0.26 & 54.06$\pm$0.06 & 65.36$\pm$0.58 \\
& FedAdam & 68.09$\pm$0.17 & 85.20$\pm$0.27 & 53.74$\pm$1.40 & 56.40$\pm$6.77 \\
& FedYogi & 70.22$\pm$0.43 & 84.88$\pm$0.38 & 53.72$\pm$1.52 & 67.98$\pm$0.58 \\
\midrule

\multirow{4}{*}{\textbf{Qwen2.5-1.5B}}
& FedAvg & 76.92$\pm$1.41 & 80.55$\pm$0.77 & 59.60$\pm$0.70 & 64.48$\pm$1.60 \\
& FedProx & 76.92$\pm$1.41 & 80.67$\pm$0.50 & 57.07$\pm$0.77 & 65.87$\pm$2.98 \\
& FedAdam & 76.15$\pm$1.49 & 76.34$\pm$1.33 & 55.86$\pm$0.98 & 65.57$\pm$0.51 \\
& FedYogi & 67.58$\pm$6.40 & 72.12$\pm$2.40 & 54.63$\pm$0.75 & 64.56$\pm$0.36 \\
\bottomrule
\end{tabular}
\end{table*}

Despite FedAvg \citep{mcmahan2017communication} being the dominant default adopted in most FL works, this section seeks to see if SPEAR still maintains good performance with differing, but common FL protocols. For this, we consider three well-established alternatives to FedAvg: FedProx \citep{li2020federated}, FedAdam \citep{reddi2020adaptive, baumgart2024not}, and FedYogi \citep{reddi2020adaptive, baumgart2024not}. For FedProx, the proximal term is set to $10^{-3}$, with $\beta_1$ and $\beta_2$ set to $0.9$ and $0.99$ for FedAdam and FedYogi. In addition, the adaptivity hyperparameter for FedAdam and FedYogi is set to $0.005$. Results are presented in Table \ref{appendix:tab_fed_algos}.

From Table \ref{appendix:tab_fed_algos}, we note that SPEAR maintains relatively good and consistent performance across all evaluated FL methods (FedAvg, FedProx, FedAdam, and FedYogi). There are, however, rare cases where there is some performance difference, e.g., HellaSwag on Qwen2.5-1.5B, where the difference between FedAvg and FedYogi is around $8\%$. However, we can still see accuracy is always above $72\%$, meaning no matter the method chosen, high accuracy is still maintained. These results indicate that SPEAR is robust to differing FL protocols across differing model types, sizes, and datasets. 

\subsection{Centralized Setting}\label{appendix:centralized-exp}

\begin{table*}[htbp]
\centering
\setlength{\tabcolsep}{7pt} 
\renewcommand{\arraystretch}{1.15}
\caption{Performance Comparison of SPEAR vs. baselines in a centralized setting on Qwen2.5-1.5B.}
\label{appendix:table-centralized}
\begin{tabular}{lccccc}
\toprule
\textbf{Algorithm} & \textbf{ARC-Challenge} & \textbf{HellaSwag} & \textbf{MathMCQA} & \textbf{StrategyQA} & \textbf{Avg.} \\
\midrule

Feedback SFT & 75.00$\pm$1.11 & 85.06$\pm$0.10 & 58.68$\pm$0.32 & 66.81$\pm$0.29 & 71.39 \\
OPSD & 22.69$\pm$0.16 & 24.71$\pm$0.03 & 0.00$\pm$0.00 & 0.00$\pm$0.00 & 11.85 \\
RLTF-SD & 75.85$\pm$2.05 & 75.48$\pm$4.16 & 43.36$\pm$10.08 & 63.90$\pm$1.46 & 64.65 \\
\rowcolor{blue!10}
\textbf{SPEAR (Ours)} & 78.67$\pm$0.68 & 86.68$\pm$0.09 & 60.73$\pm$0.07 & 67.47$\pm$0.51 & \textbf{73.39} \\

\bottomrule
\end{tabular}
\end{table*}

\begin{figure}[htbp]
\centering
\begin{subfigure}[htbp]{0.245\textwidth}
    \centering
    \includegraphics[width=\linewidth]{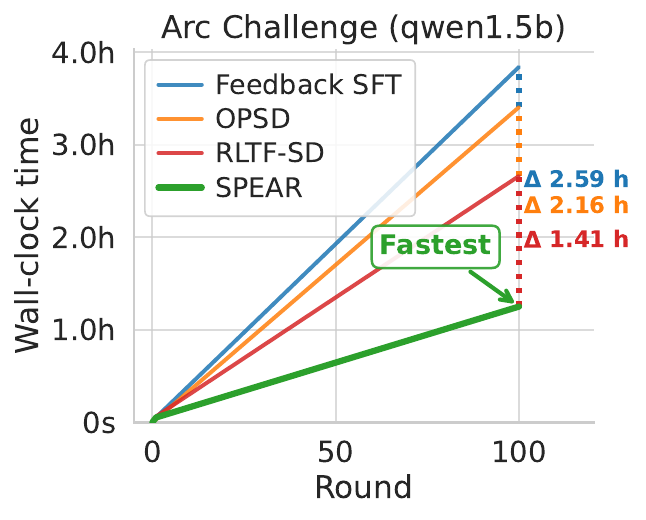}
    \caption{ARC-Challenge}
\end{subfigure}
\begin{subfigure}[htbp]{0.245\textwidth}
    \centering
    \includegraphics[width=\linewidth]{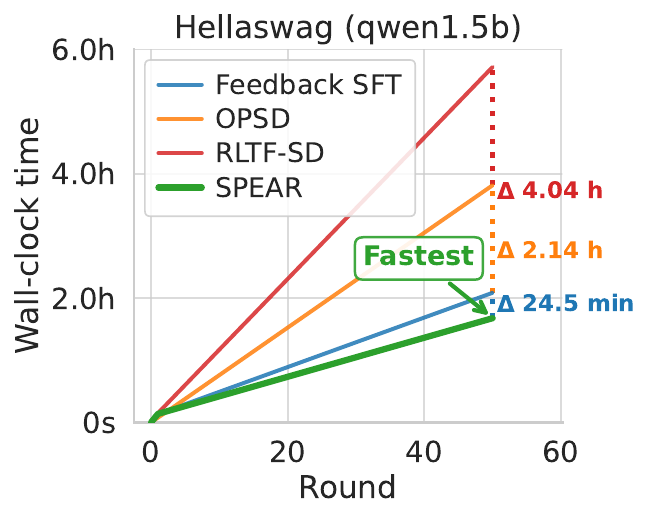}
    \caption{HellaSwag}
\end{subfigure}
\begin{subfigure}[htbp]{0.245\textwidth}
    \centering
    \includegraphics[width=\linewidth]{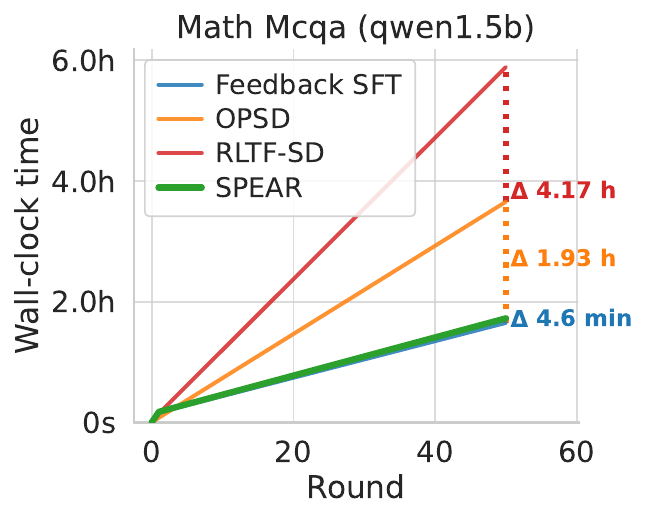}
    \caption{MathMCQA}
\end{subfigure}
\begin{subfigure}[htbp]{0.245\textwidth}
    \centering
    \includegraphics[width=\linewidth]{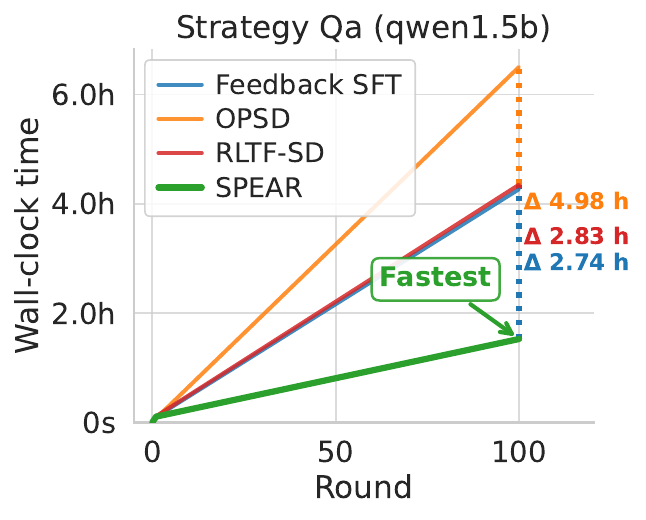}
    \caption{StrategyQA}
\end{subfigure}


\caption{Cumulative training time of each algorithm on Qwen-1.5B model in a centralized setting.}
\label{appendix:centralized_time_curves}
\vspace{-0.5em}
\end{figure}

We consider whether SPEAR can also be extended to the centralized setting outside of FL considerations (i.e., equivalent to having one client), and whether it maintains superior performance in both accuracy and wall-clock time in comparison to the baselines. Results are found in Table \ref{appendix:table-centralized} and Fig. \ref{appendix:centralized_time_curves}.

From the results in Table \ref{appendix:table-centralized}, we note the performance advantages presented by SPEAR in Sec. \ref{sec:experiments} persist even in a centralized setting. Specifically, we can note that SPEAR still maintains superior performance on all datasets in comparison to all baselines on Qwen2.5-1.5B. We see a similar result in Fig. \ref{appendix:centralized_time_curves}, where SPEAR exhibits faster training times in comparison to all baselines in 3 out of 4 datasets. In the case of MathMCQA, the difference with Feedback SFT is only a couple of minutes, and this tradeoff is marked by $2\%$ improvement in accuracy. These results indicate that SPEAR is an effective method of online fine-tuning of LLMs in both an FL and centralized environment.


\subsection{Larger Model Size: Qwen2.5-7B}\label{appendix:7b_model}

\begin{table*}[htbp]
\centering
\setlength{\tabcolsep}{4pt} 
\renewcommand{\arraystretch}{1.15}
\caption{Performance Comparison of SPEAR vs. baselines across benchmark datasets on Qwen2.5-7B.}
\label{tab:7b_results}
\begin{tabular}{lccccc}
\toprule
\textbf{Algorithm} & \textbf{ARC-Challenge} & \textbf{HellaSwag} & \textbf{MathMCQA} & \textbf{StrategyQA} & \textbf{Avg.} \\
\midrule

Feedback SFT & 89.38$\pm$0.47 & 91.93$\pm$0.07 & 54.90$\pm$2.74 & 73.87$\pm$0.51 & 77.52 \\
RLTF-SD & 89.55$\pm$0.30 & 91.96$\pm$0.11 & -- & 27.44$\pm$21.76 & 69.65 \\
\rowcolor{blue!10}
\textbf{SPEAR (Ours)} & 89.72$\pm$0.21 & 93.12$\pm$0.09 & 59.75$\pm$0.17 & 69.65$\pm$2.55 & \textbf{78.06} \\

\bottomrule
\end{tabular}
\end{table*}

Here, we consider whether SPEAR's advantages persist on models of larger size by comparing against the baselines on Qwen2.5-7B, with the results shown in Table \ref{tab:7b_results}. For ARC-Challenge, HellaSwag, and MathMCQA, we use a maximum and minimum learning rates of $5 \times 10^{-5}$ / $1 \times 10^{-5}$. For StrategyQA, it is $1 \times 10^{-4}$ / $5 \times 10^{-5}$.

Based on the results in Table \ref{tab:7b_results}, we note that the performance gains from SPEAR persist on larger model sizes. RLTF-SD on MathMCQA is omitted as well as GRPO due to memory constraints. This further highlights SPEAR's memory advantages in comparison to RL group generation based methods.


\subsection{Comparison of SPEAR against DPO-based method}\label{appendix:dpo_section}

\begin{table*}[htbp]
\centering
\setlength{\tabcolsep}{4pt} 
\renewcommand{\arraystretch}{1.15}
\caption{Performance Comparison of SPEAR against Online Feedback DPO across benchmark datasets.}
\label{tab:online_dpo_results}
\begin{tabular}{lccccc}
\toprule
\textbf{Algorithm} & \textbf{ARC-Challenge} & \textbf{StrategyQA} & \textbf{MathMCQA} & \textbf{HellaSwag} & \textbf{Avg.} \\
\midrule

Online Feedback DPO & 69.62 & 53.81 & 31.25 & 60.35 & 53.76 \\
\rowcolor{blue!10}
\textbf{SPEAR (Ours)} & 76.93 & 66.08 & 59.07 & 80.65 & \textbf{70.68} \\

\bottomrule
\end{tabular}
\end{table*}

We consider whether SPEAR's MLE and unlikelihood objective can outperform preference-based RLHF learning. For this, we have conducted experiments to compare against an online variant of DPO \citep{rafailov2023direct} that trains on satisfactory-unsatisfactory pairs found during training on Qwen2.5-1.5B. Results are presented in Table \ref{tab:online_dpo_results}.

We can see from the results that SPEAR maintains superior performance in comparison to Online Feedback DPO on average and across each individual dataset. This indicates that the proposed method with its dual MLE and unlikelihood objective outperforms a directly matched online preference-optimization baseline.

\subsection{Experiment Over Newer Model Family}\label{appendix:qwen3p5_2b_section}

\begin{table*}[htbp]
\centering
\setlength{\tabcolsep}{4pt} 
\renewcommand{\arraystretch}{1.15}
\caption{Performance Comparison of SPEAR vs. baselines across benchmark datasets on Qwen3.5-2B.}
\label{tab:qwen3p5_2b_results}
\begin{tabular}{lccccc}
\toprule
\textbf{Algorithm} & \textbf{ARC-Challenge} & \textbf{StrategyQA} & \textbf{MathMCQA} & \textbf{HellaSwag} & \textbf{Avg.} \\
\midrule

Online Feedback DPO & 80.56 & 39.80 & 31.34 & 60.44 & 53.04 \\
Feedback SFT & 82.59 & 48.47 & 52.60 & 85.21 & 67.22 \\
\rowcolor{blue!10}
\textbf{SPEAR (Ours)} & 82.79 & 59.83 & 54.38 & 85.92 & \textbf{70.73} \\

\bottomrule
\end{tabular}
\end{table*}

In contrast to the Qwen2.5 and Llama3.2 model families, here we consider Qwen3.5 \citep{qwen3.5}, which is a more recent family of models. Specifically, we consider SPEAR against Feedback-SFT and Online Feedback DPO on Qwen3.5-2B (results in Table \ref{tab:qwen3p5_2b_results}). For ARC-Challenge, HellaSwag, MathMCQA, and StrategyQA, the maximum and minimum learning rates are $1 \times 10^{-5}$ / $5 \times 10^{-6}$, $5 \times 10^{-5}$ / $1 \times 10^{-5}$, $5 \times 10^{-5}$ / $1 \times 10^{-5}$, and $1 \times 10^{-5}$ / $5 \times 10^{-6}$ respectively. 

From Table \ref{tab:qwen3p5_2b_results}, we note that SPEAR's performance gains observed in Sec. \ref{sec:experiments} persist on Qwen3.5-2B, with SPEAR outperforming the baselines on all four evaluated datasets. Overall, the results indicate SPEAR's advantages persist across more recent architectures in addition to those evaluated in Sec. \ref{sec:experiments}.

\subsection{SPEAR under Weaker Feedback}\label{app:weaker_feedback}
In this section, we consider whether SPEAR can still meaningfully learn under comparatively weaker feedback than those utilized in Sec. \ref{sec:experiments}. Specifically, we consider two types of weaker feedback compared to the original partial feedback: noisy partial feedback and process critique. For \textbf{Noisy Partial Feedback}, we replace $20\%$ of the original partial feedback content with random and irrelevant content. With \textbf{Process Critique}, the same generic advice (e.g., ``The answer is unsatisfactory. Find the first error and continue from there.") not specific to any one problem is given. The specific feedback given for Process Critique for each dataset is found in Appendix \ref{app:process_critique_prompts}.

\subsubsection{Correction Percentage before Training under Weaker Feedback}
\begin{table*}[htbp]
\centering
\small
\setlength{\tabcolsep}{4pt}
\renewcommand{\arraystretch}{1.15}
\caption{Correction percentage of Qwen2.5-1.5B and Llama3.2-3B before training under different feedback types.}
\label{tab:feedback_correction_percentage}
\begin{tabular}{llcccc}
\toprule
\textbf{Model} & \textbf{Feedback Type} & \textbf{ARC-Challenge} & \textbf{HellaSwag} & \textbf{MathMCQA} & \textbf{StrategyQA} \\
\midrule

Qwen2.5-1.5B
& Process Critique
& 16.1\% & 16.6\% & 15.3\% & 35.6\% \\

& Noisy Partial Feedback
& 18.3\% & 21.5\% & 12.2\% & 45.7\% \\

& Partial Feedback (Original)
& 28.2\% & 19.5\% & 12.6\% & 46.3\% \\

\midrule

Llama3.2-3B
& Process Critique
& 22.9\% & 22.4\% & 10.4\% & 35.9\% \\

& Noisy Partial Feedback
& 28.2\% & 31.7\% & 11.3\% & 38.3\% \\

& Partial Feedback (Original)
& 34.6\% & 31.7\% & 12.8\% & 45.6\% \\

\bottomrule
\end{tabular}
\end{table*}

In Table \ref{tab:feedback_correction_percentage}, we consider whether the foundation model can meaningfully correct itself under the weaker feedback types before training. While we note that the model under Process Critique and Noisy Partial Feedback in general corrects itself less often than the original partial feedback, it is still able to meaningfully correct a non-trivial number of samples. Moreover, for some settings (e.g., HellaSwag for Llama3.2-3B / MathMCQA for Qwen2.5-1.5B), the weaker feedback settings outperform or equal the original partial feedback in correction percentage. We demonstrate in Appendix \ref{appendix:subsec-acc-weaker-feedback} that SPEAR can learn from all three feedback types.

\subsubsection{Accuracy on Weaker Feedback}\label{appendix:subsec-acc-weaker-feedback}
\begin{table*}[htbp]
\centering
\footnotesize
\setlength{\tabcolsep}{4pt}
\renewcommand{\arraystretch}{1.15}
\caption{SPEAR under noisy and generic process critique feedback on Qwen2.5-1.5B and Llama3.2-3B.}
\label{tab:weaker_feedback_results}
\begin{tabular}{llccccc}
\toprule
\textbf{Model} & \textbf{Feedback Type} & \textbf{MathMCQA} & \textbf{HellaSwag} & \textbf{StrategyQA} & \textbf{ARC-Ch.} & \textbf{Avg.} \\
\midrule

Qwen2.5-1.5B
& Process Critique
& 49.44 & 80.28 & 64.92 & 76.79 & 67.86 \\

& Noisy Partial Feedback
& 56.56 & 82.46 & 65.07 & 78.16 & 70.56 \\

& Partial Feedback (Orig.)
& 59.07 & 80.65 & 66.08 & 76.93 & 70.68 \\

\midrule

Llama3.2-3B
& Process Critique
& 56.38 & 87.57 & 62.30 & 76.19 & 70.61 \\

& Noisy Partial Feedback
& 55.02 & 87.77 & 62.30 & 70.14 & 68.81 \\

& Partial Feedback (Orig.)
& 56.70 & 86.72 & 65.84 & 69.17 & 69.61 \\

\bottomrule
\end{tabular}
\end{table*}

\begin{table*}[htbp]
\centering
\setlength{\tabcolsep}{4pt}
\renewcommand{\arraystretch}{1.15}
\caption{SPEAR under noisy and generic process critique feedback on Qwen3.5-2B.}
\label{tab:qwen3p5_weaker_feedback_results}
\begin{tabular}{lccccc}
\toprule
\textbf{Feedback Type} & \textbf{MathMCQA} & \textbf{HellaSwag} & \textbf{StrategyQA} & \textbf{ARC-Challenge} & \textbf{Avg.} \\
\midrule

Process Critique
& 40.62 & 86.19 & 50.95 & 82.59 & 65.09 \\

Noisy Partial Feedback
& 54.72 & 85.56 & 51.52 & 83.20 & 68.75 \\

Partial Feedback (Original)
& 54.38 & 85.92 & 59.83 & 82.79 & 70.73 \\

\bottomrule
\end{tabular}
\end{table*}

Next, we consider whether performance on comparatively weaker feedback still enables comparable performance to the original partial feedback. In addition to Qwen2.5-1.5B and Llama3.2-3B (Table \ref{tab:weaker_feedback_results}), we also consider the Qwen3.5-2B model (Table \ref{tab:qwen3p5_weaker_feedback_results}). 

From the results in Tables \ref{tab:weaker_feedback_results} and \ref{tab:qwen3p5_weaker_feedback_results}, SPEAR remains effective under both Process Critique and Noisy Partial Feedback across all model families and datasets evaluated. Therefore, the model's ability to learn \textbf{under weaker generic and noisy feedback utilizing SPEAR's training methodology is not isolated to a specific model family or dataset}. Moreover, by observing Table \ref{tab:main_results} in Sec. \ref{sec:experiments}, both Process Critique and Noisy Partial Feedback-based SPEAR \textbf{outperform the baselines.} Overall, these results indicate SPEAR remains an effective methodology when considering more practical scenarios where high-quality feedback is not always guaranteed. 
\subsection{Open-Ended Scenario with Weak Feedback}\label{appendix:mbpp_section}

\begin{table*}[htbp]
\centering
\setlength{\tabcolsep}{6pt}
\renewcommand{\arraystretch}{1.15}
\caption{MBPP results on Qwen2.5-1.5B.}
\label{tab:mbpp_results}
\begin{tabular}{lccccc}
\toprule
& \textbf{SPEAR} & \textbf{Feedback SFT} & \textbf{Online Feedback DPO} & \textbf{RLTF-SD} & \textbf{GRPO} \\
\midrule

\textbf{Pass Rate (\%)} & 40.6 & 37.2 & 24.5 & 44.0 & 40.8 \\

\bottomrule
\end{tabular}
\end{table*}

We consider whether SPEAR can remain effective under two difficult conditions simultaneously: (i) an open-ended scenario where multiple differing completions can be satisfactory and (ii) under weak feedback. For this, we consider the Mostly Basic Python Problems (MBPP) \citep{austin2021program} dataset, which accepts multiple differing solutions, given that they pass the test cases for each sample. In addition, we utilize the Process Critique-based feedback from Appendix \ref{app:weaker_feedback}. The exact formatting prompts used can be found in Appendix \ref{app:mbpp_prompt}.

Given the results from Table \ref{tab:mbpp_results}, SPEAR achieves a 40.6\% pass rate with Qwen2.5-1.5B, \textbf{trained using only general process critique-based feedback.} It \textbf{outperforms both online feedback baselines} (Feedback SFT and Online Feedback DPO) while remaining nearly identical in performance to GRPO. Importantly, although RLTF-SD exhibits a moderately higher pass rate of 44\%, \textbf{GRPO and RLTF-SD require privileged ground-truth information during training, incur substantially greater training cost, and do not operate in the same online feedback setting}, in contrast with SPEAR. These results therefore demonstrate that SPEAR can effectively extend to open-ended tasks under simple feedback while retaining its online-learning advantages. Importantly, this experiment also illustrates why SPEAR's binary signal is better interpreted as interaction-level satisfaction rather than supervision from a unique ground-truth completion. MBPP admits multiple distinct programs as successful whenever they satisfy the task's test cases; the binary signal therefore indicates whether the produced response is acceptable without specifying which response the model should generate.

\section{Detailed Hyperparameter Setup}\label{appendix:hyperparams}
In this section, we outline the hyperparameter settings utilized in the running of experiments within the manuscript (in addition to those presented in Sec. \ref{sec:experiments}) for each algorithm.

\subsection{Hyperparameter Settings (SPEAR)}
\begin{table}[htbp]
\centering
\small
\caption{Training and federated learning hyperparameters for the Qwen2.5-1.5B SPEAR runs.}
\setlength{\tabcolsep}{5pt}
\begin{tabular}{lcccc}
\toprule
\textbf{Hyperparameter} & \textbf{Math MCQA} & \textbf{ARC-Challenge} & \textbf{StrategyQA} & \textbf{HellaSwag} \\
\midrule
Max sequence length & 2048 & 2048 & 1024 & 2048 \\
LoRA rank & 16 & 16 & 16 & 16 \\
LoRA alpha & 32 & 32 & 32 & 32 \\
LoRA dropout & 0.1 & 0.1 & 0.1 & 0.1 \\
\midrule
Number of clients & 50 & 50 & 50 & 50 \\
Dirichlet alpha & 1.0 & 1.0 & 1.0 & 1.0 \\
Number of FL rounds & 50 & 100 & 100 & 50 \\
Clients per round aggregated & 3 & 3 & 5 & 5 \\
Local training steps per round & 5 & 5 & 5 & 10 \\
Ratio of rounds to warmup & 0.1 & 0.1 & 0.1 & 0.1 \\
\midrule
Generation temperature & 0.8 & 0.8 & 0.8 & 0.8 \\
Lambda win weight $\lambda_{w}$ & 1.0 & 1.0 & 1.0 & 1.0 \\
Lambda loss weight $\lambda_{l}$ & 0.1 & 0.1 & 0.1 & 0.1 \\
Unlikelihood threshold $\mu$ & 0.3 & 0.3 & 0.3 & 0.3 \\
$\tau$ & 32 & 0 & 24 & 0 \\
\midrule
Learning rate & $5.05\times 10^{-5}$ & $5.0\times 10^{-5}$ & $1.0\times 10^{-4}$ & $5.0\times 10^{-5}$ \\
Batch size & 4 & 16 & 16 & 16 \\
Grad accumulation steps & 4 & 1 & 1 & 1 \\
Linear warmup ratio & 0.1 & 0.1 & 0.1 & 0.1 \\
Minimum LR for cosine decay & $1.0\times 10^{-5}$ & $1.0\times 10^{-5}$ & $5.0\times 10^{-5}$ & $1.0\times 10^{-5}$ \\
\bottomrule
\end{tabular}
\label{tab:qwen1.5b_hparams}
\end{table}

\begin{table}[htbp]
\centering
\caption{Training and federated learning hyperparameters for the Llama-3.2-3B SPEAR runs.}
\small
\setlength{\tabcolsep}{5pt}
\begin{tabular}{lcccc}
\toprule
\textbf{Hyperparameter} & \textbf{Math MCQA} & \textbf{ARC-Challenge} & \textbf{StrategyQA} & \textbf{HellaSwag} \\
\midrule
Max sequence length & 2048 & 2048 & 1024 & 2048 \\
LoRA rank & 16 & 16 & 16 & 16 \\
LoRA alpha & 32 & 32 & 32 & 32 \\
LoRA dropout & 0.1 & 0.1 & 0.1 & 0.1 \\
\midrule
Number of clients & 50 & 50 & 50 & 50 \\
Dirichlet alpha & 1.0 & 1.0 & 1.0 & 1.0 \\
Number of FL rounds & 50 & 100 & 100 & 50 \\
Clients per round aggregated & 3 & 3 & 5 & 5 \\
Local training steps per round & 5 & 5 & 5 & 10 \\
Ratio of rounds to warmup & 0.1 & 0.1 & 0.1 & 0.1 \\
\midrule
Generation temperature & 0.8 & 0.8 & 0.8 & 0.8 \\
Lambda win weight $\lambda_{w}$ & 1.0 & 1.0 & 1.0 & 1.0 \\
Lambda loss weight $\lambda_{l}$ & 0.1 & 0.1 & 0.1 & 0.1 \\
Unlikelihood threshold $\mu$ & 0.3 & 0.3 & 0.3 & 0.3 \\
$\tau$ & 32 & 0 & 24 & 0 \\
\midrule
Learning rate & $5.05\times 10^{-5}$ & $5.0\times 10^{-5}$ & $1.0\times 10^{-4}$ & $5.0\times 10^{-5}$ \\
Batch size & 4 & 8 & 16 & 8 \\
Grad accumulation steps & 4 & 2 & 1 & 2 \\
Linear warmup ratio & 0.1 & 0.1 & 0.1 & 0.1 \\
Minimum LR for cosine decay & $1.0\times 10^{-5}$ & $1.0\times 10^{-5}$ & $5.0\times 10^{-5}$ & $1.0\times 10^{-5}$ \\
\bottomrule
\end{tabular}
\label{tab:llama3b_hparams}
\end{table}

\subsection{Hyperparameter Settings of baselines (GRPO / OPSD / RLTF-SD)}
Here, the hyperparameter settings of all baselines evaluated (GRPO / OPSD / RLTF-SD) against SPEAR are presented. For StrategyQA, note the max sequence length is increased for the baselines as the entirety of the ``facts" column is inputted as feedback for OPSD; for GRPO, it's included as the regular input, thereby requiring a larger sequence of tokens. For SPEAR, only partial portions of this column are included when attempting revision, allowing a smaller sequence length. 

\begin{table}[htbp]
\centering
\caption{Training, federated learning, and GRPO hyperparameters for the Qwen2.5-1.5B runs.}
\setlength{\tabcolsep}{5pt}
\begin{tabular}{lcccc}
\toprule
\textbf{Hyperparameter} & \textbf{Math MCQA} & \textbf{ARC-Challenge} & \textbf{StrategyQA} & \textbf{HellaSwag} \\
\midrule
Max sequence length & 2048 & 2048 & 1400 & 2048 \\
LoRA rank & 16 & 16 & 16 & 16 \\
LoRA alpha & 32 & 32 & 32 & 32 \\
LoRA dropout & 0.1 & 0.1 & 0.1 & 0.1 \\
\midrule
Number of clients & 50 & 50 & 50 & 50 \\
Dirichlet alpha & 1.0 & 1.0 & 1.0 & 1.0 \\
Number of FL rounds & 50 & 100 & 100 & 50 \\
Clients per round aggregated & 3 & 3 & 5 & 5 \\
Local training steps per round & 5 & 5 & 5 & 10 \\
Ratio of rounds to warmup & 0.1 & 0.1 & 0.1 & 0.1 \\
\midrule
Learning rate & $5.05\times10^{-5}$ & $5.0\times10^{-5}$ & $1.0\times10^{-4}$ & $5.0\times10^{-5}$ \\
Batch size & 2 & 4 & 4 & 4 \\
Grad accumulation steps & 8 & 4 & 4 & 4 \\
Linear warmup ratio & 0.1 & 0.1 & 0.1 & 0.1 \\
Minimum LR for cosine decay & $1.0\times10^{-5}$ & $1.0\times10^{-5}$ & $5.0\times10^{-5}$ & $1.0\times10^{-5}$ \\
\midrule
GRPO $\beta_{\mathrm{KL}}$ & 0.02 & 0.02 & 0.1 & 0.02 \\
GRPO clip range & 0.2 & 0.2 & 0.2 & 0.2 \\
GRPO num. generations & 2 & 4 & 2 & 4 \\
\bottomrule
\end{tabular}
\label{tab:qwen1.5b_grpo_hparams}
\end{table}

\begin{table*}[htbp]
\centering
\caption{Training, federated learning, and GRPO hyperparameters for the Llama-3.2-3B runs.}
\small
\setlength{\tabcolsep}{5pt}
\begin{tabular}{lcccc}
\toprule
\textbf{Hyperparameter} & \textbf{Math MCQA} & \textbf{ARC-Challenge} & \textbf{StrategyQA} & \textbf{HellaSwag} \\
\midrule
Max sequence length & 2048 & 2048 & 1400 & 2048 \\
LoRA rank & 16 & 16 & 16 & 16 \\
LoRA alpha & 32 & 32 & 32 & 32 \\
LoRA dropout & 0.1 & 0.1 & 0.1 & 0.1 \\
\midrule
Number of clients & 50 & 50 & 50 & 50 \\
Dirichlet alpha & 1.0 & 1.0 & 1.0 & 1.0 \\
Number of FL rounds & 50 & 100 & 100 & 50 \\
Clients per round aggregated & 3 & 3 & 5 & 5 \\
Local training steps per round & 5 & 5 & 5 & 10 \\
Ratio of rounds to warmup & 0.1 & 0.1 & 0.1 & 0.1 \\
\midrule
Learning rate & $5.05\times10^{-5}$ & $5.0\times10^{-5}$ & $1.0\times10^{-4}$ & $5.0\times10^{-5}$ \\
Batch size & 2 & 4 & 4 & 4 \\
Grad accumulation steps & 8 & 4 & 4 & 4 \\
Linear warmup ratio & 0.1 & 0.1 & 0.1 & 0.1 \\
Minimum LR for cosine decay & $1.0\times10^{-5}$ & $1.0\times10^{-5}$ & $5.0\times10^{-5}$ & $1.0\times10^{-5}$ \\
\midrule
GRPO $\beta_{\mathrm{KL}}$ & 0.02 & 0.02 & 0.1 & 0.02 \\
GRPO clip range & 0.2 & 0.2 & 0.2 & 0.2 \\
GRPO num. generations & 2 & 4 & 2 & 4 \\
\bottomrule
\end{tabular}
\label{tab:llama3b_grpo_hparams}
\end{table*}

\begin{table}[htbp]
\centering
\caption{Training and federated learning hyperparameters for OPSD on the Qwen2.5-1.5B runs.}
\small
\setlength{\tabcolsep}{5pt}
\begin{tabular}{lcccc}
\toprule
\textbf{Hyperparameter} & \textbf{Math MCQA} & \textbf{ARC-Challenge} & \textbf{StrategyQA} & \textbf{HellaSwag} \\
\midrule
Max sequence length & 2048 & 2048 & 1400 & 2048 \\
LoRA rank & 16 & 16 & 16 & 16 \\
LoRA alpha & 32 & 32 & 32 & 32 \\
LoRA dropout & 0.1 & 0.1 & 0.1 & 0.1 \\
\midrule
Number of clients & 50 & 50 & 50 & 50 \\
Dirichlet alpha & 1.0 & 1.0 & 1.0 & 1.0 \\
Number of FL rounds & 50 & 100 & 100 & 50 \\
Clients per round aggregated & 3 & 3 & 5 & 5 \\
Local training steps per round & 5 & 5 & 5 & 10 \\
\midrule
Learning rate & $5.05\times10^{-5}$ & $5.0\times10^{-5}$ & $1.0\times10^{-4}$ & $5.0\times10^{-5}$ \\
Batch size & 4 & 16 & 4 & 8 \\
Grad accumulation steps & 4 & 1 & 4 & 2 \\
Linear warmup ratio & 0.1 & 0.1 & 0.1 & 0.1 \\
Minimum LR for cosine decay & $1.0\times10^{-5}$ & $1.0\times10^{-5}$ & $5.0\times10^{-5}$ & $1.0\times10^{-5}$ \\
\bottomrule
\end{tabular}
\label{tab:qwen1.5b_opsd_hparams}
\end{table}

\begin{table}[htbp]
\centering
\caption{Training and federated learning hyperparameters for OPSD on the Llama-3.2-3B runs.}
\small
\setlength{\tabcolsep}{5pt}
\begin{tabular}{lcccc}
\toprule
\textbf{Hyperparameter} & \textbf{Math MCQA} & \textbf{ARC-Challenge} & \textbf{StrategyQA} & \textbf{HellaSwag} \\
\midrule
Max sequence length & 2048 & 2048 & 1400 & 2048 \\
LoRA rank & 16 & 16 & 16 & 16 \\
LoRA alpha & 32 & 32 & 32 & 32 \\
LoRA dropout & 0.1 & 0.1 & 0.1 & 0.1 \\
\midrule
Number of clients & 50 & 50 & 50 & 50 \\
Dirichlet alpha & 1.0 & 1.0 & 1.0 & 1.0 \\
Number of FL rounds & 50 & 100 & 100 & 50 \\
Clients per round aggregated & 3 & 3 & 5 & 5 \\
Local training steps per round & 5 & 5 & 5 & 10 \\
\midrule
Learning rate & $5.05\times10^{-5}$ & $5.0\times10^{-5}$ & $1.0\times10^{-4}$ & $5.0\times10^{-5}$ \\
Batch size & 2 & 8 & 4 & 8 \\
Grad accumulation steps & 8 & 2 & 4 & 2 \\
Linear warmup ratio & 0.1 & 0.1 & 0.1 & 0.1 \\
Minimum LR for cosine decay & $1.0\times10^{-5}$ & $1.0\times10^{-5}$ & $5.0\times10^{-5}$ & $1.0\times10^{-5}$ \\
\bottomrule
\end{tabular}
\label{tab:llama3b_opsd_hparams}
\end{table}

\begin{table}[htbp]
\centering
\caption{Training, federated learning, and RLTF-SD hyperparameters for the Qwen2.5-1.5B runs.}
\small
\setlength{\tabcolsep}{5pt}
\begin{tabular}{lcccc}
\toprule
\textbf{Hyperparameter} & \textbf{Math MCQA} & \textbf{ARC-Challenge} & \textbf{StrategyQA} & \textbf{HellaSwag} \\
\midrule
Max sequence length & 2048 & 2048 & 1400 & 2048 \\
LoRA rank & 16 & 16 & 16 & 16 \\
LoRA alpha & 32 & 32 & 32 & 32 \\
LoRA dropout & 0.1 & 0.1 & 0.1 & 0.1 \\
\midrule
Number of clients & 50 & 50 & 50 & 50 \\
Dirichlet alpha & 1.0 & 1.0 & 1.0 & 1.0 \\
Number of FL rounds & 50 & 100 & 100 & 50 \\
Clients per round aggregated & 3 & 3 & 5 & 5 \\
Local training steps per round & 5 & 5 & 5 & 10 \\
Ratio of rounds to warmup & 0.1 & 0.1 & 0.1 & 0.1 \\
\midrule
Learning rate & $5.05\times10^{-5}$ & $5.0\times10^{-5}$ & $1.0\times10^{-4}$ & $5.0\times10^{-5}$ \\
Batch size & 2 & 4 & 4 & 4 \\
Grad accumulation steps & 8 & 4 & 4 & 4 \\
Linear warmup ratio & 0.1 & 0.1 & 0.1 & 0.1 \\
Minimum LR for cosine decay & $1.0\times10^{-5}$ & $1.0\times10^{-5}$ & $5.0\times10^{-5}$ & $1.0\times10^{-5}$ \\
\midrule
RLTF-SD $\gamma$ & 1.0 & 1.0 & 1.0 & 1.0 \\
RLTF-SD SD Coefficient & 0.1 & 0.1 & 0.1 & 0.1 \\
RLTF-SD RL Coefficient & 1.0 & 1.0 & 1.0 & 1.0 \\
RLTF-SD Early Termination & True & True & True & True \\
\bottomrule
\end{tabular}
\label{tab:qwen1.5b_rltfsd_hparams}
\end{table}

\begin{table*}[htbp]
\centering
\caption{Training, federated learning, and RLTF-SD hyperparameters for the Llama-3.2-3B runs.}
\small
\setlength{\tabcolsep}{5pt}
\begin{tabular}{lcccc}
\toprule
\textbf{Hyperparameter} & \textbf{Math MCQA} & \textbf{ARC-Challenge} & \textbf{StrategyQA} & \textbf{HellaSwag} \\
\midrule
Max sequence length & 2048 & 2048 & 1400 & 2048 \\
LoRA rank & 16 & 16 & 16 & 16 \\
LoRA alpha & 32 & 32 & 32 & 32 \\
LoRA dropout & 0.1 & 0.1 & 0.1 & 0.1 \\
\midrule
Number of clients & 50 & 50 & 50 & 50 \\
Dirichlet alpha & 1.0 & 1.0 & 1.0 & 1.0 \\
Number of FL rounds & 50 & 100 & 100 & 50 \\
Clients per round aggregated & 3 & 3 & 5 & 5 \\
Local training steps per round & 5 & 5 & 5 & 10 \\
Ratio of rounds to warmup & 0.1 & 0.1 & 0.1 & 0.1 \\
\midrule
Learning rate & $5.05\times10^{-5}$ & $5.0\times10^{-5}$ & $1.0\times10^{-4}$ & $5.0\times10^{-5}$ \\
Batch size & 2 & 4 & 4 & 4 \\
Grad accumulation steps & 8 & 4 & 4 & 4 \\
Linear warmup ratio & 0.1 & 0.1 & 0.1 & 0.1 \\
Minimum LR for cosine decay & $1.0\times10^{-5}$ & $1.0\times10^{-5}$ & $5.0\times10^{-5}$ & $1.0\times10^{-5}$ \\
\midrule
RLTF-SD $\gamma$ & 1.0 & 1.0 & 1.0 & 1.0 \\
RLTF-SD SD Coefficient & 0.1 & 0.1 & 0.1 & 0.1 \\
RLTF-SD RL Coefficient & 1.0 & 1.0 & 1.0 & 1.0 \\
RLTF-SD Early Termination & True & True & True & True \\
\bottomrule
\end{tabular}
\label{tab:llama3b_rltfsd_hparams}
\end{table*}


\section{Prompt Format}\label{appendix:prompt-format}

\subsection{Main Format}\label{appendix:main-prompt-format}
Here, the text prompt templates (both initial input and feedback loop) for each dataset are provided below. Parts outlined with $\{\}$ in blue indicate data fields to be inputted into the text prompts.

\definecolor{headerBg}{RGB}{220, 130, 150}
\definecolor{boxBg}{RGB}{255, 245, 248}
\definecolor{fieldBlue}{RGB}{30, 100, 200}
\definecolor{headerText}{RGB}{255,255,255}

\tcbset{
  promptbox/.style={
    enhanced,
    breakable,
    colback=boxBg,
    colframe=headerBg,
    coltitle=headerText,
    fonttitle=\bfseries\small\sffamily,
    title={#1},
    attach boxed title to top left={yshift=-2mm, xshift=4mm},
    boxed title style={
      colback=headerBg,
      colframe=headerBg,
      sharp corners,
      arc=2pt,
    },
    arc=4pt,
    boxrule=1pt,
    left=6pt, right=6pt, top=8pt, bottom=6pt,
    fontupper=\small\ttfamily,
  }
}

\newcommand{\field}[1]{\textcolor{fieldBlue}{\{#1\}}}

\newcommand{\fbsep}{%
  \vspace{4pt}%
  \hrule height 0.4pt \relax%
  \vspace{2pt}%
  \noindent\textit{\small + Feedback (appended on unsatisfactory response):}%
  \vspace{4pt}%
}

\begin{tcolorbox}[promptbox={Prompt Template --- MathMCQA}]

Solve the math problem step by step, then choose the correct answer.\\
Show your reasoning, then end with exactly: ``Answer: X''\\
where X is the letter A, B, C, or D.\\[4pt]
{[PROBLEM]}\\
\field{problem}\\[4pt]
{[OPTIONS]}\\
\field{options}

\fbsep

Incorrect. Your answer `\field{predicted\_letter}' is wrong.\\[4pt]
Hint --- consider these opening steps of the solution:\\
\field{hint\_lines}\\[4pt]
Reason step by step, then end your response with exactly:
``Answer: X'' where X is A, B, C, or D.
DO NOT repeat the letter you used previously.

\end{tcolorbox}

\bigskip

\begin{tcolorbox}[promptbox={Prompt Template --- HellaSwag}]

Choose the best continuation of the context.\\
Respond with ONLY the letter (A, B, C, or D).\\[4pt]
{[CONTEXT]}\\
\field{context}\\[4pt]
{[OPTIONS]}\\
\field{options}\\[4pt]
Answer:

\fbsep

Incorrect. `\field{predicted\_letter}' is wrong.\\[4pt]
Hint: Consider an answer that starts with ``\field{hint\_words}...''\\
Choose the most natural continuation.\\[4pt]
Answer with exactly one letter: A, B, C, or D.
DO NOT choose the answer you used previously.

\end{tcolorbox}

\bigskip

\begin{tcolorbox}[promptbox={Prompt Template --- ARC-Challenge}]

Answer the multiple-choice question by choosing the best option.\\
Respond with ONLY the letter (A, B, C, or D).\\[4pt]
{[QUESTION]}\\
\field{question}\\[4pt]
{[OPTIONS]}\\
\field{options}\\[4pt]
Answer:

\fbsep

Incorrect. `\field{predicted\_letter}' is wrong.\\[4pt]
Hint: Consider an answer that starts with ``\field{hint\_words}...''\\
Consider the underlying scientific principle.\\[4pt]
Answer with exactly one letter: A, B, C, or D.
DO NOT choose the answer you used previously.

\end{tcolorbox}

\bigskip

\begin{tcolorbox}[promptbox={Prompt Template --- StrategyQA}]

Answer the following yes/no question by reasoning step by step.\\
Think through the question carefully, then end your response with\\
exactly: ``Answer: yes'' or ``Answer: no'' on the final line.\\[4pt]
{[QUESTION]}\\
\field{question}

\fbsep

Incorrect. `\field{predicted\_answer}' is not the right answer.\\[4pt]
Hint --- consider these relevant facts:\\
\hspace*{8pt}- \field{facts}\\
\hspace*{8pt}$\vdots$\\[4pt]
Reason through the question step by step, then end your response with
exactly: ``Answer: yes'' or ``Answer: no''.
DO NOT repeat the answer you gave previously.

\end{tcolorbox}


\subsection{Format for Process Critique Feedback}\label{app:process_critique_prompts}

Below, we outline the specifics of the process critique-based feedback outlined in Appendix \ref{app:weaker_feedback} for each dataset. Importantly, outside of the feedback field, the other components are the same as those presented in Appendix \ref{appendix:main-prompt-format}.

\begin{tcolorbox}[promptbox={Process Critique Feedback --- ARC-Challenge}]

The prior answer is unsatisfactory. Check what quantity or scientific
relationship the question asks about, eliminate options that violate that
relationship, and then revise.

\medskip
I am not providing the correct option. Do not use the previous answer.

\end{tcolorbox}

\medskip

\begin{tcolorbox}[promptbox={Process Critique Feedback --- HellaSwag}]

The prior continuation is unsatisfactory. Check chronology, actors,
physical plausibility, and whether the continuation is stylistically
natural.

\medskip
I am not providing the correct option. Do not use the previous answer.

\end{tcolorbox}

\medskip

\begin{tcolorbox}[promptbox={Process Critique Feedback --- MathMCQA}]

The prior solution is unsatisfactory. Locate the earliest step whose
arithmetic or algebra is not independently justified, recompute from
there, and compare the result with every option.

\medskip
I am not providing the correct option. Do not use the previous answer.

\end{tcolorbox}

\medskip

\begin{tcolorbox}[promptbox={Process Critique Feedback --- StrategyQA}]

The prior answer is unsatisfactory. Separate the question into factual
claims, verify whether each claim supports or contradicts the conclusion,
and then reconsider.

\medskip
I am not providing the correct answer.

\end{tcolorbox}

\subsection{MBPP Prompt Format}\label{app:mbpp_prompt}
\begin{tcolorbox}[promptbox={Feedback Template --- MBPP}]

\textbf{General feedback}

\medskip

The implementation executed but failed at least one hidden functional test.
Re-read the task, check boundary cases and return-value semantics, and revise
the algorithm.

\medskip

\textbf{[REQUIRED INTERFACE]} Your revised program must define this callable
interface exactly. The function name and parameter list are part of the task
specification; the function body remains for you to implement.

\begin{verbatim}
def example_function(n):
    ...
\end{verbatim}

The hidden tests, test inputs, expected outputs, assertions, and dataset
reference implementation are not being provided. Return only a complete
Python implementation in one \texttt{```python} code block.

\medskip
\hrule
\medskip

\textit{\textbf{The first instruction above is replaced by one of the
following when a specific execution error occurs:}}

\medskip

\textbf{Syntax error}

\smallskip

The implementation has a Python syntax or indentation error near line
\texttt{\{line\_number\}}. Repair the program structure and check that every
block is complete.

\medskip

\textbf{Runtime error}

\smallskip

The implementation raised \texttt{\{exception\_type\}} while a hidden
functional test was running. Check initialization, input types, boundary
cases, and all return paths.

\medskip

\textbf{Timeout}

\smallskip

The implementation did not finish within the execution limit. Check
termination, recursion, loop bounds, and algorithmic complexity.

\medskip

\textbf{No extractable Python code}

\smallskip

The response is unsatisfactory because no executable Python implementation
could be extracted.

\end{tcolorbox}

\section{Derivation of Theorem}\label{app:proofs}

In this section, we provide the full proof of Theorem~\ref{thm:margin} and an associated key algebraic lemma utilized for the final derivation of said theorem.

\subsection{Proof of Lemma}
\begin{lemma}[Confidence-Amplified Log Bound]\label{lem:log-bound}
For any $\mu \in (0,1)$ and $p \in (\mu, 1)$:
\begin{equation}\label{eq:log-bound}
-\log(1-p) \;\geq\; \log p \;+\; h(\mu),
\end{equation}
where $h(\mu)$ is as defined in \eqref{eq:h-mu}. The bound is tight: equality holds at $p = 1/2$ when $\mu < 1/2$, in the limit $p \to 1/2^+$ when $\mu = 1/2$, and in the limit $p \to \mu^+$ when $\mu > 1/2$.
\end{lemma}

\begin{proof}
Define $f(p) := -\log(1-p) - \log p = \log\frac{1}{p(1-p)}$. Inequality \eqref{eq:log-bound} is equivalent to $f(p) \geq h(\mu)$ for all $p > \mu$.

Since $p(1-p) = \tfrac{1}{4} - (p - \tfrac{1}{2})^2 \leq \tfrac{1}{4}$ for all $p \in (0,1)$, with equality if and only if $p = \tfrac{1}{2}$, we have $f(p) \geq \log 4$ universally, with $f(p)$ uniquely minimized at $p^* = \tfrac{1}{2}$. Moreover, $f'(p) = \frac{2p-1}{p(1-p)}$, so $f$ is strictly decreasing on $(0, \tfrac{1}{2})$ and strictly increasing on $(\tfrac{1}{2}, 1)$.

\textbf{Case 1} ($\mu \leq \tfrac{1}{2}$): The domain $\{p > \mu\}$ contains $p^* = \tfrac{1}{2}$, so $\inf_{p > \mu} f(p) = f(\tfrac{1}{2}) = \log 4 = h(\mu)$.

\textbf{Case 2} ($\mu > \tfrac{1}{2}$): $f$ is strictly increasing on $(\mu, 1)$, so $\inf_{p > \mu} f(p) = \lim_{p \to \mu^+} f(p) = f(\mu) = \log \frac{1}{\mu(1-\mu)} = h(\mu)$.

In both cases $f(p) \geq h(\mu)$ for all $p > \mu$, establishing \eqref{eq:log-bound}.
\end{proof}

\subsection{Proof of Theorem}
\begin{proof}[Proof of Theorem~\ref{thm:margin}]
Since $\lambda_w \ell_\mathrm{win}(\boldsymbol{\theta}) \geq 0$ and $\lambda_l \ell_\mathrm{lose}(\boldsymbol{\theta}) \geq 0$, the bound $\ell_\mathrm{SPEAR}(\boldsymbol{\theta}) \leq \epsilon$ implies individually:
\begin{align}
\ell_\mathrm{win}(\boldsymbol{\theta}) &\leq \frac{\epsilon}{\lambda_w}, \label{pf:win-ub}\\
\ell_\mathrm{lose}(\boldsymbol{\theta}) &\leq \frac{\epsilon}{\lambda_l}. \label{pf:lose-ub}
\end{align}

\textbf{Step 1: Win lower bound.}
Expanding $\ell_\mathrm{win}$ from \eqref{eq:sft} for the single sample and applying \eqref{pf:win-ub}:
\[
-\sum_{i=0}^{|\mathbf{y}^+|-1} \log p_{\boldsymbol{\theta}}(y^+_i \mid \mathbf{c}^{(0)}, \mathbf{y}^+_{<i}) \;\leq\; \frac{\epsilon}{\lambda_w}.
\]
Dividing by $|\mathbf{y}^+|$:
\begin{equation}\label{pf:win-lb}
\frac{1}{|\mathbf{y}^+|}\log p_{\boldsymbol{\theta}}(\mathbf{y}^+ \mid \mathbf{c}^{(0)}) \;\geq\; -\frac{\epsilon}{\lambda_w\, |\mathbf{y}^+|}.
\end{equation}

\textbf{Step 2: Active-lose upper bound.}
Expanding $\ell_\mathrm{lose}$ from \eqref{eq:unlikelihood-loss} under Assumption~\ref{assump:confidence} (so all tokens in $\mathcal{T}_\tau(\mathbf{y}^-)$ are active) and applying \eqref{pf:lose-ub}:
\[
\alpha \sum_{i \in \mathcal{T}_\tau(\mathbf{y}^-)}\Bigl(-\log\bigl(1 - p_{\boldsymbol{\theta}}(y^-_i \mid \mathbf{c}^{(0)}, \mathbf{y}^-_{<i})\bigr)\Bigr) \;\leq\; \frac{\epsilon}{\lambda_l}.
\]
Dividing by $\alpha > 0$:
\begin{equation}\label{pf:ul-sum}
\sum_{i \in \mathcal{T}_\tau(\mathbf{y}^-)}\Bigl(-\log\bigl(1 - p_{\boldsymbol{\theta}}(y^-_i \mid \mathbf{c}^{(0)}, \mathbf{y}^-_{<i})\bigr)\Bigr) \;\leq\; \frac{\epsilon}{\lambda_l\alpha}.
\end{equation}

\textbf{Step 3: Applying Lemma~\ref{lem:log-bound}.}
By Assumption~\ref{assump:confidence} and Lemma~\ref{lem:log-bound}, for each $i \in \mathcal{T}_\tau(\mathbf{y}^-)$:
\[
\log p_{\boldsymbol{\theta}}(y^-_i \mid c^{(0)}, \mathbf{y}^-_{<i}) \;\leq\; -\log\bigl(1 - p_{\boldsymbol{\theta}}(y^-_i \mid \mathbf{c}^{(0)}, \mathbf{y}^-_{<i})\bigr) - h(\mu).
\]
Summing over $\mathcal{T}_\tau(\mathbf{y}^-)$ and applying \eqref{pf:ul-sum}:
\begin{equation}\label{pf:lose-logp}
\sum_{i \in \mathcal{T}_\tau(\mathbf{y}^-)} \log p_{\boldsymbol{\theta}}(y^-_i \mid \mathbf{c}^{(0)}, \mathbf{y}^-_{<i}) \;\leq\; \frac{\epsilon}{\lambda_l\alpha} - |\mathcal{T}_\tau(\mathbf{y}^-)|\, h(\mu).
\end{equation}
Dividing by $|\mathcal{T}_\tau(\mathbf{y}^-)|$:
\begin{equation}\label{pf:lose-avg-ub}
\frac{1}{|\mathcal{T}_\tau(\mathbf{y}^-)|}\sum_{i \in \mathcal{T}_\tau(\mathbf{y}^-)} \log p_{\boldsymbol{\theta}}(y^-_i \mid \mathbf{c}^{(0)}, \mathbf{y}^-_{<i}) \;\leq\; \frac{\epsilon}{\lambda_l\alpha\,|\mathcal{T}_\tau(\mathbf{y}^-)|} - h(\mu).
\end{equation}

\textbf{Step 4: Assembling the margin bound.}
Substituting \eqref{pf:win-lb} and \eqref{pf:lose-avg-ub} into the definition of $M_\tau$ \eqref{eq:margin}:
\begin{align*}
M_\tau(\boldsymbol{\theta};\, \mathbf{y}^+, \mathbf{y}^-)
&= \frac{1}{|\mathbf{y}^+|}\log p_{\boldsymbol{\theta}}(\mathbf{y}^+ \mid \mathbf{c}^{(0)}) - \frac{1}{|\mathcal{T}_\tau(\mathbf{y}^-)|}\sum_{i \in \mathcal{T}_\tau(\mathbf{y}^-)} \log p_{\boldsymbol{\theta}}(y^-_i \mid \mathbf{c}^{(0)}, \mathbf{y}^-_{<i}) \\[4pt]
&\geq -\frac{\epsilon}{\lambda_w\,|\mathbf{y}^+|} - \left(\frac{\epsilon}{\lambda_l\alpha\,|\mathcal{T}_\tau(\mathbf{y}^-)|} - h(\mu)\right) \\[4pt]
&= h(\mu) - \epsilon\!\left(\frac{1}{\lambda_w\,|\mathbf{y}^+|} + \frac{1}{\lambda_l\,\alpha\,|\mathcal{T}_\tau(\mathbf{y}^-)|}\right),
\end{align*}
which establishes \eqref{eq:margin-bound} and completes the proof. The characterization of $h(\mu)$ follows from Lemma~\ref{lem:log-bound}.

\end{proof}

\subsection{Theorem Extension to Batch-Level}
\label{app:batch_extension}

Theorem~\ref{thm:margin} is stated for a single win-lose pair to make the
mechanism transparent. However, the objective-level result extends directly
to arbitrary finite minibatches containing multiple win and lose traces, as
used in practice. Firstly, let the win and lose subsets induced by minibatch $\mathcal{B}_k$ be
\[
\mathcal{W}_{\mathcal{B}}
=
\left\{
\left(\mathbf{c}^{(0)}_j,\mathbf{y}^+_j\right)
\right\}_{j=0}^{b_w-1},
\qquad
\mathcal{L}_{\mathcal{B}}
=
\left\{
\left(\mathbf{c}^{(0)}_r,\mathbf{y}^-_r,\alpha_r\right)
\right\}_{r=0}^{b_l-1},
\]
where $b_w$ and $b_l$ denote the number of win and lose traces,
respectively. Next, we define the set of lose tail tokens
activated by the confidence gate as
\begin{equation}
\label{eq:batch-active-set}
A_r(\boldsymbol{\theta})
=
\left\{
i \in \mathcal{T}_\tau(\mathbf{y}^-_r)
:
p_{\boldsymbol{\theta}}
\left(
y^-_{r,i}
\mid
\mathbf{c}^{(0)}_r,\mathbf{y}^-_{r,<i}
\right)
>
\mu
\right\}.
\end{equation}

Now, we define the total number of win tokens and the weighted number of active lose
tokens as
\begin{equation}
\label{eq:batch-token-counts}
S_w
=
\sum_{j=0}^{b_w-1}
|\mathbf{y}^+_j|,
\qquad
S_l
=
\sum_{r=0}^{b_l-1}
\alpha_r
|A_r(\boldsymbol{\theta})|
\end{equation}

and define the token-weighted batch margin as
\begin{align}
\overline{M}_\tau(\boldsymbol{\theta})
&=
\frac{
\displaystyle
\sum_{j=0}^{b_w-1}
\log
p_{\boldsymbol{\theta}}
\left(
\mathbf{y}^+_j
\mid
\mathbf{c}^{(0)}_j
\right)
}{
S_w
}
-
\frac{
\displaystyle
\sum_{r=0}^{b_l-1}
\alpha_r
\sum_{i\in A_r(\boldsymbol{\theta})}
\log
p_{\boldsymbol{\theta}}
\left(
y^-_{r,i}
\mid
\mathbf{c}^{(0)}_r,\mathbf{y}^-_{r,<i}
\right)
}{
S_l
}.
\label{eq:batch-margin}
\end{align}

\textbf{Batch-Level Margin Extension.}
\label{prop:batch-margin}
Suppose $S_w>0$ and $S_l>0$. If
\[
\ell_{\mathrm{SPEAR}}(\boldsymbol{\theta})
\leq
\epsilon,
\]
then
\begin{equation}
\label{eq:batch-margin-bound}
\overline{M}_\tau(\boldsymbol{\theta})
\geq
h(\mu)
-
\epsilon
\left(
\frac{b_w}{\lambda_w S_w}
+
\frac{b_l}{\lambda_l S_l}
\right).
\end{equation}

\begin{proof}
As in the proof of Theorem~\ref{thm:margin}, nonnegativity of the two
components of the SPEAR objective implies
\begin{equation}
\ell_{\mathrm{win}}(\boldsymbol{\theta})
\leq
\frac{\epsilon}{\lambda_w},
\qquad
\ell_{\mathrm{lose}}(\boldsymbol{\theta})
\leq
\frac{\epsilon}{\lambda_l}.
\label{eq:batch-component-bounds}
\end{equation}

For a minibatch containing $b_w$ win traces, the win loss is the average
negative log-likelihood across win examples, resulting in
\[
-\frac{1}{b_w}
\sum_{j=0}^{b_w-1}
\log
p_{\boldsymbol{\theta}}
\left(
\mathbf{y}^+_j
\mid
\mathbf{c}^{(0)}_j
\right)
\leq
\frac{\epsilon}{\lambda_w},
\]
which gives
\begin{equation}
\label{eq:batch-win-bound}
\frac{
\displaystyle
\sum_{j=0}^{b_w-1}
\log
p_{\boldsymbol{\theta}}
\left(
\mathbf{y}^+_j
\mid
\mathbf{c}^{(0)}_j
\right)
}{
S_w
}
\geq
-
\frac{b_w\epsilon}{\lambda_w S_w}.
\end{equation}

Similarly, for the lose subset,
\[
\frac{1}{b_l}
\sum_{r=0}^{b_l-1}
\alpha_r
\sum_{i\in A_r(\boldsymbol{\theta})}
\left[
-\log
\left(
1-
p_{\boldsymbol{\theta}}
\left(
y^-_{r,i}
\mid
\mathbf{c}^{(0)}_r,\mathbf{y}^-_{r,<i}
\right)
\right)
\right]
\leq
\frac{\epsilon}{\lambda_l},
\]
and therefore
\begin{equation}
\label{eq:batch-ul-bound}
\sum_{r=0}^{b_l-1}
\alpha_r
\sum_{i\in A_r(\boldsymbol{\theta})}
\left[
-\log
\left(
1-
p_{\boldsymbol{\theta}}
\left(
y^-_{r,i}
\mid
\mathbf{c}^{(0)}_r,\mathbf{y}^-_{r,<i}
\right)
\right)
\right]
\leq
\frac{b_l\epsilon}{\lambda_l}.
\end{equation}

By construction, every $i\in A_r(\boldsymbol{\theta})$ satisfies
\[
p_{\boldsymbol{\theta}}
\left(
y^-_{r,i}
\mid
\mathbf{c}^{(0)}_r,\mathbf{y}^-_{r,<i}
\right)
>
\mu.
\]
Hence applying Lemma~\ref{lem:log-bound} tokenwise results in
\begin{align*}
\log
p_{\boldsymbol{\theta}}
\left(
y^-_{r,i}
\mid
\mathbf{c}^{(0)}_r,\mathbf{y}^-_{r,<i}
\right)
&\leq
-\log
\left(
1-
p_{\boldsymbol{\theta}}
\left(
y^-_{r,i}
\mid
\mathbf{c}^{(0)}_r,\mathbf{y}^-_{r,<i}
\right)
\right)
-
h(\mu),
\end{align*}

and then multiplying by $\alpha_r$, summing over all active lose tokens and all lose
examples, and using \eqref{eq:batch-ul-bound}, we obtain
\begin{align}
&\sum_{r=0}^{b_l-1}
\alpha_r
\sum_{i\in A_r(\boldsymbol{\theta})}
\log
p_{\boldsymbol{\theta}}
\left(
y^-_{r,i}
\mid
\mathbf{c}^{(0)}_r,\mathbf{y}^-_{r,<i}
\right)
\nonumber\\
&\qquad\leq
\frac{b_l\epsilon}{\lambda_l}
-
h(\mu)
\sum_{r=0}^{b_l-1}
\alpha_r
|A_r(\boldsymbol{\theta})|
\nonumber\\
&\qquad=
\frac{b_l\epsilon}{\lambda_l}
-
S_l h(\mu).
\label{eq:batch-lose-log-bound}
\end{align}

Dividing by $S_l$ gives
\begin{equation}
\label{eq:batch-lose-avg-bound}
\frac{
\displaystyle
\sum_{r=0}^{b_l-1}
\alpha_r
\sum_{i\in A_r(\boldsymbol{\theta})}
\log
p_{\boldsymbol{\theta}}
\left(
y^-_{r,i}
\mid
\mathbf{c}^{(0)}_r,\mathbf{y}^-_{r,<i}
\right)
}{
S_l
}
\leq
\frac{b_l\epsilon}{\lambda_l S_l}
-
h(\mu).
\end{equation}

Finally, substituting \eqref{eq:batch-win-bound} and
\eqref{eq:batch-lose-avg-bound} into \eqref{eq:batch-margin} yields
\begin{align*}
\overline{M}_\tau(\boldsymbol{\theta})
&\geq
-\frac{b_w\epsilon}{\lambda_w S_w}
-
\left(
\frac{b_l\epsilon}{\lambda_l S_l}
-
h(\mu)
\right)\\
&=
h(\mu)
-
\epsilon
\left(
\frac{b_w}{\lambda_w S_w}
+
\frac{b_l}{\lambda_l S_l}
\right),
\end{align*}
which establishes \eqref{eq:batch-margin-bound}.
\end{proof}

Theorem~\ref{thm:margin} is recovered as the single-pair special case
$b_w=b_l=1$. Under Assumption~\ref{assump:confidence},
$A_1(\boldsymbol{\theta})=\mathcal{T}_\tau(\mathbf{y}^-)$, so that
\[
S_w=|\mathbf{y}^+|,
\qquad
S_l=\alpha|\mathcal{T}_\tau(\mathbf{y}^-)|,
\]
and \eqref{eq:batch-margin-bound} reduces exactly to
\eqref{eq:margin-bound}.



\end{document}